\documentclass{article} 

\usepackage[utf8]{inputenc} 
\usepackage[T1]{fontenc}    
\usepackage{hyperref}       
\usepackage{url}            
\usepackage{booktabs}       
\usepackage{amsfonts}       
\usepackage{nicefrac}       
\usepackage{microtype}      
\usepackage{xcolor}         
\usepackage{enumerate}
\usepackage{fullpage}
\usepackage{nicefrac}  

\usepackage{hyperref}
\usepackage{url}
\usepackage{fullpage}
\usepackage{parskip}
\usepackage{authblk}
\usepackage{amsthm,amsmath,amssymb}

\newcommand{\eqqref}[1]{(\ref{#1})}

\usepackage{graphicx}
\usepackage{wrapfig}
\usepackage{mathtools}
\usepackage{xcolor}
\usepackage[euler]{textgreek}

\newcommand{\Energy}{\mathfrak E}
\newcommand{\energy}{\mathfrak e}
\newcommand{\gmap}{\zeta}
\newcommand{\fami}{{\bf fam}}
\newcommand{\bd}{{\bf d}}
\newcommand{\boldv}{{\bf v}}
\newcommand{\boldu}{{\bf u}}
\newcommand{\boldw}{{\bf w}}
\newcommand{\boldp}{{\bf p}}
\newcommand{\SDM}{SDM}
\newcommand{\bxt}{{\bf x^{\text{test}}}}
\newcommand{\nspl}{n_\text{spl}}
\newcommand{\bu}{{\bf u}}
\newcommand{\bv}{{\bf v}}
\newcommand{\fcount}{\; \mathfrak f}

\newcommand{\bw}{{\bf w}}

\definecolor{blue2}{rgb}{0.2, 0.2, 0.6}

\newcommand{\Ttrain}{\mathcal D_{\mathcal T}^{\text{train}}}
\newcommand{\Ttest}{\mathcal D_{\mathcal T}^{\text{test}}}

\newcommand{\ones}{{\bf 1}}

\newcommand{\voc}{{\mathcal V}}

\newcommand{\entropy}{{\mathcal H}}

\newcommand{\featspace}{\mathcal F}
\newcommand{\featmap}{\psi}

\newcommand{\partition}{\Phi}
\newcommand{\concept}{\mathcal C}
\newcommand{\data}{{\mathcal X}}
\newcommand{\bx}{{\bf x}}
\newcommand{\by}{{\bf y}}
\newcommand{\bz}{{\bf z}}
\newcommand{\bc}{{\bf c}}
\newcommand{\bp}{{\bf p}}
\newcommand{\bk}{{\bf k}}
\newcommand{\ephi}{\mathring \varphi}
\newcommand{\vphi}{\varphi}

\newcommand{\Dist}{\text{Distinct}}

\newcommand{\nat}{\mathbb{N}}
\newcommand{\real}{\mathbb{R}}

\newcommand{\altchoose}{\genfrac{\lbrace}{\rbrace}{0pt}{}}

\usepackage{pifont}

\begingroup
\makeatletter
\@for\theoremstyle:=definition,remark,plain\do{%
\expandafter\g@addto@macro\csname th@\theoremstyle\endcsname{%
\addtolength\thm@preskip\parskip
}%
}
\endgroup

\newtheorem{theorem}{Theorem}
\newtheorem{definition}{Definition}

\newtheorem{lemma}{Lemma}

\usepackage{caption}
\usepackage{subcaption}
\usepackage{enumitem}

\title{Long-Tailed Learning Requires   Feature Learning }

\author{
  Thomas Laurent\thanks{Loyola Marymount  University,   \texttt{tlaurent@lmu.edu}}
 \and
  James H. von Brecht  
  \and
  Xavier Bresson\thanks{National University of Singapore, \texttt{xavier@nus.edu.sg}}
}
\date{}
\begin{document}

\maketitle

\begin{abstract}
We propose a simple data model inspired from natural data such as text or images, and use it to study the importance of learning features in order to achieve good generalization.   Our data model 
follows a long-tailed distribution in the sense that some rare   subcategories have few representatives in the training set. In this context we provide evidence that a learner succeeds if and only if it identifies the correct features, and moreover derive non-asymptotic generalization error bounds that precisely quantify the penalty that one must pay for not learning features.
\end{abstract}

\section{Introduction}

Part of the motivation for deploying a neural network arises from the belief that  algorithms that learn features/representations generalize better than algorithms that do not.  We try to give some mathematical ballast to this notion by studying a data model where, at an intuitive level, a learner succeeds if and only if it manages to learn the correct features. The data model itself attempts to capture two key structures observed in natural data such as text or images.  First, it is endowed with a latent structure at the patch or word level that is directly tied to a classification task. Second, the data distribution has a long-tail, in the sense that rare and uncommon instances collectively form a significant fraction of the data. We derive non-asymptotic generalization error bounds that quantify, within our framework, the penalty that one must pay for not learning features.

We first prove a two part result that quantifies precisely the necessity of learning features within the context of our data model. The first part  shows that a trivial nearest neighbor classifier performs perfectly when given knowledge of the correct features. The second part shows  it is impossible to \emph{a priori} craft a feature map  that  generalizes well when  using a nearest neighbor classification rule.   In other words, success or failure depends only on the ability to identify the correct features and not on the underlying classification rule. Since this cannot be done \emph{a priori}, the features must be learned.

Our theoretical results therefore support the idea that  algorithms cannot generalize on long-tailed data if they do not learn features. Nevertheless, an algorithm that does learn features can generalize well. Specifically, the most direct neural network architecture for our data model generalizes almost perfectly when using either a linear classifier or a  nearest neighbor classifier on the top of the \emph{learned} features. Crucially, designing the architecture requires knowing only the meta structure of the problem, but no \emph{a priori} knowledge of the correct features. This illustrates the built-in advantage of neural networks; their ability to learn features significantly eases the design burden placed on the  practitioner.

Subcategories in  commonly used visual recognition datasets  
tend to follow a long-tailed distribution  \cite{salakhutdinov2011learning, zhu2014capturing, feldman2020neural}.  Some common subcategories have a wealth of representatives in the training set,  whereas many rare subcategories  only have a few representatives. At an intuitive level,  learning features seems especially important on a long-tailed dataset since features learned from the common subcategories help to properly classify test points from a rare subcategory. Our theoretical results help support this intuition.

 We note that when considering  complex visual recognition tasks, datasets are almost unavoidably  long-tailed 
\cite{liu2019large}
  --- even if the dataset contains millions of images, it is to be expected that many subcategories will have  few samples. In this  setting, the classical approach of deriving asymptotic performance guarantees based on a large-sample limit  is not a fruitful avenue. Generalization must be approached from a different point of view (c.f. \cite{feldman2020does} for very interesting work in this direction). In particular, the analysis must be non-asymptotic. One of our main contribution is to derive, within the context of our data model, generalization error bounds that are non-asymptotic and relatively tight --- by this we mean that our results hold for small numbers of data samples  and   track reasonably well with empirically evaluated generalization error.


In Section \ref{section:oneshot} we introduce our data model and in Section  \ref{section:statement} we discuss our  theoretical results. For the simplicity of exposition, both sections focus on the case where  each rare subcategory has a \emph{single} representative in the training set. Section \ref{section:extend_bound} is concerned with the general case in which each rare subcategory has  \emph{few} representatives. Section \ref{section:proof} provides an overview of our proof techniques.  Finally, in  Section \ref{section:empirical}, we investigate empirically a few questions that we couldn't resolve analytically. In particular, our error bounds are restricted to the case  in which  a nearest neighbor classification rule is applied on the top of the features --- we provide empirical evidence in this last section that replacing the nearest neighbor classifier  by a linear classifier leads to very minimal improvement. This further support the notion that, on our data model, it is the ability to learn features that drives success,  not the specific classification rule used on the top of the features.


{\bf Related work.}
 By now, a rich literature has developed that studies the generalization abilities of  neural networks. A major theme in this line of work is the use of the PAC learning framework to derive generalization bounds for neural networks (e.g. \cite{bartlett2017spectrally, neyshabur2017pac, golowich2018size, arora2018stronger, neyshabur2018towards}), usually by proving a bound on the difference between the finite-sample empirical loss and true loss. While powerful in their generality, such approaches are usually task independent and asymptotic; that is, they are mostly agnostic to any idiosyncrasies in the data generating process and need a statistically meaningful number of samples in the training set. As such, the PAC learning framework is not well-tailored to our specific aim of studying  generalization on long-tailed data distributions; indeed, in such setting, a rare subcategory might  have only a handful of representatives in the training set.
 

After breakthrough results (e.g. \cite{jacot2018neural, du2018gradient, allen2019convergence, ji2019polylogarithmic}) showed that vastly over-parametrized neural networks become kernel methods (the so-called Neural Tangent Kernel  or NTK) in an appropriate limit, much effort has gone toward analyzing the extent to which neural networks outperform kernel methods
\cite{yehudai2019power, wei2019regularization, refinetti2021classifying,  ghorbani2019limitations, ghorbani2020neural,  karp2021local, allen2019can, allen2020backward, li2020learning, malach2021quantifying}.  Our interest lies not in proving such a gap for its own sake, but rather in using the comparison to gain some understanding on the importance of learning features in computer vision and NLP contexts.

Analyses that shed theoretical light onto learning with long-tailed distributions \cite{feldman2020does, brown2021memorization} or onto specific learning mechanisms \cite{karp2021local} are perhaps closest to our own. The former analyses \cite{feldman2020does, brown2021memorization} investigate the necessity of memorizing rare training examples in order to obtain near-optimal generalization error when the data distribution is long-tailed. Our analysis differs to the extent that we focus on the necessity of learning features and sharing representations in order to properly classify rare instances. Like us, the latter analysis \cite{karp2021local} also considers a computer vision inspired task and uses it to compare a neural network to a kernel method, with the ultimate aim of studying the learning mechanism involved. Their object of study  (finding a sparse signal in the presence of noise), however, markedly differs from our own {(learning with long-tailed distributions).}

\begin{figure}[t] 
         \centering
          \includegraphics[width=1.0\textwidth]{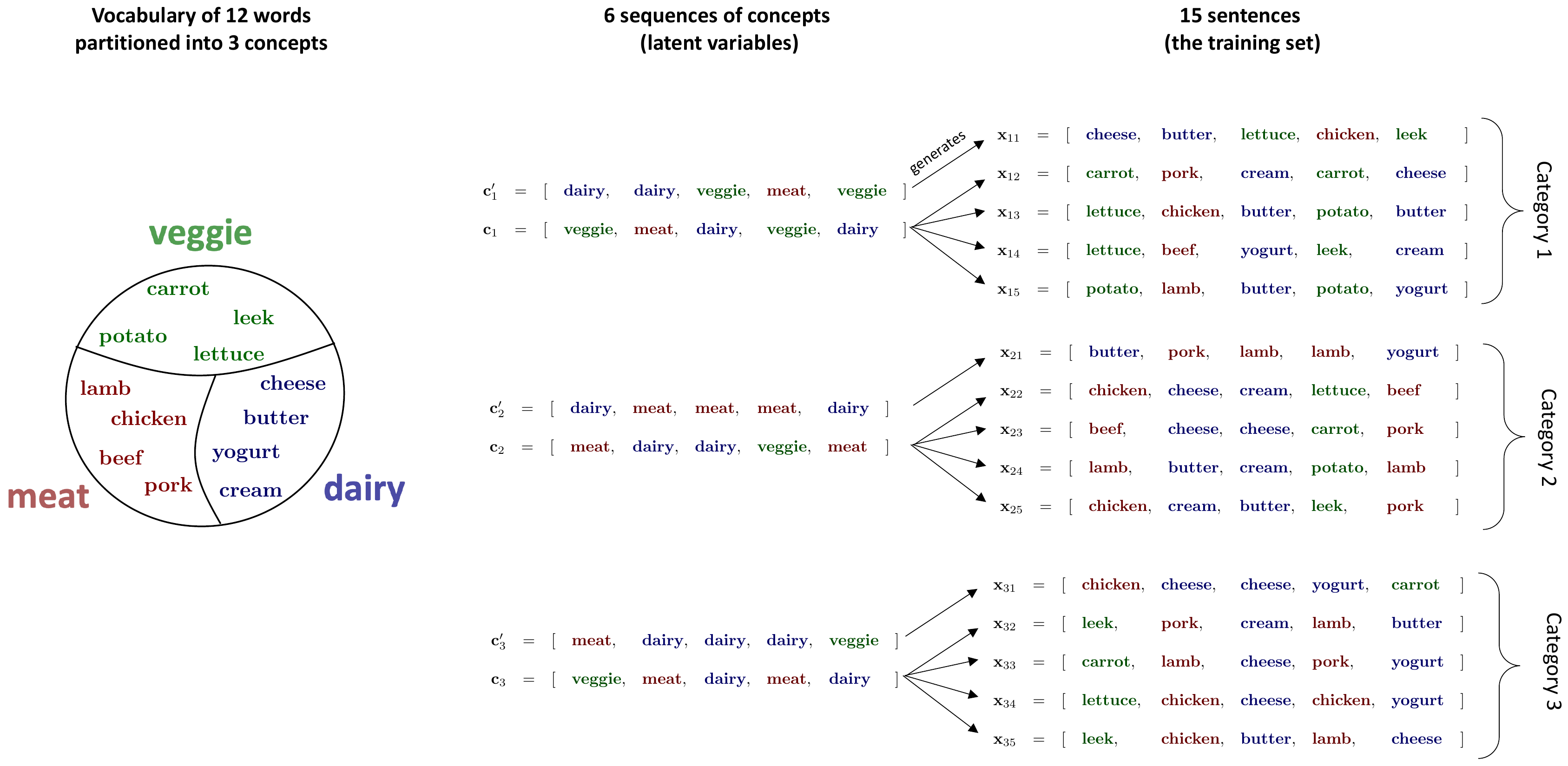}
             \caption{
             Data model with parameters set to $L=5$,
             $n_w=12$, $n_c=3$, $R=3$, and  $n_\text{spl} = 5$.}
             \label{figure:food}
\end{figure}


\section{The Data Model} \label{section:oneshot}
We begin with a simple example to explain our data model and to illustrate, at an intuitive level, the importance of learning features when faced with a long-tailed data distribution.
 For the sake of exposition we adopt NLP terminology such as `words' and `sentences,' but the image-based terminology of  `patches' and `images' would do as well.

The starting point is a very standard mechanism for generating observed data from some underlying collection of latent variables.
Consider the   data model depicted in Figure \ref{figure:food}.
We have a vocabulary of $n_w=12$ words and a set of $n_c=3$ concepts:
$$
\voc = \{ \text{potato, cheese, carrots, chicken}, \ldots \} \qquad \text{and} \qquad \mathcal C = \{\text{vegetable, dairy, meat}\}.
$$ 
The $12$ words are partitioned into  the $3$ concepts as shown on the left of Figure \ref{figure:food}.
 We also have $6$ sequences of concepts of length $L=5$.  They are denoted by  $\bc_1,\bc_2,\bc_3$ and  $\bc'_1,\bc'_2,\bc'_3$.   Sequences of concepts are latent variables that  generate sequences of words. For example
 $$
  [ \text{dairy},  \text{dairy}, \text{veggie}, \text{meat} , \text{veggie} ] \;\;  \overset{\text{generates}}{\longrightarrow} \;\; 
   [\text{cheese},  \text{butter}, \text{lettuce}, \text{chicken}, \text{leek}  ]
 $$
 The sequence of words on the right was obtained by sampling each word uniformly at random from the corresponding concept.
 For example, the first word was randomly chosen out of  the dairy concept ({\it butter, cheese, cream, yogurt}), and  the last word was randomly chosen out of the vegetable concept ({\it potato, carrot,  leek, lettuce}.)  Sequences of words will be referred to as sentences.
 
 The non-standard aspect of our model comes from how we use the `latent-variable $\to$ observed-datum' process to form a training distribution. 
 The training set in Figure \ref{figure:food} is  made of $15$  sentences split into $R=3$ categories.
 The latent variables $\bc'_1,\bc'_2,\bc'_3$ each generate a single sentence, whereas  the latent variables $\bc_1,\bc_2,\bc_3$ each generate $4$ sentences. We will refer to $\bc_1,\bc_2,\bc_3$ as the {\bf familiar} sequences of concepts since they generate most of the sentences encountered in the training set. On the other hand $\bc'_1,\bc'_2,\bc'_3$ will be called {\bf unfamiliar}. Similarly, a sentence generated by a familiar (resp. unfamiliar) sequence of concepts will be called a familiar (resp. unfamiliar) sentence. The former represents a datum sampled from the  head of a distribution while the latter represents a datum sampled from its  tail. We denote by $\bx_{r,s}$  the $s^{th}$ sentence of the $r^{th}$ category, indexed so that the first sentence of each category is an unfamiliar sentence and the remaining ones are familiar.

Suppose now that we have trained a learning algorithm  on the training set described above and that at inference time we are presented with a previously unseen sentence generated by the {\bf unfamiliar} sequence of concept $\bc'_1= [ \text{dairy},  \text{dairy}, \text{veggie}, \text{meat} , \text{veggie} ]$. To fix ideas, let's say  that sentence is:
\begin{equation}  \label{test_sentence}
\bx^{\text{test}} =  [\text{butter},  \text{yogurt}, \text{carrot}, \text{beef}, \text{lettuce}  ]
\end{equation}
This sentence is hard to classify since there is a single sentence in the training set that has been generated by the same sequence of concepts, namely
\begin{equation}  \label{train_sentence}
\bx_{1,1} =  [\text{cheese},  \text{butter}, \text{lettuce}, \text{chicken}, \text{leek}  ] \;.
\end{equation}
Moreover these two sentences do not overlap at all (i.e. the $i^{th}$ word of $\bx^{\text{test}}$ is different from the $i^{th}$ word of $\bx_{1,1}$ for all $i$.) 
To properly classify  $\bx^\text{test}$, the algorithm \emph{must} have learned the equivalences {\it butter} $\leftrightarrow$ {\it cheese},  {\it yogurt} $\leftrightarrow$ {\it butter}, {\it carrot} $\leftrightarrow$ {\it lettuce}, and so forth. 
In other words, the  algorithm must have learned the underlying concepts.

Nevertheless, a neural network with a well-chosen architecture can easily succeed at such a classification task. Consider, for example, the network depicted on Figure \ref{figure:mlpmixer}. 
Each word of the input sentence, after being encoded into a one-hot-vector, goes through a multi-layer perceptron (MLP 1 on the figure) shared across words.
\begin{wrapfigure}{r}{0.35\textwidth}
 \centering
    \includegraphics[totalheight=0.2\textheight]{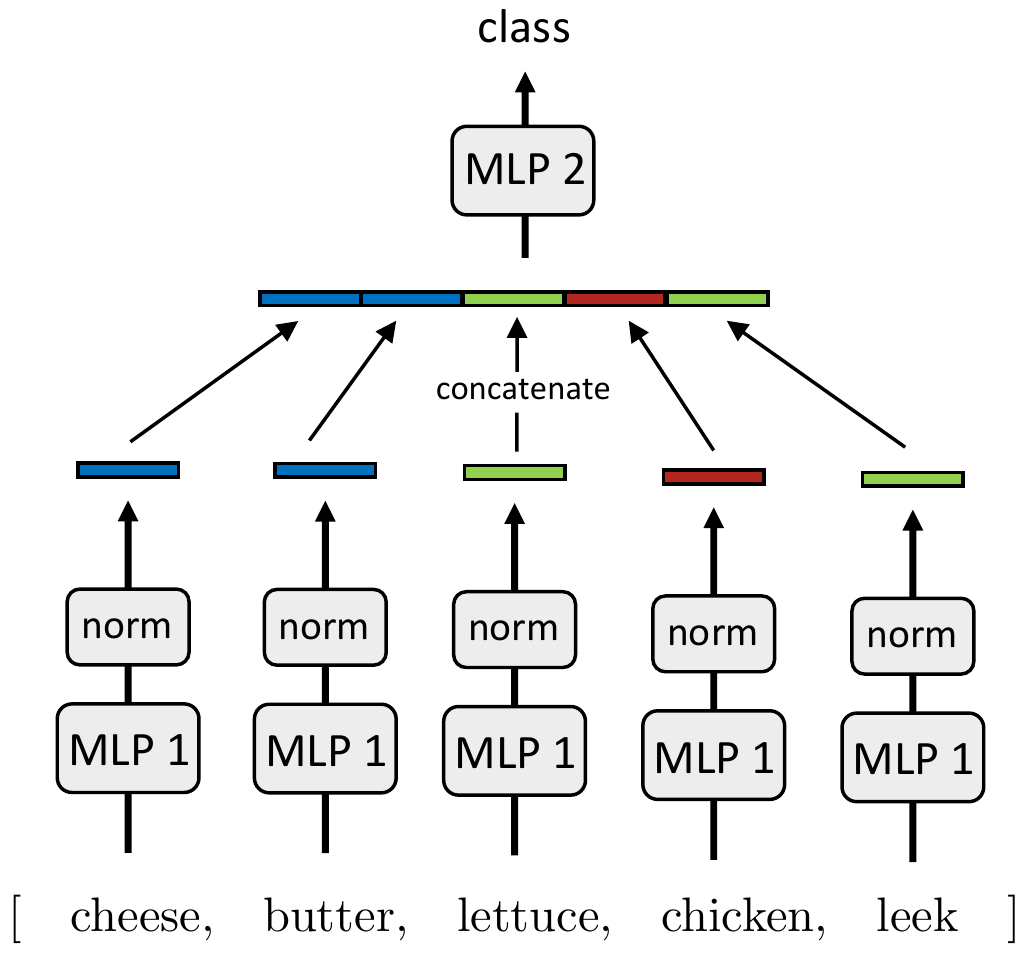}
    \caption{A simple neural net.}
    \label{figure:mlpmixer}
\end{wrapfigure} 
 The output is then normalized using LayerNorm \cite{ba2016layer} to produce a representation of the word.
  The word representations are then concatenated into a single vector that goes through a second multi-layer perceptron (MLP 2 on the figure). 
This network, if properly trained, will learn to give similar representations to  words that belong to the same concept. Therefore, if it correctly classifies the train point $\bx_{1,1}$  given by \eqqref{train_sentence}, it will necessarily correctly classify the test point $\bx^\text{test}$ given by  \eqqref{test_sentence}. So the neural network is able to classify the previously unseen sentence $\bx^\text{test}$
despite the fact that the training set contains a single example with the same underlying sequence of concepts. This comes from the fact that  the neural network  learns features and representations  from the familiar part of the training set (generated by the  head of the  distribution), and uses these, at test time, to correctly classify the unfamiliar sentences (generated by the tail of the distribution). In other words, because it learns features, the neural network has no difficulty handling the long-tailed nature of the distribution.  

To summarize, the variables $L, n_w, n_c, R$ and $n_\text{spl}$ parametrize instances of our data model. They denote, respectively, the length of the sentences, the number of words in the vocabulary, the number of concepts, the number of categories, and the number of samples per category. So in the example presented in Figure \ref{figure:food} we have $L=5$, $n_w=12$, $n_c=3$, $R=3$ and $n_\text{spl} = 5$ (four familiar sentences and one unfamiliar sentence per category). The vocabulary $\mathcal V$ and set of concepts $\mathcal C$ are discrete sets with $|\voc|=n_w$ and $|\mathcal C|=n_c$, rendered as $\mathcal V=\{1,\ldots, n_w \}$ and   $\mathcal C=\{1,\ldots, n_c \}$ for concreteness. A partition of the vocabulary into concepts, like the one depicted at the top of Figure \ref{figure:food}, is encoded by a function $\vphi: \mathcal V \to \mathcal C$ that assigns words to concepts.  We require that each concept contains the same number of words, so that $\vphi$ satisfies
\begin{equation} \label{equipartition0}
|\vphi^{-1}(\{c\})| =  |\{w \in \mathcal V: \vphi(w) =c \}| = n_w/n_c \qquad \text{for all } c \in \mathcal C,
\end{equation}
and we refer to such a function $\vphi: \mathcal V \to \mathcal C$  satisfying \eqqref{equipartition0} as  equipartition of  the vocabulary. The set
$$
\Phi = \{ \text{All functions $\vphi$ from $ \mathcal V$ to  $\mathcal C$ that satisfy  \eqqref{equipartition0} } \}
$$
denotes the collection of all such equipartitions, while the data space and latent space are denoted
$$
\data = \mathcal V^L \qquad \text{ and } \qquad \mathcal Z = \mathcal C^L,
$$
respectively. Elements of $\mathcal X$ are sentences of $L$ words and they take the form  $\bx = [x_1, x_2, \ldots,x_L],$ while elements of $\mathcal Z$ take the form $\bc = [c_1,c_2, \ldots, c_L]$ and correspond to sequences of concepts.

In the context of this work, a feature map refers to any function $\featmap: \data \mapsto \featspace$ from data space  to feature space. The feature space $\featspace$ can be any Hilbert space (possibly infinite dimensional)  and we denote by $\langle \cdot, \cdot \rangle_\featspace$ the associated inner product. Our analysis applies to the case in which a nearest neighbor classification rule is applied on the top of the extracted features. Such  rule works as follow:
 given a test point $\bx$, the inner products $\langle \featmap(\bx) , \featmap(\by) \rangle_\featspace$ are evaluated for all $\by$ in the training set;   the test point  $\bx$ is then given  the label of the training point $\by$ that led to the highest inner product.


\section{Statement and Discussion of Main Results} \label{section:statement}

Our main result states that, in the context of our data model, features must be tailored (i.e. learned) to each specific task. Specifically, it is not possible to find a universal feature map $\featmap: \data \to \featspace$ that performs well  on a collection of \emph{tasks} like the one depicted on Figure \ref{figure:food}. 
In the context of this work, a  \emph{task} refers to a tuple 
\begin{equation} \label{tuple}
  \mathcal T = (\;\;  \vphi \; \; ; \;\;   \bc_1, \ldots, \bc_R \;\; ; \;\;  \bc'_1, \ldots, \bc'_R \;\; ) \; \in \;  \Phi \times \mathcal Z^{2R}
\end{equation}
that prescribes a partition of the vocabulary into concepts, $R$ familiar sequences of concepts, and $R$ unfamiliar sequences of concepts. Given such a task $\mathcal T$ we generate a training set $S$ as described in the previous section. This training set contains  $R\times n_{spl}$ sentences split over $R$ categories, and each category contains a single unfamiliar sentence. 
Randomly generating the training set $S$ from the task $\mathcal T$ corresponds to sampling
$
S  \sim \Ttrain
$
from a distribution $\Ttrain$ defined on the space $\mathcal X^{R\times n_{spl}}$ and parametrized by the variables in \eqqref{tuple} (the appendix provides an explicit formula for this distribution).
We measure performance of an algorithm by its ability to generalize on previously unseen unfamiliar sentences. Generating an unfamiliar sentence amounts to drawing a sample
$
\bx \sim  \Ttest
$
from a distribution $\Ttest$ on the space $\mathcal X$ parametrized by the variables $\varphi, \bc'_1, \ldots, \bc'_R$ in \eqqref{tuple} that determine unfamiliar sequences of concepts. Finally, associated with every task $\mathcal T$ we have a labelling function $f_\mathcal T: \data \to \{1,\ldots,R\}$ that assigns the label $r$ to sentences generated by either $\bc_r$ or $\bc_r'$
(this function is ill-defined if two sequences of concepts from different categories are identical, but this issue is easily resolved by formal statements in the appendix).
Summarizing  our notations, for every task $\mathcal T \in \Phi \times \mathcal Z^{2R}$ we have a distribution $\Ttrain$ on  the space $\mathcal X^{R\times n_{spl}}$,  a distribution $\Ttest$ on  the space $\mathcal X$, and a labelling function $f_{\mathcal T}.$

Given a feature space $\featspace$, a feature map  $\psi: \data \to \featspace$, and  a task $\mathcal T \in   \Phi \times \mathcal Z^{2R}$,   the expected generalization error of the nearest neighbor classification rule  on unfamiliar sentences is given by:
\begin{equation} \label{error}
\text{err}(\featspace, \psi, \mathcal T)= 
\underset{S \sim \Ttrain}{\mathbb E}    \Bigg[  
\underset{\bx \sim \Ttest}{\mathbb P}    \left[  
f_{\mathcal T}\left( \arg \max_{\by \in S}  \langle \psi(\bx), \psi(\by) \rangle_\featspace \right) \neq f_{\mathcal T}(\bx)
 \right]
 \Bigg].
\end{equation}
For simplicity, if the test point  has multiple nearest neighbors with inconsistent labels in the training set (and so the $\arg \max$ returns multiple training points $\by$), we will count the classification as a failure for the nearest neighbor classification rule.
 We therefore replace \eqqref{error} by the more formal (but more cumbersome) formula
\begin{equation}  \label{error22}
\text{err}(\featspace, \psi, \mathcal T)= 
\underset{S \sim  \Ttrain}{\mathbb E}    \Bigg[  
\underset{\bx \sim \Ttest}{\mathbb P}    \left[  \exists \by \in  \arg \max_{\by \in S}  \langle \psi(\bx), \psi(\by) \rangle_\featspace \text{ such that  }
f_{\mathcal T}(\by) \neq f_{\mathcal T}(\bx)
 \right]
 \Bigg]
\end{equation}
to make this explicit. Our main theoretical results concern performance of a learner not on a single task $\mathcal T$ but on a collection of tasks
$
\mathfrak T = \{\mathcal T_1, \mathcal T_2, \ldots, \mathcal T_{N_{\text{tasks}}} \},
$
and so we define
\begin{equation} \label{expected_gen}
\overline{\text{err}}(\featspace, \psi, \mathfrak T) = \frac{1}{|\mathfrak T|} \sum_{\mathcal T \in \mathfrak T} \text{err}(\featspace, \psi, \mathcal T)
\end{equation}
as the expected generalization error on such a collection $\mathfrak T$ of tasks. As a task refers to an element of the discrete set  $\Phi \times \mathcal Z^{2R}$, any subset $\mathfrak{T} \subset \Phi \times \mathcal Z^{2R}$ defines a collection of tasks. Our main result concerns the case where the collection of tasks $\mathfrak T = \Phi \times \mathcal Z^{2R}$ consists in \emph{all possible tasks} that one might encounter. For concreteness, we choose specific values for the model parameters and state the following special case of our main theorem (Theorem \ref{thm:main} at the end of this section) ---

\begin{theorem}\label{thm:special} Let $L=9$, $n_w=150$, $n_c=5$,  $R=1000$ and   $n_\text{spl}\ge 2$. Let   $\mathfrak T = \Phi \times \mathcal Z^{2R}$. 
 Then 
$$
\overline{\text{err}}(\featspace, \psi, \mathfrak T)> 98.4\%
$$
for all  feature spaces  $\featspace$, and all feature maps $\featmap: \data \mapsto \featspace$.
\end{theorem}
In other words, for the model parameters specified above, it is not possible to design a `task-agnostic' feature map $\psi$ that works well if we are uniformly uncertain about which specific task we will face. Indeed, the best possible feature map  will fail at least $98.4\%$ of the time at classifying unfamiliar sentences (with a nearest-neighbor classification rule), where the probability is with respect to the random choices of the task, of the training set, and of the unfamiliar test  sentence.

{\bf Interpretation:} Our desire to understand \emph{learning} demands that we consider a collection of tasks rather than a single one, for if we consider only a single task then the problem, in our setting, becomes trivial. Indeed, assume $\mathfrak T =\{\mathcal T_1\}$ with $\mathcal T_1 = ( \vphi  ;    \bc_1, \ldots, \bc_R  ;   \bc'_1, \ldots, \bc'_R )$ consists only of a single task. With knowledge of this task we can easily construct a feature map  $\psi: \data \to \real^{Ln_c}$ that performs perfectly. Indeed, the map
\begin{equation} \label{onehot}
\psi([x_1, \ldots, x_L]) = [{\bf e}_{\vphi(x_1)}, \ldots, {\bf e}_{\vphi(x_L)}] 
\end{equation}
that simply `replaces' each word $x_\ell$ of the input sentence by the one-hot-encoding ${\bf e}_{\vphi(x_\ell)}$ of its corresponding concept will do. 
A bit of thinking reveals that the nearest neighbor classification rule associated with feature map \eqqref{onehot}  perfectly solves the task $\mathcal T_1$. This is due to the fact that sentences  generated by the same sequence of concepts are  mapped by $\psi$ to the exact same location in feature space. As a consequence, the nearest neighbor classification rule will match the unfamiliar test sentence $\bx$ to the unique training sentence $\by$ that occupies the same location in feature space, and this training sentence has the correct label by construction (assuming that sequences of concepts from different categories are distinct). 
 To put it formally: 
\begin{theorem}\label{thm:trivial} Given a task   $\mathcal T \in \Phi \times \mathcal Z^{2R}$ satisfying $\bc'_r \neq \bc'_s$ and $\bc'_r \neq \bc_s$ for all $r \neq s$, 
there exists a feature space  $\featspace$ and a feature map $\featmap: \data \mapsto \featspace$ such that
$
{\text{err}}(\featspace, \psi, \mathcal T)= 0
$.
\end{theorem}
 
Consider now the case where $\mathfrak T=\{\mathcal T_1, \mathcal T_2\}$ consists of two tasks. According to Theorem \ref{thm:trivial}  there exists a $\psi$ that perfectly solves $\mathcal T_1$, but this $\psi$ might perform poorly on $ \mathcal T_2$, and vice versa. So, it might not be possible to design good features if we do not know \emph{a priori} which of these tasks we will face. Theorem \ref{thm:special} states that, in the extreme case where $\mathfrak T$  contains all possible tasks, this is indeed the case --- the best possible `task-agnostic' features $\psi$ will perform catastrophically on average. In other words, features must be task-dependent in order to succeed.

To draw a very approximate analogy, imagine once again that $\mathfrak T =\{\mathcal T_1\}$ and that  $\mathcal T_1$ represents, say, a hand-written digit classification task. A practitioner, after years of experience, could hand-craft a very good feature map $\psi$ that performs almost perfectly for this task. If we then imagine the case $\mathfrak T=\{\mathcal T_1, \mathcal T_2\}$ where $\mathcal T_1$ represents a hand-written digit classification task and  $\mathcal T_2$ represents, say, an animal classification task, then it becomes more difficult for a practitioner to handcraft a feature map $\psi$ that works well for \emph{both} tasks. In this analogy, the size of the set $\mathfrak T$ encodes the amount of knowledge the practitioner has about the specific tasks she will face. The extreme choice $\mathfrak T= \Phi \times \mathcal Z^{2R}$ corresponds to  the practitioner knowing nothing beyond the fact that \emph{natural images are made of patches}. Theorem \ref{thm:special} quantifies, in this extreme case, the impossibility of hand-crafting a feature map $\psi$ knowing only the range of possible tasks and not the specific task itself. In a realistic setting the collection of tasks $\mathfrak T$ is smaller, of course, and the data generative process itself is more coherent 
 than in our simplified setup. 
 Nonetheless, we hope our analysis sheds some light on some of the essential limitations of algorithms that do not learn features. 

Finally, our empirical results (see Section \ref{section:empirical}) show that a simple algorithm that learns features does not face this obstacle. We do \emph{not} need knowledge of the specific task $\mathcal{T}$ in order to design a good neural network architecture, but only of the family of tasks $\mathfrak T= \Phi \times \mathcal Z^{2R}$ that we will face. Indeed, the architecture in Figure \ref{figure:mlpmixer} succeeds at classifying unfamiliar test sentences more than $99\%$ of the time. This probability, which we empirically evaluate, is with respect to the choice of the task, the choice of the training set, and the choice of the unfamiliar test sentence (we use the values of $L, n_w, n_c$ and $R$ from  Theorem \ref{thm:special}, and $n_\text{spl} = 6,$ for this experiment). Continuing with our approximate analogy, this means our hypothetical practitioner needs no domain specific knowledge beyond the patch structure of natural images when designing a successful architecture. In sum, successful feature design requires task-specific knowledge while successful architecture design requires only knowledge of the task family. 

{\bf Main Theorem:} Our main theoretical result  extends Theorem \ref{thm:special} to arbitrary values of  $L$, $n_w$, $n_c$, $n_{\rm spl}$ and $R$. The resulting formula involves various combinatorial quantities.
We denote by ${ n \choose k}$ the  binomial coefficients and  by $\altchoose{n}{k}$  the Stirling numbers of the second kind. 
Let $\nat=\{0,1,2,\ldots\}$ 
and let    $\gamma, \hat \gamma: \nat^{L+1}\to \nat$ be the functions defined by  
$
\gamma(\bk) := \sum_{i=1}^{L+1} (i-1) k_i$   and  $\hat \gamma(\bk) := \sum_{i=1}^{L+1} i k_i,
$
respectively. We then define, for $0 \leq \ell \leq L$,  the sets
$$
\mathcal S_\ell:= \left\{ \bk \in \nat^{L+1}:  \;\;\; \hat\gamma(\bk)  = n_w  \quad \text{ and } \quad   \ell \le  \gamma(\bk) \le L \right\}.
$$
We
let $\mathcal S = \mathcal S_0$,
 and we note that the inclusion $\mathcal S_\ell \subset \mathcal S$ always holds. Given  $\bk \in  \nat^{L+1}$ we denote by 
$$
 \mathcal A_{\bk} := \Big\{A \in \mathbb N^{(L+1) \times n_c}: \;\;  \sum_{i=1}^{L+1} i A_{ij} =  n_w/n_c  \;\;  \text{for all  $j$} \;\;\; \text{ and } \;\;\;  \sum_{j=1}^{n_c} A_{ij}  = k_i \;\;  \;\;  \text{for all $i$}   \Big\}
$$
the set of $\bk$-admissible matrices. Finally, we let  $\mathfrak f, \mathfrak g: \mathcal S \to \real$ be the functions defined by 
\begin{align*}
& \mathfrak f(\bk ) := {\Big( (n_w/n_c)! \Big)^{n_c}}  \frac{n_c^L}{ n_w! }  \sum_{A \in \mathcal A_{k} }  
\left( \prod_{i=1}^{L+1} \frac{k_i!}{A_{i,1}!\;  A_{i,2}!\; \cdots  \;A_{i,n_c}! } \right) \qquad \text{and} \qquad \\
&  \mathfrak g({\bf k} ) :=  \frac{\gamma({\bf k})!}{n_w^{2L}} \left( \frac{n_w!}{k_1! k_2! \cdots k_{L+1}!} \prod_{i=2}^{L+1}\left( \frac{i^{(i-2)}}{i!} \right)^{k_i} \right) \left( \sum_{i=\gamma(\bf k)}^L  {L \choose i}   \altchoose{ i}{ \gamma(\bk)} \;  2^{i} n_w^{L-i} \right),
\end{align*}
respectively. With these definitions at hand, we may now state our main theorem.
\begin{theorem}[Main Theorem] \label{thm:main}
Let   $\mathfrak T = \Phi \times \mathcal Z^{2R}$. 
 Then 
\begin{equation} \label{main_lower_bound}
\overline{\text{err}}(\featspace, \psi, \mathfrak T) \ge  \left(
  \sum_{\bk \in \mathcal S_\ell}  \mathfrak f(\bk ) \mathfrak g(\bk )  \right) -   \frac{1}{R} \left(1 + \frac{1}{2} \max_{\bk \in \mathcal S_\ell} \;  \mathfrak  f(\bk )  \right) 
\end{equation}
for all  feature spaces  $\featspace$, all feature maps $\featmap: \data \mapsto \featspace$, and
 all $0 \le \ell \le L$.
\end{theorem}
The combinatorial quantities involved appear a bit daunting at a first glance, but, within the context of the proof, they all take a quite intuitive meaning. 
The heart of the proof involves the analysis of a measure of concentration that we call the permuted moment, and of an associated graph-cut problem. 
The combinatorial quantities arise quite naturally in the course of analyzing the graph cut problem.
  We provide a quick overview of the proof  in Section \ref{section:proof}, and refer to the appendix for full details. For now, it suffices to note that we have a formula (i.e. the right hand side of \eqqref{main_lower_bound}) that can be exactly evaluated with a few lines code. This formula provides a relatively tight lower bound for the generalization error. Theorem \ref{thm:special} is then a direct consequence --- plugging $L=9$, $n_w=150$, $n_c=5$,  $R=1000$ and  $\ell=7$ in  the right hand side of \eqqref{main_lower_bound} gives the claimed  $98.4\%$ lower bound. 

\section{Multiple Unfamiliar Sentences per Category} \label{section:extend_bound}

The  two previous sections were concerned with the case in which  each unfamiliar sequence of concepts has a \emph{single} representative in the training set. In this section we consider the  more general case in which each unfamiliar sequence of concepts has  $n^*$ representatives in the training set, see Figure \ref{figure:nstar}. Using   a simple union bound, inequality \eqqref{main_lower_bound} easily extends to this situation --- the resulting formula  is a bit cumbersome so we  present it in the appendix  (see Theorem \ref{thm:main2}). 
 In the concrete case where  $L=9$, $n_w=150$, $n_c=5$,  $R=1000$ this formula simplifies to
\begin{equation}\label{zia}
\overline{\text{err}}(\featspace, \psi, \mathfrak T) \ge 1 - 0.015\, n^* - 1/R \qquad \text{for all $\mathcal F$ and all $\psi$,}
\end{equation}
 therefore exhibiting an affine relationship between the error rate and the number $n^*$ of unfamiliar sentences per category.
\begin{figure}[t]
 \centering
   \includegraphics[totalheight=0.27\textheight]{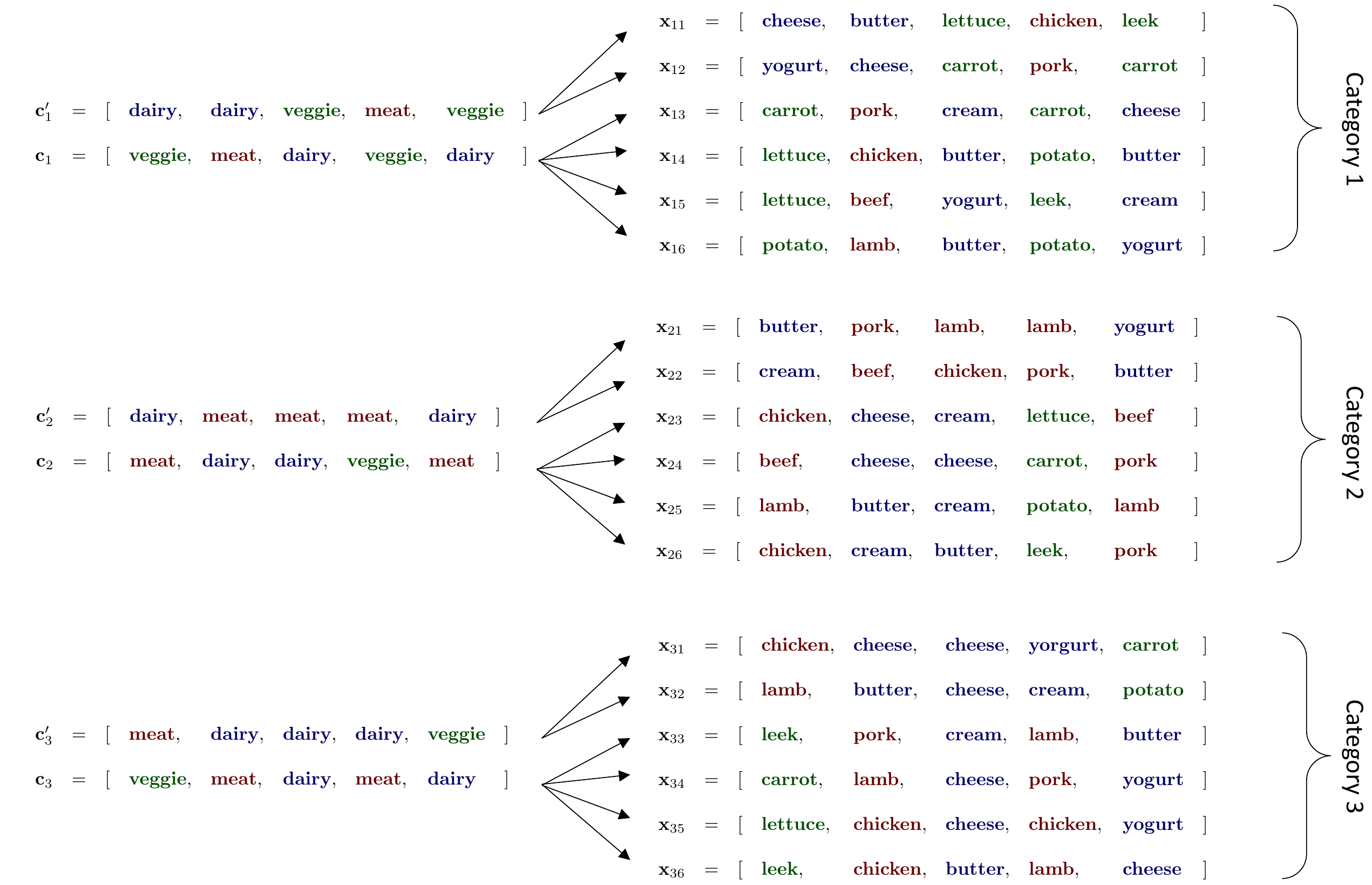}
     \caption{More general version of our data model with multiple unfamiliar sentences per category. The model parameters in this example are $L=5$, $n_w=12$, $n_c=3$, $R=3$,   $n_\text{spl} = 6$ and $n^*=2$.}  
     \label{figure:nstar}
\end{figure}
Note that choosing $n^*=1$ in \eqqref{zia} leads to a $98.4\%$ lower bound on the  error rate, therefore recovering the result from 
Theorem \ref{thm:special}. 
This lower bound then  decreases by $1.5\%$ with each additional unfamiliar sentence per category in the training set.

We would like to emphasize one more time the importance of non-asymptotic analysis in the  long-tailed learning setting.  For example, in inequality \eqqref{zia}, the difficulty lies in obtaining a value as small as possible for the coefficient in front of $n^*$.
We accomplish this via a careful analysis of the graph cut problem associated with our data model.


%
%
%


\section{Proof Outline --- Permuted Moment and Optimal Feature Map} \label{section:proof}

The proof involves two main ingredients. First, the key insight of our analysis is the realization that  generalization in our data model is closely tied to the \emph{permuted moment} of a probability distribution. To state this central concept, it will prove convenient to think of probability distributions on  $\data$ as vectors ${\bf p} \in \real^N$ with $N=|\data|$, together with indices $0 \leq i \leq N-1$ given by some arbitrary (but fixed) indexing of the elements of data space. Then $p_i$ denotes the probability of the $i^{{\rm th}}$ element of $\data$ in this indexing.  We use $S_N$ to denote the set of permutations of $\{0,1, \ldots, N-1\}$ and $\sigma \in S_N$ to refer to a particular permutation. The  $t^{{th}}$ permuted moment  of the probability vector ${\bf p} \in \real^N$ is
\begin{equation} \label{perm_mom}
  \entropy_{t}(\bp) \;\; := \;\;   \max_{\sigma \in S_{N}} \;\; \sum_{i=0}^{N-1}    \left(i/N\right)^{t}  \, p_{\sigma(i)}  
\end{equation}
Since   \eqqref{perm_mom} involves a maximum over all possible permutations, the definition clearly does not depend on the way the set $\data$ was indexed. In order to maximize the sum, the permutation $\sigma$ must match the largest values of $p_i$ with the largest values of  $(i/N)^t$, so the maximizing permutation simply orders the entries $p_i$ from smallest to largest. A very peaked distribution that gives large probability to only a handful of elements of $\data$ will have large permuted moment. Because of this, the permuted moment is akin to the negative entropy; it has large values for delta-like distributions and small values for uniform ones.  From definition  \eqqref{perm_mom} it is clear that  $0 \le \entropy_{t}(\bp) \le 1$ for all probability vectors $\bp$, and it is easily verified that the permuted moment is convex. These properties, as well as various useful bounds for the permuted moment, are presented and proven  in the appendix.

Second, we identify a specific  feature map, $\psi^{\star}: \data \to \featspace^\star$, which is optimal for a collection of tasks  closely related to the ones considered in our data model. Leveraging the optimality of $\psi^\star$ on these related tasks allows us to derive an error bound  that holds for the tasks of interest.
The feature map $\psi^\star$ is better understood through its associated kernel, which is given by the formula
\begin{equation} \label{def:Kstar}
K^\star(\bx,\by) = \; 
 \langle \psi^\star(\bx), \psi^\star(\by) \rangle_{\featspace^\star} =
 \frac{n_c^L}{n_w^L} \; \;
\frac{\big|\{\vphi \in \partition : \vphi(x_\ell) = \vphi(y_\ell)  \text{ for all } 1 \le \ell \le L\}\big|}{|\Phi|}.
\end{equation}
{Up to normalization, $K^\star(\bx,\by)$ simply counts the number of equipartitions of the vocabulary  for which sentences $\bx$ and $\by$ have the same underlying sequence of concepts.
Intuitively this makes sense, for the \emph{best possible} kernel must leverage the only information we have at hand. We know the general structure of the problem (words are partitioned into concepts) but not the partition itself. So to try and determine if sentences $(\bx, \by)$ were generated by the same sequence of concepts, the best we can do is to simply try all possible equipartitions of the vocabulary and count how many of them wind up generating $(\bx, \by)$ from the same underlying sequence of concepts. A high count makes it more likely that $(\bx, \by)$ were generated by the same sequence of concepts. The optimal kernel $K^\star$ does exactly this, and provides a good (actually optimal, see the appendix) measure of similarity between  pairs  of sentences.}

 For fixed $\bx \in \data$, the function $\by \mapsto K^{\star}(\bx, \by)$ defines a probability distribution on data space. The connection between generalization error, permuted moment, and optimal feature map, come from the fact that
\begin{equation}  \label{nonrobot}
\sup_{\featspace, \psi} \left[1 - \overline{\text{err}}(\featspace, \psi, \mathfrak T)\right] \leq \frac1{|\data|} \sum_{\bx \in \data} \mathcal{H}_{2R-1}\left(K^\star(\bx,\,\cdot)\right) + \frac1{R},
\end{equation}
and so, up to a small error $1/R$, it is the permuted moments of  $K^{\star}$ that determine the success rate. We then obtain the lower bound \eqqref{main_lower_bound} by studying these moments in great detail.  A simple union bound is then used to obtain inequalities such as \eqqref{zia}.

\section{Empirical Results} \label{section:empirical}

We conclude by presenting  empirical results that complement our theoretical findings.  The full details of these experiments 
(training procedure, hyperparameter choices, number of experiments ran to estimate the success rates, and standard deviations of these success rates), 
as well as additional experiments, can  be found in Appendix  \ref{section:ex}. {Codes are available at \url{https://anonymous.4open.science/r/Long_Tailed_Learning_Requires_Feature_Learning-17C4}.}

{\bf Parameter Settings.}
 We consider five parameter settings for the data model depicted in Figure \ref{figure:nstar}. Each setting corresponds to a column in Table \ref{thetablemain}.  
In all five settings, we set the parameters $L=9$, $n_w=150$, $n_c=5$ and $R=1000$ to the values for which the error bound \eqqref{zia} holds. We choose values for the  parameters $n_{\rm spl}$ and $n^*$ so that the $i^{th}$ column of the table corresponds to a setting in which the training set contains $5$ familiar and  $i$ unfamiliar sentences per category. Recall that $n_{\rm spl}$ is the total number of samples per category in the training set. So the first column of the table corresponds to a setting in which each category contains $5$ familiar sentences and $1$ unfamiliar sentence, whereas the last column corresponds to a setting in which each category contains $5$ familiar sentences and $5$ unfamiliar sentences.

{\bf Algorithms.} 
We evaluate empirically seven different algorithms.  The first two  rows of the table correspond to experiments in which the neural network in  Figure  \ref{figure:mlpmixer} is trained with SGD and constant learning rate. 
 At test time, we consider two different strategies to classify test sentences. 
 The first row of the table considers the usual situation in which the trained neural network is used to classify test points. The second row  considers the situation in which the trained neural network is only used to  extract features  (i.e. the concatenated words representation right before MLP2). The classification of test points is then accomplished by running a nearest neighbor classifier on these learned features.
The third  (resp. sixth) row of the table shows the results obtained when running a nearest neighbor algorithm (resp. SVM) on the features $\psi^\star$ of the optimal feature map. By the kernel trick, these algorithms only require the values of the optimal kernel $\langle \psi^\star(\bx), \psi^\star(\by) \rangle_{\featspace^\star}$, computed via \eqqref{def:Kstar}, and not the features $\psi^\star$ themselves. The fourth (resp. seventh) row shows results obtained when  running a nearest neighbor algorithm (resp. SVM) on features extracted by the simplest possible  feature map, that is
\begin{equation*}
\psi_{\rm one-hot}([x_1, \ldots, x_L]) = [{\bf e}_{x_1}, \ldots, {\bf e}_{x_L}]
\end{equation*}
where ${\bf e}_{x_\ell}$ denotes the one-hot-encoding of the $\ell^{th}$ word of the input sentence. 
Finally, the last row considers  a SVM with Gaussian Kernel (also called RBF kernel).  

 \begin{table}
  \caption{Success rate on unfamiliar test sentences.}  
  \label{thetablemain}
  \centering {\small
  \begin{tabular}{llllll}
    \toprule
              &    \small{$n^*\!=\!1$}   & \small{$n^*\!=\!2$}  &    \small{$n^*\!=\!3$}  & \small{$n^*\!=\!4$}  & \small{$n^*\!=\!5$}  \\
                      &    \small{\!$n_{\rm spl}\!=\!6$}   & \small{\!$n_{\rm spl}\!=\!7$}  &    \small{\!$n_{\rm spl}\!=\!8$}  & \small{\!$n_{\rm spl}\!=\!9$}  & \small{\!\!$n_{\rm spl}\!=\!10$}  \\
        \midrule
    {\small Neural network in  Figure  \ref{figure:mlpmixer}}     &  $99.8\% $ &  $99.9\%  $   & $99.9\% $ & $99.9\% $ & $100\% $   \\
     {\small Nearest neighb. on features learned by  neural net}   &  $99.9\% $ &  $99.9\% $   & $99.9\% $& $99.9\% $&$99.9\% $  \\
     \midrule
     {\small Nearest neighb. on features extracted by $\psi^\star$}     &  $0.7\%$ & $1.1\%$ & $1.5\%$ & $1.8\%$ & $2.2\%$ \\ 
         {\small Nearest neighb. on features extracted by $\psi_{\rm one-hot}$}     &                                 $0.6\%$ &  $1.1\%$ & $1.4\%$ & $1.7\%$ & $2.1\%$  \\
      {\small Theoretical upper  bound ($0.015 n^* + 1/1000$)}    &  $1.6\%$ & $3.1\%$    &$4.6\%$ & $6.1\%$  & $7.6\%$  \\
     \midrule
           {\small     SVM  on features extracted by $\psi^\star$}       &  $0.6\% $    &  $1.5\%  $  &$2.2\% $& $3.2\% $ & $4.2\% $\\
        {\small SVM  on features extracted by $\psi_{\rm one-hot}$}      &  $0.5\% $ & $1.1\%$    &$1.9\% $ & $2.8\% $  & $3.8\% $  \\
        {\small  SVM with Gaussian kernel}      &   $0.6\% $    &  $1.1\% $ &$2.0\% $&$2.8\% $& $3.6\% $ \\
    \bottomrule
  \end{tabular} }
\end{table}



{\bf Results.} The first two rows of the table correspond to algorithms that \emph{learn} features from the data; the remaining rows correspond to algorithms that use a pre-determined (not learned) feature map. Table   \ref{thetablemain} reports the success rate of each algorithm on unfamiliar test sentences. A crystal-clear pattern emerges. Algorithms that learn features  generalize almost perfectly, while algorithms that do not learn features catastrophically fail. Moreover, the specific classification rule matters little. For example,  replacing MLP2 by a  nearest neighbor classifier on the top of features learned by the neural network leads to equally accurate results. Similarly, replacing the nearest neighbor classifier by a SVM on the top of features extracted by $\psi^\star$ or $\psi_{\rm one-hot}$ leads to almost equally poor results. The only thing that matters is whether or not the features are learned.
Finally, inequality  \eqqref{zia} gives an upper bound of  $0.015 n^* + 1/1000$ on the success rate of the nearest neighbor classification rule applied on the top of \emph{any possible feature map} (including $\psi^\star$ and $\psi_{\rm one-hot}$). The fifth row of Table \ref{thetablemain} compares this bound against the empirical accuracy obtained with $\psi^\star$ and $\psi_{\rm one-hot}$, and the comparison shows that our theoretical upper bound is relatively tight. 

{When $n^*=1$ our main theorem states that no feature map can succeed more than $1.6\%$ of the time on unfamiliar test sentences (fifth row of the table). At first glance this appears to contradict the empirical performance of the feature map extracted by the neural network, which succeeds $99\%$ of the time (second row of the table). The resolution of this apparent contradiction lies in the order of operations.
The point here is to separate \emph{hand crafted} or \emph{fixed} features from \emph{learned} features via the order of operations.
 If we choose the feature map \emph{before} the random selection of the task then the algorithm performs poorly since it uses unlearned, task-independent features. By contrast, the neural network learns a feature map from the training set, and since the training set is generated by the task, this process takes place \emph{after} the random selection of the task.  It therefore uses task-dependent features, and the network performs almost perfectly for the specific task that generated its training set. 
But by our main theorem, it too must fail if the task changes but the features do not. }



%


{\bf Acknowledgements}

Xavier Bresson is supported by NRF Fellowship NRFF2017-10 and NUS-R-252-000-B97-133.
  
\bibliography{bibliography}
\bibliographystyle{plain}


\newpage
\appendix

\begin{center}
{\Large \bf Appendix }
\end{center}

\

In Section \ref{moment} we prove a few elementary properties of the permuted moment \eqqref{perm_mom}. 
Section \ref{sectionA} is devoted to the proof of inequality \eqqref{nonrobot}, which we restate here for convenience:
\begin{equation} \label{connection}
\sup_{\mathcal F, \psi} \left[1 - \overline{\text{err}}(\mathcal F, \psi, \mathfrak T)\right] \leq \frac1{|\data|} \sum_{\bx \in \data} \mathcal{H}_{2R-1}\left(K^\star(\bx,\,\cdot)\right) + \frac1{R}
\end{equation}
where  the collection of tasks $\mathfrak T = \Phi \times \mathcal Z^{2R}$ consists in \emph{all possible tasks} that one might encounter.
 Inequality \eqqref{connection} plays a central role in our work as it establishes the connection between the generalization error,  the permuted moment, and the optimal kernel $K^\star$ defined by \eqqref{def:Kstar}. The proof is non-technical and easily accessible.
 In Section \ref{sectionC} we provide the following upper bound on the permuted moment of the optimal kernel:
 \begin{equation} \label{upper_bound_pm}
\frac{1}{|\data|} \sum_{\bx \in \data} \entropy_{2R-1}\left(K^\star(\bx,\,\cdot)\right)  \le  \left( 1 -  \sum_{\bk \in \mathcal S_\ell} \fcount(\bk)  \mathfrak g(\bk)  \right)  + \frac{1}{2R} \left(  \max_{\bk \in \mathcal S_\ell} \fcount(\bk) \right) 
\end{equation}
for all $0 \le \ell \le L.$
 The proof is combinatorial in nature, and involves the analysis of a graph-cut problem.  Combining \eqqref{connection} and \eqqref{upper_bound_pm} establishes Theorem  \ref{thm:main}. In Section \ref{appendix:multiple} we consider  the  case in which each unfamiliar sequence of concepts has  $n^*$ representatives in the training set. A simple union bound shows that, in this situation, 
inequality  \eqqref{connection} becomes
\begin{equation} \label{connection11}
\sup_{\mathcal F, \psi} \left[1 - \overline{\text{err}}(\mathcal F, \psi, \mathfrak T)\right] \leq \frac{n^*}{|\data|} \sum_{\bx \in \data} \mathcal{H}_{2R-1}\left(K^\star(\bx,\,\cdot)\right) + \frac1{R}
\end{equation}
 Combining \eqqref{connection11} and \eqqref{upper_bound_pm} then provides our most general error bound, see Theorem \ref{thm:main2}. Inequality \eqqref{zia} in the main body of the paper is just a special case of Theorem \ref{thm:main2}.
  Finally, in Section \ref{section:ex}, we provide the full details of the experiments.


\section{Properties of the Permuted Moment} \label{moment}
The permuted moment, in Section \ref{section:proof}, was defined for probability vectors only. 
It will prove convenient to consider the permuted moment of nonnegative vectors as well. We denote by  $\real_+ = [0,+\infty)$ the nonnegative real numbers, and by $\real_+^N$ the vectors with $N$ nonnegative real entries indexed  from $i=0$ to $i=N-1$.
The   permuted moment of
 $\bu \in \real_+^N$ is then given by
 \begin{equation} \label{zazo}
  \entropy_{t}(\bu) \;\; := \;\;   \max_{\sigma \in S_N} \sum_{i=0}^{N-1}    \left(i/N\right)^{t}  \,  u_{\sigma(i)}.
\end{equation}
where
$S_N$ denote the set  of permutations of $\{0,1, \ldots, N-1\}$. The concept of an \emph{ordering permutation} will prove useful in the next lemma.
\begin{definition}  \label{def1}
 $\sigma \in S_N $ is said to be an ordering permutation of $\boldu \in \real^N$ if
\begin{align} \label{sequence_of_ineq}
 u_{\sigma(0)} \le  u_{\sigma(1)}    \le \ldots \le  u_{\sigma(N-1)} .
 \end{align}
\end{definition}
The lemma below shows that the permutation maximizing \eqqref{zazo} is the one that sorts
the entries $u_i$ from smallest to largest. 
\begin{lemma} \label{matching} Let $\boldu \in \real_+^N$ and let $\sigma^*$ be an ordering permutation of $\bu$. Then
 \begin{align}
\sigma^* \;\; \in \;\;  \arg \max_{\sigma \in S_N}  \;\; \sum_{i=0}^{N-1}    \left(i/N\right)^{t}  \,  u_{\sigma(i)} . \label{fifo}
  \end{align}
\end{lemma}
\begin{proof} The optimization problem \eqqref{fifo} can be formulated as finding a pairing between the $u_i$'s and the  $(i/N)^t$'s that maximizes the sum of the product of the pairs.
 An ordering permutation of $\boldu$ corresponds to pairing the smallest entry of  $\boldu$ to $(0/N)^t$, the second smallest entry to $(1/N)^t$,  the third smallest entry  to $(2/N)^t$, and so forth. This pairing is clearly optimal.
\end{proof}
In light of the previous lemma, we see that computing the permuted moment of a vector $\bu$ can be accomplished as follow: 1) sort the entries of $\bu$ from smallest to largest; 2) compute the dot product between this sorted vector and the vector 
\begin{equation} \label{gogi}
\begin{bmatrix} \left(\frac{0}{N}\right)^t &  \left(\frac{1}{N}\right)^t & \left(\frac{2}{N}\right)^t & \ldots &  \left(\frac{N-1}{N}\right)^t\end{bmatrix}.\end{equation}
Let us now focus on  the case where $\bu$ is a probability distribution. If $\bu$ is very peaked, it must have a large permuted moment since, after sorting, most of the mass concentrates on the high values of \eqqref{gogi} located on the right. On the contrary, if $\bu$ is very spread, it must have small permuted moment since it `wastes' its mass on small values of \eqqref{gogi}.  
Because of this, the permuted moment is akin to the negative entropy; it has large values for delta-like distributions and small values for uniform ones. 

We now show that  the permuted moment is subaddiditive and one-homogeneous on $\real_+^N$
 (as a consequence it is convex on the set of probability vectors) and we derive some elementary  $\ell_1$ and $\ell_\infty$ bounds.
 We denote by  $\|\bu\|_p$  the $\ell_p$-norm of a vector $\bu$.
In particular, if $\bu \in \real^N_+$, we have
$$
\|\bu\|_1 := \sum_{i=0}^{N-1} u_i \qquad \text{ and } \qquad \|\bu\|_\infty := \max_{0 \le i \le N-1}  u_i.
$$
With this notation in hand, we can now state our lemma:
\begin{lemma} \label{very_basic}
\begin{enumerate}
\item[(i)]  $\entropy_{t}(\bu+\bv) \le   \entropy_{t}(\bu) +  \entropy_{t}(\bv)$ \;for all $\bu,\bv \in \real_+^N$.
\item[(ii)]  $\entropy_{t}(c\, \bu) =  c  \, \entropy_{t}(\bu) $ \; for all $\bu\in \real_+^N$ and all $c \ge 0$.
\item[(iii)] $\entropy_{t}(\bu) \le  \|\bu\|_1 $\;  for all $\bu\in \real_+^N$.
\item[(iv)] $\entropy_{t}(\bu) \le \frac{N}{t+1} \|\bu\|_\infty $ \;  for all $\bu\in \real_+^N$.
\end{enumerate}
\end{lemma}
\begin{proof}
Properties (i) and (ii) are obvious. To prove (iii) and (iv),  define $w_i = (i/N)^t$ and note that
$$
\|\bw\|_\infty \le 1 \qquad \text{and} \qquad \|\bw\|_1 = N  \left( \frac{1}{N}  \sum_{i=0}^{N-1}    \left(i/N\right)^{t} \right)   \le N    \int_0^1 x^t dt = \frac{N}{t+1}
$$
Then (iii) comes from 
 $\entropy_{t}(\bu) \le \|\bw\|_\infty \|\bu\|_1$ whereas (iv) comes from $\entropy_{t}(\bu) \le \|\bw\|_1 \|\bu\|_\infty.$
\end{proof}
We conclude this section with a slightly more sophisticated bound that holds for probability vectors --- this bound  will play a central role in 
 Section \ref{sectionC}.
\begin{lemma} \label{lambda5} Suppose $\bp \in \real_+^N$,  and suppose  $\sum_{i=1}^N p_i =1.$  Then
$$
\entropy_t(\bp) \le  \left(1-  \sum_{i=0}^{N-1} \min \{p_i,\lambda\} \right) +  \frac{\lambda N}{t+1}  \qquad \text{ for all } \lambda\ge 0.
$$ 
\end{lemma}
\begin{proof} Fix a $\lambda \ge 0$ and define the vectors $\bu$ and $\bv$ as follow:
$$
u_i = \min \{p_i,\lambda\}  \qquad  \text{ and} \qquad v_i =  p_i - \min \{p_i,\lambda\}  \qquad \text{ for all } 0 \le i \le N-1
$$
Note that this two vectors are non-negative and sum to $\bp$. We can therefore use Lemma \ref{very_basic} to obtain 
$$
\entropy_t(\bp) = \entropy_t(\bu + \bv) \le \entropy_t(\bu) + \entropy_t(\bv)  \le \frac{N }{t+1} \|\bu\|_\infty + \|\bv\|_1
$$
To conclude, we note that $\| \bu \|_\infty \le \lambda$, and $\|\bv\|_1 = 1-  \sum_{i=0}^{N-1} \min \{p_i,\lambda\}.$
\end{proof}

\section{Permuted Moment of $K^\star$ and  Generalization Error} \label{sectionA}

This section is devoted to the proof of  inequality \eqqref{connection}.
We start by recalling  a few definitions. The vocabulary,  set of concepts,  data space, and  latent space are
 $$\mathcal V=\{1,\ldots, n_w \}, \qquad   \mathcal C=\{1,\ldots, n_c \}, \qquad  \data = \mathcal V^L \qquad \text{ and } \qquad \mathcal Z = \mathcal C^L$$ 
 respectively.
 Elements of $\mathcal X$ are sentences of $L$ words and they take the form  $\bx = [x_1, x_2, \ldots,x_L],$ while elements of $\mathcal Z$ take the form $\bc = [c_1,c_2, \ldots, c_L]$ and correspond to sequences of concepts.    We also recall that  the collection of all equipartitions of the vocabulary is denoted by
$$
\Phi = \big\{ \text{All functions $\vphi$ from $ \mathcal V$ to  $\mathcal C$ that satisfy $|\vphi^{-1}(\{c\})|=s_c$ for all $c$   } \big\}
$$
where
$
 s_c := n_w/n_c
 $
denote the size of the concepts. 
Given $\vphi \in \Phi$,  we denote by $\mathring \varphi: \data \to \concept$ the function
$$
\ephi\big([x_1,\, x_2,\,\ldots,\,x_L]\big) :=  \big[\vphi(x_1),\,\vphi(x_2),\,\ldots,\,\vphi(x_L)\big]
$$
that operates on sentences element-wise.  The informal statement ``the sentence $\bx$ is randomly generated by the sequence of concepts $\bc$'' means that $\bx$ is sampled uniformly at random from the set 
\begin{equation} \label{pre-image}
\ephi^{-1}(\{\bc\}) = \{\bx \in \data: \ephi(\bx) = \bc\}.  
\end{equation}
We will often do the abuse of notation of writing $\ephi^{-1}(\bc)$ instead of  $\ephi^{-1}(\{\bc\})$. We now formally define the sampling process associated with our main data model.

{\bf Sampling Process DM:}

\noindent\fbox{%
    \parbox{\textwidth}{%
\begin{enumerate}[label=(\roman*)]
 \item Sample  $\mathcal T = (\; \vphi \;  ; \;   \bc_1, \ldots, \bc_R \; ; \;  \bc'_1, \ldots, \bc'_R \; ) $ uniformly at random in  $\mathfrak T=\Phi \times \mathcal Z^{2R}$.
 \item For $r=1,\ldots, R$:
 \begin{itemize}
  \item Sample $(\bx_{r,1} , \ldots,  \bx_{r,n_{\rm unf}})$ uniformly at random in $\ephi^{-1}(\bc'_r) \times \ldots \times \ephi^{-1}(\bc'_r)$. %
 \item Sample $(\bx_{r,n_{\rm unf}+1} , \ldots,  \bx_{r,n_{\rm spl} })$ uniformly at random in $\ephi^{-1}(\bc_r) \times \ldots \times \ephi^{-1}(\bc_r)$. 
 \end{itemize}
 \item Sample  $\bx^\text{test}$  uniformly at random in $\ephi^{-1}(\bc'_1)$.
\end{enumerate}
}
}

Step (i) of the above sampling process consists in selecting at random a task $\mathcal T$ among all possible tasks. Step (ii) consists in generating a training set\footnote{We refer to $S$ as the `training set'. In our formalism however  it is not a set, but an element of $\data^{R\times n_{spl}}$.} 
$S \in \data^{R\times n_{spl}}$
 exactly as depicted on Figure \ref{figure:food}: each  unfamiliar sequence of concept $\bc_r'$ generates $n_{\rm unf}$ sentences, whereas each familiar sequence of concept $\bc_r$ generates 
$n_{\rm fam}$ sentences (recall that the number of samples per category is $n_{\rm spl}= n_{\rm unf} + n_{\rm fam}$).
 Finally step (iii) consists in randomly generating an unfamiliar test sentence $\bx^{\text{test}} \in \data$. Without loss of generality we assume that this test sentence is generated by the unfamiliar sequence of concept $\bc_1'$. 

We denote by
$
\varrho_{ {\rm DM}}
$
the p.d.f. of the sampling process DM. This function is defined on the sample space   
$$ \Omega_{\rm DM} := \left(\Phi \times \concept^{2R}\right)  \times \data^{R\times n_{spl}} \times \data \;.$$
A sample from $\Omega_{\rm DM}$ takes the form
 \begin{multline*}
 \omega = \Big(\underbrace{\vphi\;; \; \bc_1, \ldots, \bc_R \; ; \;   \bc'_1,\ldots, \bc'_R}_{\text{The task}} \; ;  \; \underbrace{\bx_{1,1}, \ldots,\bx_{1,\nspl} \; ; \;    \ldots \; ; \;  \bx_{R,1}, \ldots,\bx_{R,\nspl}}_{\text{The training sentences}} \; ; \; \underbrace{\bxt}_{\text{The test sentence}}   \Big) 
 \end{multline*}
 and we have the following formula for $\varrho_{ {\rm DM}}$
 \begin{equation} \label{rhodm}
\varrho_{ {\rm DM}}(\omega) := \frac{1}{|\Phi||\concept|^{2R}}\prod^{R}_{r=1}\left(\frac{\ones_{\ephi^{-1}(\bc'_{r})}(\bx_{r,\,1})}{\left|\ephi^{-1}(\bc'_{r})\right|}\prod^{n_\text{spl}}_{s=2}\frac{\ones_{\ephi^{-1}(\bc_{r})}(\bx_{r,\,s})}{\left|\ephi^{-1}(\bc_{r})\right|}\right)\frac{\ones_{\ephi^{-1}(\bc'_{1})}(\bx^\text{test})}{\left|\ephi^{-1}(\bc'_{1})\right|}
\end{equation}
where $\ones_{\ephi^{-1}(\bc_r)}$ and $\ones_{\ephi^{-1}(\bc_r')}$  denote the indicator functions of the set $\ephi^{-1}(\bc_r)$ and $\ephi^{-1}(\bc_r)$ respectively. Let us compute a few marginals of $\varrho_{ {\rm DM}}$ in order to verify that it is indeed the p.d.f. of the sampling process DM. Writing  $\omega=(\mathcal T, S,\bxt)$, summing over the variables $S$ and $\bxt$, and using the fact that $
\sum_{\bx\in\data} \ones_{\ephi^{-1}(\bc)}   (\bx) = |\ephi^{-1}(\bc)|,
$
we obtain
 $$
 \sum_{S \in  \concept^{2R} }  \sum_{\bxt \in  \data } \varrho_{ {\rm DM}} (\mathcal T, S, \bx^{\text{test}}) = \frac{1}{|\Phi||\concept|^{2R}}. 
 $$
This  shows that each task $\mathcal T$ is equiprobable. Summing over the variable $S$ gives
$$
\sum_{S \in  \concept^{2R} } \varrho_{ {\rm DM}} (\mathcal T, S, \bx^{\text{test}}) =  \frac{1}{|\Phi||\concept|^{2R}} \frac{\ones_{\ephi^{-1}(\bc'_{1})}(\bx^\text{test})}{\left|\ephi^{-1}(\bc'_{1})\right|}
 $$
This shows that,  given a task $\mathcal T$, the test sentence $\bxt$ is obtained by sampling uniformly at random from $\ephi^{-1}(\bc'_{1})$.  A similar calculation shows that,  given a task $\mathcal T$, the train sentence  $\bx_{r,s}$  is obtained by sampling uniformly at random from $\ephi^{-1}(\bc_{r})$ if $s\ge 1$, and from $\ephi^{-1}(\bc_{r}')$ if $s=1$.

  Given a feature space $\featspace$ and a feature map $\psi: \data \to \featspace$, we define the events  $E_{\featspace,\psi} \subset \Omega_{\rm DM}$ as follow:
  \begin{multline}
 E_{\mathcal F, \psi} = \Big\{ \omega \in \Omega_{\rm DM} : \;\; \text{There exists } 1 \le s^* \le \nspl \text{ such that } \\ 
  \left\langle \psi(\bx^\text{test}), \psi(\bx_{1,s^*}) \right\rangle_\featspace >  \left\langle\psi(\bx^\text{test}), \psi(\bx_{r,s}) \right\rangle_\featspace \text{ for all } 2 \le r \le R \text{ and all } 1 \le s \le \nspl 
 \Big\}.
 \end{multline}
 Note that this event consists in all the outcomes $\omega = (\mathcal T, S, \bxt)$ for which the feature map $\psi$ associates the test sentence $\bxt$ to a train point $\bx_{r,s}$ from the first category. Since by construction,  $\bxt$ belongs to the first category,  $E_{\mathcal F, \psi}$ consists in all the  outcomes for which the nearest neighbor classification rule  is `successful'.
   As a consequence, when $\mathfrak T = \Phi \times \mathcal Z^{2R}$,  the generalization error  can be expressed as
  \begin{equation} \label{formal_error_def}
  \overline{\text{err}}(\featspace, \psi, \mathfrak T) =  1 -
   \mathbb P_{{\rm DM}}\Big[  E_{\mathcal F, \psi} \Big]
\end{equation}
where  $\mathbb P_{\rm DM}$  denote the probability measure on $\Omega_{\rm DM}$ induced by $\varrho_{\rm DM}$. Equation \eqqref{formal_error_def} should be viewed as our `fully formal' definition of the quantity $ \overline{\text{err}}(\featspace, \psi, \mathfrak T)$, as opposed to the more informal definition given earlier by equations \eqqref{error22} and \eqqref{expected_gen}.

The goal of this section is to prove inequality \eqqref{connection}, which, in light of  \eqqref{formal_error_def} is equivalent to 
\begin{equation} \label{zoro}
\sup_{\mathcal F, \psi}  \; \mathbb P_{{\rm DM}}\Big[  E_{\mathcal F, \psi} \Big] \le   \frac1{|\data|} \sum_{\bx \in \data} \mathcal{H}_{2R-1}\left(K^\star(\bx,\,\cdot)\right) + \frac1{R}.
\end{equation}
which  in turn equivalent to
\begin{equation} \label{zoro2}
\sup_{ \substack{K: \data \times \data \to \real \\ K \text{ pos. semi-def.} }}   \mathbb P_{{\rm DM}}\Big[  E_K \Big] \le   \frac1{|\data|} \sum_{\bx \in \data} \mathcal{H}_{2R-1}\left(K^\star(\bx,\,\cdot)\right) + \frac1{R}
\end{equation}
where the event $E_K$ is defined by
\begin{multline}
E_K = \Big\{ \omega \in \Omega_{\rm DM} : \;\; \text{There exists } 1 \le s^* \le \nspl \text{ such that } \\  
  K(\bx^\text{test}, \bx_{1,s^*})  >  K(\bx^{\text{test}}, \bx_{r,s})  \text{ for all } 2 \le r \le R \text{ and all } 1 \le s \le \nspl 
 \Big\}
 \end{multline}
 and where the supremum is taken over all kernels $K: \data \times \data \to \real$ which are symmetric positive semidefinite.
 We will actually prove a slightly stronger result, namely
 \begin{equation} \label{albator}
\sup_{ \substack{K: \data \times \data \to \real \\ K \text{is symmetric} }}   \mathbb P_{{\rm DM}}\Big[  E_K \Big] \le   \frac1{|\data|} \sum_{\bx \in \data} \mathcal{H}_{2R-1}\left(K^\star(\bx,\,\cdot)\right) + \frac1{R}
\end{equation}
where  the supremum is taken over all functions
$K: \data \times \data \to \real$ that satisfy $K(\bx,\by)=K(\by,\bx)$ for all $(\bx,\by) \in \data \times \data.$
The rest of the section is devoted to proving  \eqqref{albator}.

 In Subsection \ref{section:firstpart} we start by considering a  simpler  data model --- for this simpler data model we are able to show  that the function $\psi^\star$ implicitly defined by \eqqref{def:Kstar}  is the best possible feature map (we actually only work with the associated kernel $K^\star$, and  never need $\psi^\star$ itself). We also show that the success rate is \emph{exactly equal} to the permuted moment of $K^\star$  --- see Theorem \ref{theorem:beautiful}, which is is the central result of this section. In the remaining subsections, namely Subsection \ref{section:marginal} and  Subsection \ref{section:secondpart}, we leverage the bound obtained for the simpler data model in order to obtain bound  \eqqref{albator} for the main data model. These two subsections are mostly notational. The core of the analysis takes place in  Subsection \ref{section:firstpart}.


\subsection{A Simpler Data Model} \label{section:firstpart}

We start by presenting the sampling process associated with our simpler datamodel.

{\bf Sampling Process SDM:} 

\noindent\fbox{%
    \parbox{\textwidth}{%
\begin{enumerate}[label=(\roman*)]
	\item Sample $\vphi$  uniformly at random in $\Phi$. Sample $\bc_1, \bc_2, \ldots, \bc_{t+1}$ uniformly at random in $\mathcal Z$.
	\item For $1 \le r \le t+1$: Sample $\bx_{r}$ uniformly at random in $\ephi^{-1}(\bc_r)$.
	\item Sample  $\bxt$ uniformly at random in $\ephi^{-1}(\bc_1)$.
\end{enumerate}
}}

The function
\begin{equation} \label{rhosdm}
\varrho_{\rm \SDM}(\vphi;\bc_1, \ldots, \bc_{t+1}; \bx_1, \ldots, \bx_{t+1}; \bx^{\text{test}}) := \frac{1}{|\Phi||\concept|^{t+1}}\left(\prod^{t+1}_{r=1}\frac{\ones_{\ephi^{-1}(\bc_{r})}(\bx_{r})}{\left|\ephi^{-1}(\bc_{r})\right|} \right)\frac{\ones_{\ephi^{-1}(\bc_{1})}(\bx^\text{test})}{\left|\ephi^{-1}(\bc_{1})\right|}
\end{equation}
on $\Omega_{\rm SDM} := \Phi \times \concept^{t+1} \times \data^{t+2}$ is the p.d.f. of  the above  sampling process. We  use $\mathbb P_{\rm SDM}$ to denote the probability measure on $\Omega_{\rm SDM}$ induced by this function.  The identity in the next theorem is the central result of this section.
\begin{theorem} \label{theorem:beautiful} Let $\mathcal K$ denote the set of all symmetric functions from $\data\times\data$ to $\real$. Then 
  \begin{equation} \label{beautiful_formula_appendix}
\sup_{K \in \mathcal K} \;\; \mathbb P_{\rm SDM} \Big[  K(\bxt,\bx_1) > K(\bxt,\bx_r) \text{ for all } 2 \le r \le t+1 \Big] = \frac{1}{|\data|} \sum_{\bx \in \data} \mathcal H_{t}( K^\star_\bx )
\end{equation}
\end{theorem}
In \eqqref{beautiful_formula_appendix}, $K_\bx^\star$ stands for the function $K(\bx, \cdot)$.
Theorem \ref{theorem:beautiful} establishes an intimate connection between the permuted moment and the ability of any fixed feature map (or equivalently, any fixed kernel)  to generalize well in our framework. The sampling process considered in this theorem involves two points, $\bx^{\text{test}}$ and $\bx_1$, generated by the same sequence of concepts $\bc_1$,  and $t$ `distractor' points $\bx_2, \, \ldots, \bx_{t+1}$ generated by different sequences of concepts. Success for the kernel $K$ means correctly recognizing that $\bx^{\text{test}}$ is more `similar' to $\bx_1$ than any of the distractors, and the success rate in \eqqref{beautiful_formula_appendix} precisely quantifies its ability to do so as a function of the number $t$ of distractors. The theorem shows that the probability of success for the \emph{best possible kernel} at this task is \emph{exactly equal} to the averaged $t^{th}$-permuted moment of  $K^\star_\bx$, so it elegantly quantifies the  generalization ability of the best possible fixed feature map in term of the permuted moment. We also provide an  explicit construction for a kernel $K(\bx,\by)$ that achieves the supremum in \eqqref{beautiful_formula_appendix} --- First, choose a kernel $\varepsilon(\bx,\by)$ that satisfies
\begin{enumerate}
\item[(i)] $\varepsilon(\bx,\by) \neq \varepsilon(\bx,\bz)$  for all $\bx,\by,\bz \in \data$ with $\by \neq \bz$. 
\item[(ii)] $0 \le \varepsilon(\bx,\by) \le 1$ for all $\bx,\by \in \data$.
\end{enumerate} 
and then define the following perturbation
\begin{equation} \label{add_noise}
K(\bx,\by) =  K^\star(\bx,\by)+  \varepsilon(\bx,\by) / (2s_c^L|\partition|)
\end{equation}
of  $K^\star$. Any such kernel is a maximizer of the optimization problem in \eqqref{beautiful_formula_appendix}, so we may think of perturbations of $K^\star$ as bona-fide optimal. 

The rest of this subsection is devoted to the proof of Theorem \ref{theorem:beautiful}, and we also show, in the course of the proof, that \eqqref{add_noise}   is a maximizer of the optimization problem in \eqqref{beautiful_formula_appendix}.  
We use $\mathcal K$ to denote the set of all symmetric functions from $\data\times\data$ to $\real$. We will refers to such functions as `kernel' despite the fact that these functions are not necessarily positive semi-definite.

Proving Theorem  \ref{theorem:beautiful} requires that we study the following optimization problem:
  \begin{align}
  &\text{Maximize } \;\;   \Energy(K) := \mathbb P_{SDM} \Big[  K(\bx^{\text{test}},\bx_1) > K(\bx^{\text{test}},\bx_r) \text{ for all } 2 \le r \le t+1 \Big]  \label{opt1} \\
  &\text{over all kernels $K \in \mathcal K$.} \label{opt2}
  \end{align}
  We recall the definition of the optimal kernel,
  \begin{equation} \label{def:Kstar0}
K^\star(\bx,\by) = \; 
 \frac{1}{s_c^L} \; \;
\frac{\big|\{\vphi \in \partition : \vphi(x_\ell) = \vphi(y_\ell)  \text{ for all } 1 \le \ell \le L\}\big|}{|\Phi|}
\end{equation}
where $s_c = n_w/s_c$ denotes the size of a concept. 
We start with the following simple lemma:
\begin{lemma} \label{lemma:probdist}
The function $K^\star_\bx(\cdot) = K^\star(\bx, \cdot)$ is a probability distribution on $\data$.
 \end{lemma}
 \begin{proof}
 First note that $K^\star$ can be written as
  \begin{equation} \label{def:Kstar_appendix}
K^\star(\bx,\by) =\frac{1}{s_c^L}\frac{ |\{\vphi \in \partition : \ephi(\bx) = \ephi(\by) \}|}{|\partition|} =  \frac{1}{s_c^L |\partition| }  \sum_{\vphi \in \partition}  \ones_{\{ \ephi(\bx)  =\ephi( \by) \}} 
\end{equation}
  Since   $\vphi$ maps exactly $s_c$ words to each concept $c \in \{1,\ldots, n_c\}$, we have that
 \begin{equation} \label{scl}
 |\{\bx \in \mathcal X:  \ephi(\bx)  =  \bc \}| = s_c^L \qquad \text{ for all $\bc \in \concept$.}
  \end{equation}
  Therefore
 $$
\sum_{\by \in \data }K^\star(\bx,\by) =  \frac{1}{s_c^L |\partition| }  \sum_{\vphi \in \partition}  \sum_{\by \in \data }   \ones_{\{ \ephi(\bx)  =\ephi( \by) \}}  =     \frac{1}{s_c^L|\partition| }  \sum_{\vphi \in \partition} |\{\by \in \mathcal X:  \ephi(\by)  = \ephi(\bx) \} |   = 1
$$
\end{proof}
We now show that the marginal of the  p.d.f.  $\varrho_{\rm SDM}$ is related to $K^\star$.
  \begin{lemma} \label{vitevite} For all $\bx_1,\ldots, \bx_{t+1}$ and  $\bxt$  in $\data$ we have
  $$
   \sum_{\vphi \in \partition}  \;  \sum_{\bc_1 \in \concept}  \cdots \sum_{\bc_{t+1}\in \concept}  
\varrho_{\rm SDM} (\vphi;\bc_1, \ldots, \bc_{t+1}; \bx_1, \ldots, \bx_{t+1}; \bx^{\text{test}})    =  \frac{1}{|\data|^{t+1}} K^\star(\bx_1,\bxt).
  $$
  \end{lemma}
  \begin{proof} Identity \eqqref{scl} can be expressed as  $\left|\ephi^{-1}(\bc)\right| = s_c^L$ for all $\bc \in \concept$. As a consequence, definition \eqqref{rhosdm} of $\varrho_{\rm SDM}(\omega)$ simplifies to
\begin{equation} \label{jijilamagie}
 \varrho_{\rm SDM}(\omega)  = \alpha   \left(  \prod^{t+1}_{r=1} \ones_{\ephi^{-1}(\bc_{r})}(\bx_{r}) \right)  \ones_{\ephi^{-1}(\bc_{1})}(\bx^\text{test})  
\end{equation}
where the constant $\alpha$ is given by
$$
\alpha = \frac{1}{|\Phi||\concept|^{t+1}  s_c^{L(t+2)}} =  \frac{1}{|\Phi| n_c^{L(t+1)}  s_c^{L(t+2)}} =   \frac{1}{|\Phi| |\data|^{t+1}  s_c^{L}} 
$$
In the above we have used the fact that $|\concept|=n_c^L$ and  $|\data| = n_w^L$.
We then note that the identity
$
\ones_{\ephi^{-1}(\bc)}(\bx) = \ones_{\{\ephi(\bx)=\bc\}} 
$
 implies  
 \begin{align}
 &  \sum_{\bc \in \concept}  \ones_{\ephi^{-1}(\bc)}(\bx)=  \sum_{\bc \in \concept}  \ones_{\{\ephi(\bx)=\bc\}}   =1  \label{jiji1} \\
& \sum_{\bc \in \concept} \Big( \ones_{\ephi^{-1}(\bc)}(\bx) \; \ones_{\ephi^{-1}(\bc)}(\by) \Big)
  \;\; = \;\;  \sum_{\bc \in \concept} \Big(  \ones_{\{\ephi(\bx)=\bc\}} \; \ones_{\{\ephi(\by)=\bc\}} \Big)
 \;\; = \;\;   \ones_{\{\ephi(\bx)=\ephi(\by)\}}   \label{jiji2}
 \end{align}
 for all $\bx,\by \in \data$.
  Summing  \eqqref{jijilamagie} over the variables $\bc_1,\ldots,\bc_{t+1}$ we obtain
   \begin{align*}
   \sum_{\bc_1 \in \concept}  \cdots& \sum_{\bc_{t+1}\in \concept}   
\varrho_{\rm SDM} (\omega)  
=   \alpha    \sum_{\bc_1 \in \concept}  \cdots \sum_{\bc_{t+1}\in \concept}   \left( \ones_{\ephi^{-1}(\bc_{1})}(\bx_1)  \;\; \ones_{\ephi^{-1}(\bc_{1})}(\bx^\text{test})    \;\;  \prod^{t+1}_{r=2} \ones_{\ephi^{-1}(\bc_{r})}(\bx_{r}) \right) \\
& \quad \quad =  \alpha   \sum_{\bc_1 \in \concept} \Bigg( \ones_{\ephi^{-1}(\bc_{1})}(\bx_1)  \; \ones_{\ephi^{-1}(\bc_{1})}(\bx^\text{test}) \Bigg)   \;\; \;\;  \sum_{\bc_2 \in \concept}  \cdots \sum_{\bc_{t+1}\in \concept}      \left(  \prod^{t+1}_{r=2} \ones_{\ephi^{-1}(\bc_{r})}(\bx_{r}) \right) \\
&\quad \quad =    \alpha \;  \ones_{\{\ephi(\bx_1)=\ephi(\bxt)\}}
\end{align*}
where we have used \eqqref{jiji1} and \eqqref{jiji2} to obtain the last equality. Summing the above over the variable $\vphi$ gives $K^\star(\bx_1,\bxt)/ |\data|^{t+1}$.
  \end{proof}
The next lemma provides a purely algebraic (as opposed to probabilistic) formulation for the functional $\Energy(K)$ defined in  \eqqref{opt1}. 
  \begin{lemma} The functional $\Energy: \mathcal K \to \real$ can be expressed as
  \begin{equation} \label{reformulated}
  \Energy (K) =  \frac{1}{|\data|} \sum_{\bx \in \data}   \sum_{\by \in \data}    K^\star(\bx,\by)  \left(\frac{|\{ \bz \in \data: K(\bx,\bz) <  K(\bx,\by) \}|}{|\data|} \right)^t.
  \end{equation}
  \end{lemma}
  \begin{proof}
 Let
 $
g: \mathcal \data^{t+2} \times  \mathcal K \to \{0,1\}
$ be the indicator function defined by
$$
g( \bx_1, \ldots, \bx_{t+1}, \bx^{\text{test}}, K)  =  \begin{cases} 1 & \text{ if }  K(\bxt ,\bx_1) > K(\bxt ,\bx_r) \; \text{ for all } 2 \le r \le t+1 \\
0 & \text{otherwise}
\end{cases}
$$
Let
$\omega$ denote the sample  $(\vphi;\bc_1, \ldots, \bc_{t+1}; \bx_1, \ldots, \bx_{t+1}; \bx^{\text{test}})$. Since $g$ only depends on the last $t+2$ variables of $\omega$, we have
\begin{align}
&\Energy(K)  =  \mathbb P_{SDM} \Big[  K(\bx^{\text{test}},\bx_1) > K(\bx^{\text{test}},\bx_r) \text{ for all } 2 \le r \le t+1 \Big]  \\
&= \sum_{\vphi \in \partition}  \;  \sum_{\bc_1 \in \concept}  \cdots \sum_{\bc_{t+1}\in \concept}   \;  \sum_{\bx_1 \in \data}  \cdots \sum_{\bx_{t+1}\in \data}  \; \sum_{\bxt \in \data}   g( \bx_1, \ldots, \bx_{t+1}, \bx^{\text{test}}, K) \; \varrho_{\rm SDM} (\omega) \\
&=  \sum_{\bx_1 \in \data}  \cdots \sum_{\bx_{t+1}\in \data}  \; \sum_{\bxt \in \data}   g( \bx_1, \ldots, \bx_{t+1}, \bx^{\text{test}}, K) 
\left(  \sum_{\vphi \in \partition}  \;  \sum_{\bc_1 \in \concept}  \cdots \sum_{\bc_{t+1}\in \concept}   
\varrho_{\rm SDM} (\omega) \right) \label{he0} \\
& =  \sum_{\bx_1 \in \data}  \cdots \sum_{\bx_{t+1}\in \data}  \; \sum_{\bxt \in \data}   g( \bx_1, \ldots, \bx_{t+1}, \bx^{\text{test}}, K) 
 \;\;   \frac{1}{|\data|^{t+1}}  \;\; K^\star(\bx_1, \bxt) \label{hehe0} \\
& =   \frac{1}{|\data|}  \sum_{\bx_1 \in \data}  \sum_{\bxt \in \data}  K^\star(\bx_1, \bxt)   \left( \frac{1}{|\data|^{t}} 
\sum_{\bx_2 \in \data}  \cdots \sum_{\bx_{t+1}\in \data}   g( \bx_1, \ldots, \bx_{t+1}, \bx^{\text{test}}, K) \right) \label{hehehe0} 
\end{align}
where we have used Lemma \ref{vitevite} to go from \eqqref{he0} to \eqqref{hehe0}.
Writing the indicator function $g$ as a product of indicator functions,
\begin{align*}
g(\bx_1, \ldots, \bx_{t+1}, \bx^{\text{test}}, K)  &= \prod_{r=2}^{t+1} \ones_{\{ K(\bxt ,\bx_1)  > K(\bxt ,\bx_r) \}} 
\end{align*}
we obtain the following expression for the term appearing between parentheses in \eqqref{hehehe0}:
\begin{align*}
 \frac{1}{|\data|^{t}} 
\sum_{\bx_2 \in \data}  \cdots \sum_{\bx_{t+1}\in \data}   g( \bx_1, \ldots, \bx_{t+1}, \bx^{\text{test}}, K) & =  \frac{1}{|\data|^t} \prod_{r=2}^{t+1}  \left( \sum_{\bx_r \in \data} \ones_{\{ K(\bxt ,\bx_1)  > K(\bxt ,\bx_r) \}} \right) \\
& =  \frac{1}{|\data|^t}  \left( \sum_{\bz \in \data} \ones_{\{ K(\bxt ,\bx_1)  > K(\bxt ,\bz) \}} \right)^t \\
&=  \left(\frac{|\{ \bz \in \data: K(\bxt,\bx_1) >  K(\bxt,\bz) \}|}{|\data|} \right)^t  \\
\end{align*}
Changing the name of variables $\bxt,\bx_1$ to $\bx, \by$ gives \eqqref{reformulated}.
  \end{proof}

 We now use expression \eqqref{reformulated} for $\Energy(K)$ and reformulate optimization problem \eqqref{opt1}-\eqqref{opt2} into an equivalent optimization problem over symmetric matrices. Putting an arbitrary ordering on the set  $\mathcal X$  (starting with $i=0$) and denoting by  $K^\star_{ij}$ the value of the kernel $K^\star$ on the pair that consists of the $i^{th}$ and $j^{th}$ element of $\data$, we see that optimization problem \eqqref{opt1}-\eqqref{opt2} can be written as
 \begin{align}
 & \text{Maximize } \quad \Energy (K) :=  \frac{1}{N} \sum_{i=0}^{N-1}   \sum_{j=0}^{N-1}    K^\star_{ij}  \left(\frac{|\{ j' \in [N]: K_{ij'} <  K_{ij} \}|}{N} \right)^t \label{optb1} \\
  &  \text{over all symmetric matrices } K \in \real^{N\times N} \label{optb2}
  \end{align}
  In the above we have used the letter $N$ to denote the cardinality of $\data$, that is $N=n_w^L$,  and we have used the notation 
   $
  [N] = \{0,1,\ldots, N-1\}.
  $     
   Before solving the matrix optimization problem \eqqref{optb1}-\eqqref{optb2}, we start with a simpler vector optimization problem. 
   Let $\boldp^\star$ be a probability vector, that is $\bp^\star \in \real_+^N$ with  $\sum_{i=1}^N p_i =1$, and consider the optimization problem:
 \begin{align}
 & \text{Maximize } \quad \energy (\boldv) :=     \sum_{j=0}^{N-1}    p^\star_{j} \;  \left(\frac{|\{ j' \in [N]: v_{j'} <  v_{j} \}|}{N} \right)^t \label{optc1} \\
  &  \text{over all vector } \boldv \in \real^{N}. \label{optc2}
  \end{align}
Recall from Definition \ref{def1} that an ordering permutation of a vector $\bv$ is a permutation that sorts its entries from smallest to largest. We will say that
two vectors $\boldv, \boldw \in \real^N$ have \emph{the same ordering} if there exist $\sigma \in S_N$ which is ordering for both $\boldv$ and $\boldw$.
The following lemma is key --- it shows that  the optimization problem \eqqref{optc1}--\eqqref{optc2} has a simple solution.
 \begin{lemma} \label{lemma:vec} The following identity
 $$ \sup_{\boldv \in \real^N} \energy(\boldv) = \entropy_t(\boldp^\star)
 $$
 holds.
 Moreover, the supremum is achieved by
 any vector $\boldv \in \real^N$ that has mutually distinct entries\footnote{That is, $v_i \neq v_j$ for all $i\neq j$.}  and that has the same ordering than $\boldp^\star$. 
 \end{lemma}
 \begin{proof} Let 
 $
  \Dist(\real^N) 
 $ denote the vectors of $\real^N$ that have mutually distinct entries. We will first show that
 \begin{equation} \label{problem2}
 \sup_{\boldv \in \real^N} \energy(\boldv)= \sup_{\boldv \in  \Dist(\real^N)}  \energy(\boldv).
 \end{equation}
To do this we show that for any $\boldv \in \real^N$, there exists $\boldw \in  \Dist(\real^N)$ such that
 \begin{equation} \label{pipi}
  |\{j' \in [N]: v_{j'} < v_j\}  |\le  |\{j'\in [N]: w_{j'} < w_j\}  | \qquad \text{for all } 0\le j \le N-1.
 \end{equation}
 There are many ways to construct such a $\boldw$. One way is to simply set $w_j = \sigma^{-1}(j)$ for some permutation $\sigma$ that orders $v$. Indeed, note that $\sigma^{-1}(j)$ provides the position of  $v_j$ in the sequence of inequality \eqqref{sequence_of_ineq}. Therefore if $v_{j'} < v_{j}$ we must have that $\sigma^{-1}(j') < \sigma^{-1}(j)$.
This implies
 $$
  \{j' \in [N]: v_{j'} < v_{j}\}    \subset \{j' \in [N]: \sigma^{-1}(j') < \sigma^{-1}(j)\}  \qquad \text{for all } j \in [N]
 $$
 which in turn implies \eqqref{pipi}.

 Because of \eqqref{problem2} we can now restrict our attention to $\boldv \in \Dist(\real^N)$.
   Note that if $\boldv \in \Dist(\real^N)$, then it has a unique ordering permutation $\sigma$, 
 $$
 v_{\sigma(0)} < v_{\sigma(1)}  < v_{\sigma(2)}  < v_{\sigma(3)}  < \ldots < v_{\sigma(N-1)} 
 $$
 and, recalling that $\sigma^{-1}(j)$ provide the position of $v_j$ in the above ordering, we 
  clearly have that
 $$
  |\{j' \in [N]: v_j' < v_j\}  | = \sigma^{-1}(j).
 $$
 Therefore, if $\boldv \in  \Dist(\real^N)$ and if $\sigma$ denotes its unique ordering permutation, $\energy(\boldv)$ can be expressed as
  \begin{align} \label{bipbop}
\energy(\boldv) =   \sum_{j=0}^{N-1}    p^\star_{j} \;  \left(\frac{|\{ j' \in [N]: v_{j'} <  v_{j} \}|}{N} \right)^t =  \sum_{j=0}^{N-1} p^\star_j \; \left( \frac{\sigma^{-1}(j)}{N} \right)^t   
  = \sum_{j=0}^{N-1} p^\star_{\sigma(j)} \;  (j/N)^t 
 \end{align}
 Looking at  definition \eqqref{perm_mom} of the permuted moment, it is then  clear that $\energy(\boldv) \le \entropy_t(\boldp^\star)$ for all  $\boldv \in  \Dist(\real^N)$.  We then note that if $\boldv \in  \Dist(\real^N)$ has the same ordering than $\boldp^\star$, then its unique ordering permutation $\sigma$ must also be an ordering permutation of $\boldp^\star$. Then \eqqref{bipbop} combined
 with Lemma \ref{matching} implies that $\energy(\boldv) = \entropy_t(\boldp^\star)$. This concludes the proof.
 \end{proof}
 Relaxing the symmetric constraint in the optimization problem  \eqqref{optb1}-\eqqref{optb2} gives the following unconstrained problem over all $N$-by-$N$ matrices:
  \begin{align}
 & \text{Maximize } \quad \Energy (K) :=  \frac{1}{N} \sum_{i=0}^{N-1}   \sum_{j=0}^{N-1}    K^\star_{ij}  \left(\frac{|\{ j' \in [N]: K_{ij'} <  K_{ij} \}|}{N} \right)^t \label{optd1} \\
  &  \text{over all matrices } K \in \real^{N\times N} \label{optd2}
  \end{align}
  Let us denote by $K^\star_{i,:}$ the $i^{th}$ row of the matrix $K^\star$ and remark that $K^\star_{i,:}$ is a probability vector (because $K^\star(\bx, \cdot)$ is a probability distribution on $\data$, see Lemma \ref{lemma:probdist}).
  We then note that the above unconstrained problem decouples into $N$ separate optimization problems of the type \eqqref{optc1}-\eqqref{optc2} in which the probability vector  $\boldp^\star$ must be replaced by the probability vector $K^\star_{i,:}$. Using Lemma  \ref{lemma:vec} we therefore have that any $K \in \real^{N\times N}$ that satisfies,  for each $0 \le i \le N-1$, 
\begin{enumerate}
 \item[(a)] The entries of  $K_{i,:}$  are mutually distinct,
 \item[(b)] $K_{i,:}$ and $K^\star_{i,:}$ have the same ordering,
 \end{enumerate}
  must be a solution of \eqqref{optd1}-\eqqref{optd2}. Lemma  \ref{lemma:vec} also gives:
  $$
  \sup_{K \in \real^{N\times N}} \Energy(K) =  \frac{1}{N} \sum_{i=0}^{N-1} \entropy_t(K^\star_{i,:}).
  $$
  Now let 
  $
  \varepsilon \in \real^{N \times N}
  $ be a symmetric matrix that satisfies:
\begin{itemize}
\item[(i)] $\varepsilon_{ij} \neq \varepsilon_{ij'}$  for all  $i,j,j' \in [N]$  with $j \neq j'$,  
\item[(ii)]  $0 \le \varepsilon_{ij} \le 1$  for all   $i,j \in [N]$,
\end{itemize} 
and define the following perturbation of the matrix $K^\star$:
\begin{equation} \label{perturbationvivi}
K = K^\star+   \frac{0.5}{s_c^L|\partition|} \;\;  \varepsilon
\end{equation}
Recalling definition  \eqqref{def:Kstar_appendix} of the kernel $K^\star$,  it is clear that for each $i,j\in [N]$, we have
\begin{equation} \label{grid}
K^\star_{ij}  =  \frac{\ell}{s_c^L|\partition|}  \qquad \text{ for some integer } \ell.
\end{equation}
As a consequence perturbing $K^\star$ by adding to its entries quantities smaller  than $1/ (s_c^L|\partition|)$ can not change the ordering of its rows. Therefore  the kernel $K$ defined by \eqqref{perturbationvivi} satisfies (b). It  also satisfies (a).
Indeed,  if $K^\star_{ij}=K^\star_{ij'}$ and $j \neq j'$, then we clearly have that $K_{ij}\neq K_{ij'}$ due to (i). On the other hand if  $K^\star_{ij} \neq K^\star_{ij'}$, then  $K_{ij}\neq K_{ij'}$ due to (ii) and \eqqref{grid}. 

We have therefore constructed a symmetric matrix that is a solution of the optimization problem \eqqref{optd1}-\eqqref{optd2}. As a consequence we have 
 $$
  \sup_{K \in \mathcal K} \Energy(K)  = \sup_{K \in \real^{N\times N}} \Energy(K) =  \frac{1}{N} \sum_{i=0}^{N-1} \entropy_t(K^\star_{i,:})
  $$
  where $\mathcal K$ should now be interpreted as the set of $N$-by-$N$ symmetric matrices. The above equality proves Theorem \ref{theorem:beautiful}, and we have also shown that the perturbed kernel \eqqref{perturbationvivi} achieves the supremum.


 \subsection{Connection Between the Two Sampling Processes} \label{section:marginal}
  In this subsection we show that the p.d.f. of Sampling Process SDM can be obtained by marginalizing the p.d.f. of Sampling Process DM  over a subset of the variables. We also compute another marginal of $\varrho_{\rm DM}$ that will prove useful in the next subsection.
 Recall that  
\begin{equation} \label{rhodm}
\varrho_{ {\rm DM}}(\omega) = \frac{1}{|\Phi||\concept|^{2R}}\prod^{R}_{r=1}\left(\frac{\ones_{\ephi^{-1}(\bc'_{r})}(\bx_{r,\,1})}{\left|\ephi^{-1}(\bc'_{r})\right|}\prod^{n_\text{spl}}_{s=2}\frac{\ones_{\ephi^{-1}(\bc_{r})}(\bx_{r,\,s})}{\left|\ephi^{-1}(\bc_{r})\right|}\right)\frac{\ones_{\ephi^{-1}(\bc'_{1})}(\bxt)}{\left|\ephi^{-1}(\bc'_{1})\right|}
\end{equation}
on $\Omega_{\rm DM} := \Phi \times \concept^{2R} \times \data^{R \times n_\text{spl} + 1}$ is the p.d.f. of the sampling process for our main data model.  Samples from $\Omega_{\rm DM}$ take the form
 \begin{multline*}
 \omega = (\vphi\;\;\; ; \;\;\;  \bc_1, \bc_2, \bc_3, \ldots, \bc_R \;\;\; ; \;\;\;   \bc'_1, \bc_2', \bc_3',\ldots, \bc'_R \;\; ; \\ \bx_{1,1}, \bx_{1,2},   \bx_{1,3}, \ldots,\bx_{1,\nspl} \;\;\; ; \;\;\;    \ldots \;\;\; ; \;\;\;  \bx_{R,1}, \bx_{R,2},  \bx_{R,3},\ldots,\bx_{R,\nspl} \;\;\; ; \;\;\; \bxt   )
 \end{multline*}
 We separate these variables into two groups, $\omega = (\omega_a , \omega_b)$, where
  \begin{align}
 \omega_a &= (\vphi\;\;\; ; \;\;\;  \bc_1, \bc_2, \bc_3, \ldots, \bc_R \;\;\; ; \;\;\;   \bc'_1, \bc_2', \bc_3',\ldots, \bc'_R \;\; ;  \;\;   \bx_{1,1}, \bx_{1,2}\;\; ; \;\;\;    \ldots \;\; ; \;\;  \bx_{R,1}, \bx_{R,2} \;\;\; ; \;\;\; \bxt   )  \label{omegaa} \\
 \omega_b &=   (\bx_{1,3}, \bx_{1,4}, \ldots,\bx_{1,\nspl} \;\;\; ; \;\;\;    \ldots \;\; ; \;\;  \bx_{R,3},  \bx_{R,4}, \ldots,\bx_{R,\nspl}   ) \label{omegab}
\end{align} 
The variable $\omega_a$ belongs  to  $\Omega_a = \Phi \times \concept^{2R} \times \data^{2R+1}$, and the variable $\omega_b$ belongs to  $\Omega_b =  \data^{R(\nspl-2)}.$
Note that the variables in $\omega_a$ contains, among other, $2R$ sequences of concepts  and $2R$ training points (the first and second training points of each category). Each of these $2R$ training points is generated by one of the $2R$ sequences of concepts. So the variables involved in $\omega_a$ are generated by a process  similar to the one involved in the simpler data model. The following lemma shows that $p.d.f.$ of $\omega_a$, after marginalizing $\omega_b$, is indeed $\varrho_{SDM}$.
 \begin{lemma} \label{datamodelconnection} For all  $\omega_a \in \Omega_a$ we have
$$
\sum_{\omega_b \in \Omega_b} \varrho_{\rm DM}(\omega_a,\omega_b) =
   \frac{1}{|\Phi||\concept|^{2R}}
\left(\prod^{R}_{r=1}\frac{\ones_{\ephi^{-1}(\bc'_{r})}(\bx_{r,\,1})}{\left|\ephi^{-1}(\bc'_{r})\right|}\frac{\ones_{\ephi^{-1}(\bc_{r})}(\bx_{r,\,2})}{\left|\ephi^{-1}(\bc_{r})\right|}\right)
\Bigg(\frac{\ones_{\ephi^{-1}(\bc'_{1})}(\bx^\text{test})}{\left|\ephi^{-1}(\bc'_{1})\right|} \Bigg)
$$
\end{lemma}
Recalling the definition \eqqref{rhosdm} of  $\varrho_{\rm SDM}$, and letting $t+1 = 2R$,  we see that the above lemma states that
\begin{equation} \label{dmtosdm}
\sum_{\omega_b \in \Omega_b} \varrho_{\rm DM}(\omega_a,\omega_b) = \varrho_{\rm SDM} (\omega_a)
\end{equation}
and $\Omega_a = \Omega_{\rm SDM}$.

\begin{proof}[Proof of Lemma \ref{datamodelconnection}]
 We start by reorganizing the terms involved in the product defining $\varrho_{\rm DM}$ so that the variables in $\omega_a$ and $\omega_b$ are clearly separated:
 \begin{multline*}
 \varrho_{\rm DM}(\omega) = \\
 \frac{1}{|\Phi||\concept|^{2R}}
\left(\prod^{R}_{r=1}\frac{\ones_{\ephi^{-1}(\bc'_{r})}(\bx_{r,\,1})}{\left|\ephi^{-1}(\bc'_{r})\right|}\frac{\ones_{\ephi^{-1}(\bc_{r})}(\bx_{r,\,2})}{\left|\ephi^{-1}(\bc_{r})\right|}\right)
\Bigg(\frac{\ones_{\ephi^{-1}(\bc'_{1})}(\bx^\text{test})}{\left|\ephi^{-1}(\bc'_{1})\right|} \Bigg)
\left(\prod^{R}_{r=1}
\prod^{n_\text{spl}}_{s=3}\frac{\ones_{\ephi^{-1}(\bc_{r})}(\bx_{r,\,s})}{\left|\ephi^{-1}(\bc_{r})\right|} \right)
\end{multline*}
To demonstrate the process, let us start by summing the above formula over the first variable of $\omega_b$, namely $\bx_{1,3}$. Since this variable only occurs in the last term of the above product, we have:
\begin{multline*}
 \sum_{\bx_{1,3} \in \data }\varrho_{\rm DM}(\omega)   = 
 \frac{1}{|\Phi||\concept|^{2R}}
\left(\prod^{R}_{r=1}\frac{\ones_{\ephi^{-1}(\bc'_{r})}(\bx_{r,\,1})}{\left|\ephi^{-1}(\bc'_{r})\right|}\frac{\ones_{\ephi^{-1}(\bc_{r})}(\bx_{r,\,2})}{\left|\ephi^{-1}(\bc_{r})\right|}\right) 
\Bigg(\frac{\ones_{\ephi^{-1}(\bc'_{1})}(\bx^\text{test})}{\left|\ephi^{-1}(\bc'_{1})\right|} \Bigg) \\
\left(
\prod_{\substack{1\le r \le R \\ 3 \le s \le \nspl \\ (r,s) \neq (1,3)}}\frac{\ones_{\ephi^{-1}(\bc_{r})}(\bx_{r,\,s})}{\left|\ephi^{-1}(\bc_{r})\right|} \right) 
\left(  \sum_{\bx_{1,3} \in \data }  \frac{\ones_{\ephi^{-1}(\bc_{1})}(\bx_{1,3})}{\left|\ephi^{-1}(\bc_{1})\right|}  \right)
\end{multline*}
Since $\sum_{\bx\in\data} \ones_{\ephi^{-1}(\bc_{1})}   (\bx) = |\ephi^{-1}(\bc_{1})|$, the last term of the above product is equal to $1$ and can therefore be omitted. Repeating this process for all the $\bx_{r,s}$ that constitute $\omega_b$ leads to the desired result.
\end{proof}
In the next subsection we will need the marginal of $\varrho_{\rm DM}$ with respect to another set of variables. To this aim we write $\omega = (\omega_c, \omega_d)$ where
\begin{align}
& \omega_c =  (\vphi\;\;\; ;  \;\;\;  \bx_{1,2},   \bx_{1,3}, \ldots,\bx_{1,\nspl} \;\;\; ; \;\;\;    \ldots \;\;\; ; \;\;\;  \bx_{R,2},\bx_{R,3},,\ldots,\bx_{R,\nspl} \;\;\; ; \;\;\; \bxt   ) \label{omegac} \\
& \omega_d =   ( \bc_1,\ldots, \bc_R \;\; ; \;\;  \bc'_1,\ldots, \bc'_R \;\; ; \;\;  \bx_{1,1}\;\;;\;\;    \ldots\;\;;\;\;  \bx_{R,1} ) \label{omegad} 
\end{align}
Note that all the unfamiliar training points are contained in $\omega_d$. The test point and the familiar training points are in $\omega_c$. We also let $\Omega_c = \Phi  \times \data^{R(\nspl-1)+1}$ and $\Omega_d = \concept^{2R} \times \data^{R}$.
  \begin{lemma} \label{lemma:crazymarginal} For all  $\omega_c \in \Omega_c$ we have
$$
\sum_{\omega_d \in \Omega_d} \varrho_{\rm DM}(\omega_c,\omega_d) =
  \frac{1}{|\Phi||\data|^{R+1} s_c^{LR(\nspl-2)}} 
\prod^{R}_{r=1}
\ones_{\{\ephi(\bx_{r,2})= \ephi(\bx_{r,3})= \ldots = \ephi(\bx_{r,\nspl})\}   }
$$
\end{lemma}
\begin{proof} We reorganizing the terms involved in the product defining $\varrho_{\rm DM}$ so that the variables in $\omega_c$ and $\omega_d$ are separated:
\begin{align*}
\varrho_{\rm DM}(\omega) &= \frac{1}{|\Phi||\concept|^{2R}}
 \left(\prod^{R}_{r=1}
\prod^{n_\text{spl}}_{s=2}\frac{\ones_{\ephi^{-1}(\bc_{r})}(\bx_{r,\,s})}{\left|\ephi^{-1}(\bc_{r})\right|}\right)
\Bigg( \frac{\ones_{\ephi^{-1}(\bc'_{1})}(\bx^\text{test})}{\left|\ephi^{-1}(\bc'_{1})\right|}  \Bigg)
\left(\prod^{R}_{r=1}\frac{\ones_{\ephi^{-1}(\bc'_{r})}(\bx_{r,\,1})}{\left|\ephi^{-1}(\bc'_{r})\right|}
\right)
\end{align*}
Summing the above formula over the last variable involved in $\omega_d$, namely $\bx_{R,1}$, gives
\begin{multline*}
 \sum_{\bx_{R,1} \in \data }\varrho_{\rm DM}(\omega) 
 = \frac{1}{|\Phi||\concept|^{2R}}
\left(\prod^{R}_{r=1}
\prod^{n_\text{spl}}_{s=2}\frac{\ones_{\ephi^{-1}(\bc_{r})}(\bx_{r,\,s})}{\left|\ephi^{-1}(\bc_{r})\right|}\right)
\Bigg( \frac{\ones_{\ephi^{-1}(\bc'_{1})}(\bx^\text{test})}{\left|\ephi^{-1}(\bc'_{1})\right|}  \Bigg) \\
\left(\prod^{R-1}_{r=1}\frac{\ones_{\ephi^{-1}(\bc'_{r})}(\bx_{r,\,1})}{\left|\ephi^{-1}(\bc'_{r})\right|}
\right) 
\Bigg( \sum_{\bx_{R,1}  \in \data } \frac{\ones_{\ephi^{-1}(\bc'_{R})}(\bx_{R,1})}{\left|\ephi^{-1}(\bc'_{R})\right|} \Bigg) 
\end{multline*}
The last term in the above product is equal to $1$ and can therefore be omitted. Iterating this process gives
\begin{align*}
\sum_{\bx_{1,1}\in\data} \cdots \sum_{\bx_{R,1} \in \data }\varrho_{\rm DM}(\omega) 
 =  \frac{1}{|\Phi||\concept|^{2R}}
\left(\prod^{R}_{r=1}
\prod^{n_\text{spl}}_{s=2}\frac{\ones_{\ephi^{-1}(\bc_{r})}(\bx_{r,\,s})}{\left|\ephi^{-1}(\bc_{r})\right|}\right)
\Bigg( \frac{\ones_{\ephi^{-1}(\bc'_{1})}(\bx^\text{test})}{\left|\ephi^{-1}(\bc'_{1})\right|}  \Bigg)
\end{align*}
We then use the fact that
$
\sum_{c_1' \in \concept} \ones_{\ephi^{-1}(\bc'_{1})}(\bxt) = 1,
$ see  \eqqref{jiji1}, 
together with $\left|\ephi^{-1}(\bc_{1}')\right|=s_c^L$, see \eqqref{scl}, to obtain
\begin{align*}
\sum_{c_1' \in \concept}  \sum_{\bx_{1,1}\in\data} \cdots \sum_{\bx_{R,1} \in \data }\varrho_{\rm DM}(\omega) 
 =  \frac{1}{|\Phi||\concept|^{2R} s_c^L}
\left(\prod^{R}_{r=1}
\prod^{n_\text{spl}}_{s=2}\frac{\ones_{\ephi^{-1}(\bc_{r})}(\bx_{r,\,s})}{\left|\ephi^{-1}(\bc_{r})\right|}\right)
\end{align*}
We then sum over  $\bc'_2, \ldots, \bc'_R$. Since these variables are not involved in the above formula we get
\begin{align*}
\sum_{c_1' \in \concept}  \cdots \sum_{c_R' \in \concept}  \sum_{\bx_{1,1}\in\data} \cdots \sum_{\bx_{R,1} \in \data }\varrho_{\rm DM}(\omega) 
& =  \frac{1}{|\Phi||\concept|^{R+1} s_c^L}
\left(\prod^{R}_{r=1}
\prod^{n_\text{spl}}_{s=2}\frac{\ones_{\ephi^{-1}(\bc_{r})}(\bx_{r,\,s})}{\left|\ephi^{-1}(\bc_{r})\right|}\right) \\
& =  \frac{1}{|\Phi||\concept|^{R} |\data|}
\left(\prod^{R}_{r=1}
\prod^{n_\text{spl}}_{s=2}\frac{\ones_{\ephi^{-1}(\bc_{r})}(\bx_{r,\,s})}{\left|\ephi^{-1}(\bc_{r})\right|}\right) 
\end{align*}
where we have used $|\concept| s_c^L = n_c^L s_c^L = |\data|$ to obtain the last equality. Summing over $\bc_1$ gives
\begin{align*}
\sum_{\bc_1 \in \concept}\sum_{c_1' \in \concept} & \cdots \sum_{c_R' \in \concept}   \sum_{\bx_{1,1}} \cdots \sum_{\bx_{R,1} \in \data }  \varrho_{\rm DM}(\omega) 
 =  \frac{1}{|\Phi||\concept|^{R} |\data|}  \sum_{\bc_1 \in \concept}
\left(\prod^{R}_{r=1}
\prod^{n_\text{spl}}_{s=2}\frac{\ones_{\ephi^{-1}(\bc_{r})}(\bx_{r,\,s})}{\left|\ephi^{-1}(\bc_{r})\right|}\right) \\
 &=  \frac{1}{|\Phi||\concept|^{R} |\data|} 
\left(\prod^{R}_{r=2}
\prod^{n_\text{spl}}_{s=2}\frac{\ones_{\ephi^{-1}(\bc_{r})}(\bx_{r,\,s})}{\left|\ephi^{-1}(\bc_{r})\right|}\right) 
\left(  \sum_{\bc_1 \in \concept} \label{prodind}
\prod^{n_\text{spl}}_{s=2}\frac{\ones_{\ephi^{-1}(\bc_{1})}(\bx_{1,\,s})}{\left|\ephi^{-1}(\bc_{1})\right|}\right) \\
&=  \frac{1}{|\Phi||\concept|^{R} |\data|} 
\left(\prod^{R}_{r=2}
\prod^{n_\text{spl}}_{s=2}\frac{\ones_{\ephi^{-1}(\bc_{r})}(\bx_{r,\,s})}{\left|\ephi^{-1}(\bc_{r})\right|}\right) 
\left( \frac{\ones_{\{\ephi(\bx_{1,2})=  \ephi(\bx_{1,3})=\ldots = \ephi(\bx_{1,\nspl})\}   }}{\left|\ephi^{-1}(\bc_{1})\right|^{\nspl-1}}\right) 
\end{align*}
To obtain the last equality we have used \eqqref{jiji2} but for a product of $\nspl-1$ indicator functions instead of just two.
Iterating this process we obtain
\begin{align*}
\sum_{\bc_1 \in \concept} \cdots \sum_{\bc_R \in \concept}
\sum_{c_1' \in \concept}  \cdots \sum_{c_R' \in \concept}  \sum_{\bx_{1,1}\in \data} \cdots \sum_{\bx_{R,1} \in \data }\varrho_{\rm DM}(\omega) 
&=  \frac{1}{|\Phi||\concept|^{R} |\data|} 
\prod^{R}_{r=1}
\frac{\ones_{\{\ephi(\bx_{r,2})=   \ephi(\bx_{1,3})= \ldots = \ephi(\bx_{r,\nspl})\}   }}{\left|\ephi^{-1}(\bc_{r})\right|^{\nspl-1}} 
\end{align*}
Using one more time that  $\left|\ephi^{-1}(\bc_{r})\right|=s_c^L$ and $|\concept| s_c^L = |\data|$  gives to the desired result.
\end{proof}

  \subsection{Conclusion of Proof} \label{section:secondpart}
  
We now establish the desired upper bound \eqqref{albator}, which we restate below for convenience:
 \begin{equation} \label{albator99}
\sup_{K \in \mathcal K}   \mathbb P_{{\rm DM}}\Big[  E_K \Big] \le   \frac1{|\data|} \sum_{\bx \in \data} \mathcal{H}_{2R-1}\left(K^\star(\bx,\,\cdot)\right) + \frac1{R}
\end{equation}  
  where 
 \begin{multline}
 E_K= \Big\{ \omega \in \Omega_{\rm DM} : \;\; \text{There exists } 1 \le s^* \le \nspl \text{ such that } \\ 
  K\left(\bx^\text{test}, \bx_{1,s^*} \right) >  K\left(\bx^\text{test}, \bx_{r,s} \right) \text{ for all } 2 \le r \le R \text{ and all } 1 \le s \le \nspl 
 \Big\}
 \end{multline}
 
 \begin{figure}[t]
          \centering
         \includegraphics[totalheight=0.35\textheight]{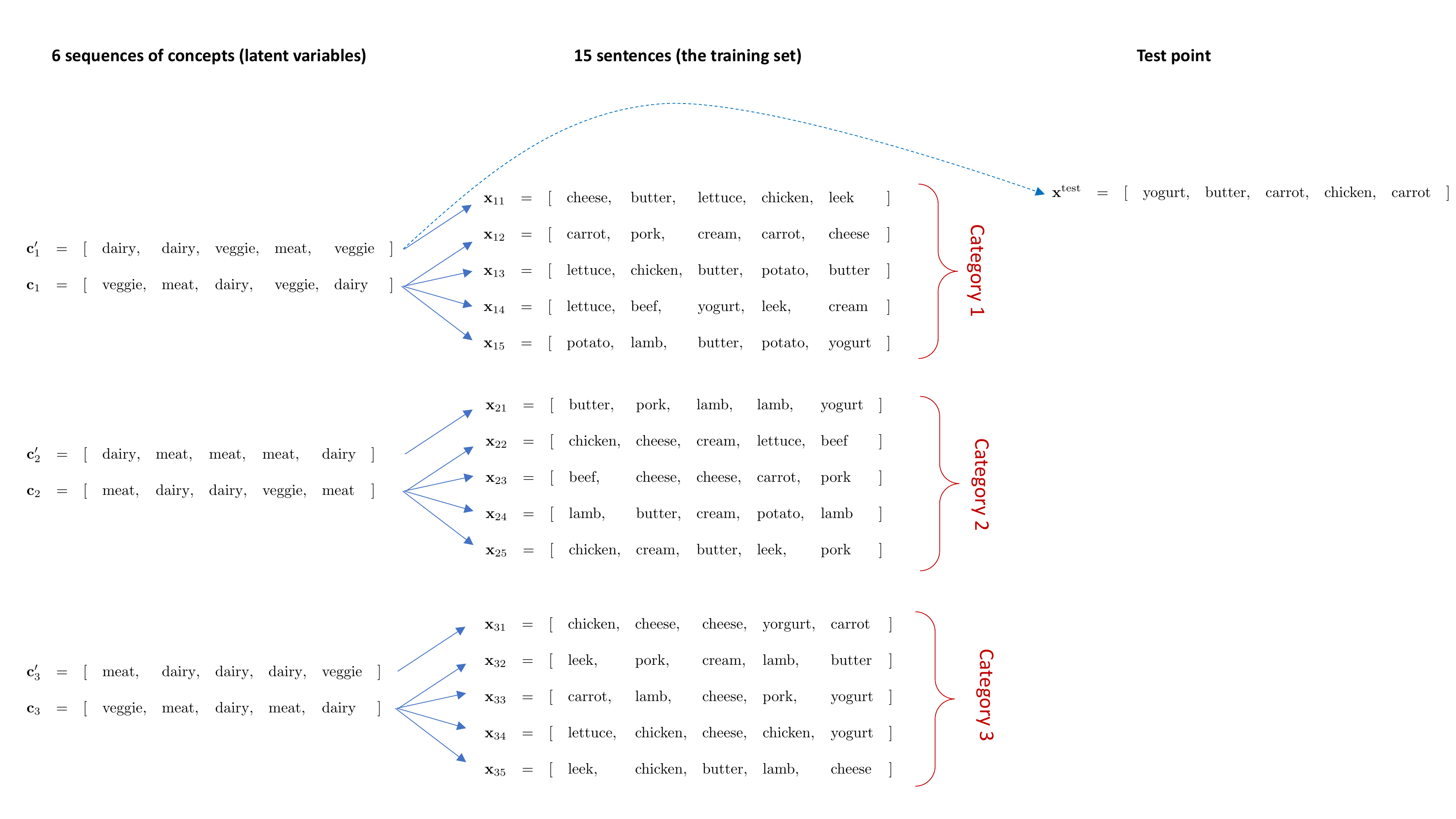}
            \caption{The test point $\bx^\text{test}$ and the train point $\bx_{1,1}$ are generated by the same sequence of concepts.}
            \label{figure:foodtest}
\end{figure}

We recall that the test point $\bx^\text{test}$ is generated by the unfamiliar sequence of concepts $\bc_1'$ and that it belongs to category 1, see Figure \ref{figure:foodtest}.
The event $E_K $ describes all the outcomes in which
 the training point  most similar to $\bx^\text{test}$ (where similarity is measured with respect to the kernel $K$)  belongs to the first category. There are two very distinct cases within the event $E_K$: the  training point  most similar to $\bx^\text{test}$ can be   $\bx_{1,1}$  --- this  corresponds to a  `meaningful success'  in which the learner  recognizes that $\bx_{1,1}$ is generated by the same sequence of concepts than $\bx^\text{test}$, see Figure \ref{figure:foodtest}. Or the training point most similar  to $\bx^\text{test}$ can be one of the points $\bx_{1,2}, \ldots, \bx_{1,\nspl}$ ---  this  corresponds to a `lucky success'  because $\bx_{1,2}, \ldots, \bx_{1,\nspl}$   are not related to $\bx^\text{test}$ (they are generated by a different sequence of concept, see Figure \ref{figure:foodtest}). To make this discussion formal, we fix a kernel $K \in \mathcal K$, and we partition the event $E_K$ as follow
 \begin{equation} \label{tutut1}
E_K = E_\text{meaningful} \cup E_\text{luck}
\end{equation}
where
\begin{align*}
 & E_{\text{meaningful}} =  E_K \; \cap \;  \Big\{ \omega \in \Omega_{\rm DM}:  \;\;  K\left(\bx^\text{test}, \bx_{1,1} \right) >  K\left(\bx^\text{test}, \bx_{1,s} \right)   \text{ for all }  2 \le s \le \nspl \Big\}  
 \\& E_{\text{luck}} = E_K  \; \cap \;  \Big\{ \omega \in \Omega_{\rm DM}:  \;\;  K\left(\bx^\text{test}, \bx_{1,1} \right) \le K\left(\bx^\text{test}, \bx_{1,s} \right)   \text{ for some }  2 \le s \le \nspl \Big\} 
  \end{align*}
The next two lemmas provide upper bounds for the probability of the events $E_\text{meaningful}$ and $E_\text{luck}.$
\begin{lemma} \label{lemma:bigigi} 
$ \displaystyle
\mathbb P_{\rm DM} [E_\text{meaningful} ] \le \frac{1}{|\data|} \sum_{\bx \in \data} \mathcal H_{2R-1}( K^\star_\bx ).
$
\end{lemma}
\begin{proof} 
Define the event  
 \begin{multline*}
A :=
\Big\{ \omega \in \Omega_{\rm DM} : \;\;
  K\left(\bx^\text{test}, \bx_{1,1} \right) >  K\left(\bx^\text{test}, \bx_{r,1} \right) \text{ for all } 2 \le r \le R 
 \Big\} \\
\cap
 \Big\{ \omega \in \Omega_{\rm DM} : \;\;
  K\left(\bx^\text{test}, \bx_{1,1} \right) >  K\left(\bx^\text{test}, \bx_{r,2} \right) \text{ for all } 1 \le r \le R 
 \Big\} 
  .
 \end{multline*}

This events involves only the first two training points of each category. On the example depicted on Figure \ref{figure:foodtest}, that would be the points  $\bx_{1,1}$ and $\bx_{1,2}$, the points  $\bx_{2,1}$ and $\bx_{2,2}$, and finally the points  $\bx_{3,1}$ and $\bx_{3,2}$. The event $A$ consists in all the outcomes in which, among these $2R$ training points, $\bx_{1,1}$ is most similar to $\bx^\text{test}$. We then make two key remarks. First,  these $2R$ points are generated by $2R$ distinct sequences on concepts --- so if we restrict our attention to these $2R$ points, we are in a situation very similar to the simpler data model $\varrho_{SDM}$ (i.e. we first generate $2R$ sequences of concepts, then from each sequence of concepts we generate a single training point, and finally we generate a test point from the first sequence of concepts.) We will make this intuition precise by appealing to the fact that $\varrho_{SDM}$ is the marginal of $\varrho_{DM},$ and this will allow us to obtain a bound for  $\mathbb P_{\rm DM} [A]$ in term of the permuted moment of $K^\star.$   The second remark is that $E_\text{meaningful}$ is clearly contained in $A$, and therefore we have 
\begin{equation} \label{dfg}
\mathbb P_{\rm DM} [E_\text{meaningful} ] \le \mathbb P_{\rm DM} [A]  
\end{equation}
so an upper bound for  $\mathbb P_{\rm DM} [A] $ is also an upper bound for $\mathbb P_{\rm DM} [E_\text{meaningful} ]$.

Let us rename some of the variables. We define $\bd_1, \ldots,\bd_{2R}$, and $\by_1, \ldots, \by_{2R}$ as follow:
\begin{align*}
&\bd_{2r-1} = \bc'_r  \quad \;\;\,\text{ and } \quad \bd_{2r} = \bc_r \qquad \;\; \text{ for } r=1,\ldots,R \\
& \by_{2r-1} = \bx_{r,1}   \quad \text{ and } \quad  \by_{2r} = \bx_{r,2} \qquad \text{ for } r=1,\ldots,R 
\end{align*}
On the example depicted on Figure \ref{figure:foodtest}, that would be: 
$$
\begin{matrix*}
\by_1 = \bx_{1,1} ,& \;\; \by_2 = \bx_{1,2} ,& \;\; \by_3 = \bx_{2,1} ,& \;\; \by_4 = \bx_{2,2} ,& \;\; \by_5 = \bx_{3,1} ,& \;\; \by_6 = \bx_{3,2}  \\
\bd_1 = \bc_1,  & \;\; \bd_2 = \bc_1',  & \;\; \bd_3 = \bc_2,  & \;\; \bd_4 = \bc_2',  & \;\; \bd_5 = \bc_3,  & \;\; \bd_6= \bc_3'
\end{matrix*}
$$
In other words, the $\by_r$'s are the first two training points of each category and the $\bd_r$'s are their corresponding sequence of concepts. With these notations it is clear that the training points  $\by_r$ are generated by distinct sequences of concepts, and that the test point $\bx^\text{test}$ is generated by the same sequence of concepts than $\by_1$.
Moreover the event $A$ can now be conveniently written as
$$
A =  \{ \omega \in \Omega_{\rm DM}: \;\;  K\left(\bx^\text{test}, \by_{1} \right) >  K\left(\bx^\text{test}, \by_{r} \right)  \text{ for all } 2 \le  r \le 2R \}.
$$
Let $h: \data^{2R+1} \to \real$ be the indicator function defined by
$$
h( \by_1, \ldots, \by_{2R}, \bxt )   =   \begin{cases} 1 & \text{ if } K\left(\bx^\text{test}, \by_1 \right) >  K\left(\bx^\text{test}, \bx_r \right)  \text{ for all }  2 \le r \le 2R  \\
0 & \text{otherwise}
\end{cases}
$$
We now recall  the splitting $\omega = (\omega_a, \omega_b)$  described in \eqqref{omegaa}-\eqqref{omegab} and note that  $\omega_a$ can be written as
$$
\omega_{a} = (\vphi \;\; ; \;\; \bd_1, \ldots, \bd_{2R} \;\; ; \;\; \by_1, \ldots, \by_{2R} \;\; ; \;\; \bxt)
$$
Since $h$ only depends on the variables involved in $\omega_a$, and since, according to Lemma \ref{datamodelconnection},  $\sum_{\omega_b}\varrho_{\rm DM}(\omega_a,\omega_b) = \varrho_{\rm SDM}(\omega_a)$, we obtain
\begin{align*}
\mathbb P_{\rm DM}[A] &= \sum_{\omega_a \in \Omega_a} \sum_{\omega_b \in \Omega_b} h( \by_1, \ldots, \by_{2R}, \bxt ) \; \varrho_{\rm DM}(\omega_a,\omega_b) \\
& =  \sum_{\omega_a \in \Omega_a} h( \by_1, \ldots, \by_{2R}, \bxt ) \; \varrho_{\rm SDM}(\omega_a) \\
& = \mathbb P_{\rm SDM}[ \omega_a \in \Omega_a: \;\;   K\left(\bx^\text{test}, \by_{1} \right) >  K\left(\bx^\text{test}, \by_{r} \right)  \text{ for all } 2 \le  r \le 2R ] \\
& \le  \frac{1}{|\data|} \sum_{\bx \in \data} \mathcal H_{2R-1}( K^\star_\bx )
\end{align*}
where we have used Theorem \ref{theorem:beautiful} in order to get the last inequality (with the understanding that $t+1=2R$.) 
Combining the above bound with \eqqref{dfg} concludes the proof.
\end{proof}
We now estimate the probability of the event $E_\text{luck}.$
 \begin{lemma} \label{lemma:bigigi2} 
$ \displaystyle
\mathbb P_{\rm DM} [E_\text{luck} ] \le \frac{1}{R}.
$
\end{lemma}
\begin{proof}
 For $1 \le r \le R$, we  define the events
 \begin{align*}
 B_r = \bigcap_{\substack{1 \le r' \le R \\ r'\neq r}} \left\{\omega \in \Omega_{\rm DM}: \;\;  \max_{2 \le s \le \nspl} K(\bxt , \bx_{r,s}) > \max_{2 \le s' \le \nspl} K(\bxt , \bx_{r',s'}) \right\}
 \end{align*}
 Note that the events $B_r$ involve only the training points with an index $s\ge 2$: these are the familiar training points. On the example depicted on Figure \ref{figure:foodtest}, these are the training points generated by $\bc_1, \bc_2$ and $\bc_3$.
 Let us pursue with this example.  The event $B_1$ consists in all the outcomes in which one of the points generated by $\bc_1$ is more similar to $\bx^\text{test}$ than  any of the points generated by $\bc_2$ and $\bc_3$.  The event $B_2$ consists in all the outcomes in which one of the points generated by $\bc_2$ is more similar to $\bx^\text{test}$ than  any of the points generated by $\bc_1$ and $\bc_3$.  And finally the   event $B_3$ consists in all the outcomes in which one of the points generated by $\bc_3$ is more similar to $\bx^\text{test}$ than  any of the points generated by $\bc_1$ and $\bc_2$. Importantly, the test point $\bx^{\text{test}}$ is generated by the unfamiliar sequence of concepts $\bc_1'$, and this sequence of concept is unrelated to the sequence $\bc_1, \bc_2$ and $\bc_3$. So one would expect that, from simple symmetry, that
 \begin{equation} \label{equiprobable99}
 \mathbb P_{\rm DM} [B_1] = \mathbb P_{\rm DM} [B_2] = \mathbb P_{\rm DM} [B_3].  
 \end{equation}
 We will prove \eqqref{equiprobable99} rigorously, for general $R$,  using Lemma  \ref{lemma:crazymarginal} from the previous subsection. But before to do so, let us show that  \eqqref{equiprobable99} implies the desired upperbound on the probability of $E_\text{luck}.$ First, note that $E_\text{luck} \subset B_1$ and therefore
 \begin{equation} \label{exq}
 \mathbb P_{\rm DM} [E_\text{luck} ] \le \mathbb P_{\rm DM} [B_1].
 \end{equation}
 Then, note that $B_1$, $B_2$ and $B_3$ are mutually disjoints, and therefore,  continuing with the same example, 
 $$
  \mathbb P_{\rm DM} [B_1] + \mathbb P_{\rm DM} [B_2] + \mathbb P_{\rm DM} [B_3] =   \mathbb P_{\rm DM} [B_1 \cup B_2 \cup B_3] \le 1
 $$
 which, combined with \eqqref{equiprobable99} and \eqqref{exq}, gives  $ \mathbb P_{\rm DM} [E_\text{luck} ]  \le 1/3$ as desired. 
 
 We now provide a formal proof of  \eqqref{equiprobable99}. 
  As in the proof of the previous lemma, it is convenient to rename some of the variables. Let denote by $\fami_r$ the variable that consists in the familiar training points from category $r$:
$$
\fami_r = (\bx_{r,2}, \ldots, \bx_{r,\nspl}) \in \data^{\nspl-1}
$$
 With this notation we have $\fami_{r,s} = \bx_{r,s+1}.$ 
 We now recall the  splitting $\omega = (\omega_c, \omega_d)$ described in \eqqref{omegac}-\eqqref{omegad}, and note that  $\omega_c$ can be written as 
\begin{equation} \label{hihihi}
\omega_c = (\vphi; \;\; \fami_1 \;\;;\;\; \ldots\;\;;\;\; \fami_R \;\; ; \;\;  \bxt).
\end{equation}
Using Lemma \ref{lemma:crazymarginal} we have
\begin{align}
\sum_{\omega_d \in \Omega_d} \varrho_{\rm DM}(\omega_c,\omega_d) &=
  \alpha
\prod^{R}_{r=1}
\ones_{\{\ephi(\bx_{r,2})= \ephi(\bx_{r,3})= \ldots = \ephi(\bx_{r,\nspl})\}   } \\
&=
  \alpha
\prod^{R}_{r=1}
\ones_{\{\ephi(\fami_{r,1})= \ephi(\fami_{r,2})=  \ldots = \ephi(\fami_{r,\nspl-1})\}   }  \label{impl1}\\
&=
  \alpha
\prod^{R}_{r=1} h(\vphi, \fami_r) \label{impl2}
\end{align}
where $\alpha$ is the constant appearing in front of the product in Lemma  \ref{lemma:crazymarginal} an $h: \Phi \times \mathcal  \data^{\nspl-1} \to \{0,1\}$ is the indicator function implicitly defined in equality \eqqref{impl1}-\eqqref{impl2}. 
With the slight abuse of notation of viewing $\fami_r$ as a set instead of a tuple, let us rewrite the event $B_r$ as 
\begin{align*}
 B_r = \bigcap_{\substack{1 \le r' \le R \\ r'\neq r}} \left\{\omega \in \Omega_{\rm DM}: \;\;  \max_{\bx \in \fami_r} K(\bxt , \bx) > \max_{\bx \in \fami_{r'}} K(\bxt , \bx) \right\}
 \end{align*}
We  also define the corresponding indicator function 
$$
g(\fami_r, \fami_{r'} , \bxt) = 
\begin{cases}
1 &   \text{ if }  \displaystyle  \max_{\bx \in \fami_r} K(\bxt , \bx) > \max_{\bx \in \fami_{r'}} K(\bxt , \bx)  \\ \\
0 & \text{otherwise}
\end{cases}
$$
We now compute $ \mathbb P_{\rm DM}(B_1)$ (the formula for the other $\mathbb P_{\rm DM}(B_r)$ are obtained in a similar manner.)
Recall from \eqqref{hihihi} that the variables involved in $\fami_r$ only appear in $\omega_c$.
Therefore
\begin{align}
\mathbb P_{\rm DM}[B_1] &= \sum_{\omega_c \in \Omega_c} \sum_{\omega_d \in \Omega_d}   \left( \prod_{r=2}^R g(\fami_1, \fami_{r},  \bxt)  \right) \  \varrho_{\rm DM}(\omega_c,\omega_d) \\
& = \sum_{\omega_c \in \Omega_c}  \left( \prod_{r=2}^R  g(\fami_1, \fami_{r},  \bxt)  \right) \  \sum_{\omega_d \in \Omega_d}  \varrho_{\rm DM}(\omega_c,\omega_d) \\
& = \alpha \sum_{\omega_c \in \Omega_c}  \left( \prod_{r=2}^R g(\fami_1, \fami_{r},  \bxt)  \right) \  \Bigg( 
\prod^{R}_{r'=1} h(\vphi, \fami_{r'}) \bigg) 
\end{align}
where we have used  \eqqref{impl2} to obtain the last equality.
Let us now compare $ \mathbb P_{\rm DM}(B_1)$ and $ \mathbb P_{\rm DM}(B_2)$. Letting $\mathcal Z:= \data^{\nspl-1}$, and recalling that $\omega_c = (\vphi; \;\; \fami_1 \;\;,\;\; \ldots\;\;,\;\; \fami_R \;\; ; \;\;  \bxt)$, we obtain:
\begin{align*}
&\mathbb P_{\rm DM}[B_1] = \sum_{\vphi \in \Phi} \sum_{\fami_1 \in \mathcal Z} \cdots \sum_{\fami_R\in \mathcal Z} \sum_{\bxt\in \mathcal X} \left( \prod_{\substack{1 \le r \le R \\ r \neq 1}} g(\fami_1, \fami_r, \bxt)\right) \left( \prod_{r'=1}^R h(\phi, \fami_r')\right) \\
&\mathbb P_{\rm DM}[B_2] = \sum_{\vphi \in \Phi} \sum_{\fami_1 \in \mathcal Z} \cdots \sum_{\fami_R\in \mathcal Z} \sum_{\bxt\in \mathcal X} \left( \prod_{\substack{1 \le r \le R \\ r \neq 2}} g(\fami_2, \fami_r, \bxt)\right) \left( \prod_{r'=1}^R h(\phi, \fami_r')\right) 
\end{align*}
From the above it is clear that $\mathbb P_{\rm DM}[B_1] = \mathbb P_{\rm DM}[B_2]$ (as can be seen by exchanging the name of the variables $\fami_1$ and $\fami_2$). Similar reasoning shows that the events  $B_r$  are all equiprobable, which concludes the proof.
\end{proof}
Combining  Lemma \ref{lemma:bigigi} and \ref{lemma:bigigi2}, together with the fact that   $E_K = E_\text{correct} \cup E_\text{luck}$,  concludes the proof of \eqqref{albator99}.  Inequality \eqqref{albator99} implies inequality \eqqref{zoro2}, which itself is equivalent to inequality \eqqref{connection}.


\section{Upper bound for the Permuted Moment of $K^\star$} \label{sectionC}
This section is devoted to the proof of  inequality \eqqref{upper_bound_pm}, which we state below as a theorem for convenience.
\begin{theorem} \label{thm:count} For all $0 \le \ell \le L$,  we have the upper bound
\begin{equation} \label{stirling}
\frac{1}{|\data|} \sum_{\bx \in \data} \entropy_t(K^\star_\bx) \le \left( 1 - \sum_{\bk \in \mathcal S_\ell}   \mathfrak f(\bk ) \mathfrak g(\bk ) \right) +   \frac{1}{t+1} \left( \max_{\bk \in \mathcal S_\ell} \;  \mathfrak  f(\bk )  \right). 
\end{equation}
\end{theorem}
The rather intricate formula for the function 
$\mathfrak f$ and $\mathfrak g$ can be found in the main body of the paper, but we will recall them as we go through the proof.

We also  recall  that the optimal kernel is given by the formula:
\begin{equation} \label{def:Kstar2}
K^\star(\bx,\by) = \; 
 \frac{1}{s_c^L} \; \;
\frac{\big|\{\vphi \in \partition : \vphi(x_\ell) = \vphi(y_\ell)  \text{ for all } 1 \le \ell \le L\}\big|}{|\Phi|} 
\end{equation}
The key insight to derive the upper bound \eqqref{stirling} is to note that each pair of sentences $(\bx,\by)$ induces a graph on the vocabulary $\{1,2,\ldots,n_w\}$, and that the quantity  $$\big|\{\vphi \in \partition : \vphi(x_\ell) = \vphi(y_\ell)  \text{ for all } 1 \le \ell \le L\}\big|$$
can be interpreted as the number of equipartitions of the vocabulary that do not sever any of the edges of the  graph. This graph-cut interpretation of the optimal kernel is presented in detail in Subsection \ref{section:graphcut}. 
In Subsection \ref{section:levelset} we derive a formula for $K^\star$ which is more tractable than \eqqref{def:Kstar2}. 
To do this we partition 
$\data \times \data$ into subsets on which $K^\star$ is constant, then  provide a formula for the value of $K^\star$ on each of these subsets (c.f. Lemma \ref{Klevelset}). With this formula at hand, we  then appeal to  Lemma
 \ref{lambda5} to derive a first bound for the permuted moment of $K^\star$ (c.f. Lemma \ref{lemma:first_go}). This first bound is not fully explicit because it involves the size of the subsets on which $K^\star$ is constant. In section \ref{section:forest} we appeal to Cayley's formula, a classical result from graph theory, to estimate the size of these subsets (c.f. Lemma \ref{iii1}) and therefore conclude the proof of Theorem \ref{thm:count}. 

We now introduce the combinatorial notations that will be used in this section, and we recall a few basics combinatorial facts.
We denote by $\nat =\{0,1,2,\ldots\}$ the nonnegative integers.  We use the standard notations
$$
{n \choose k} := \frac{n!}{k! (n-k)!}  \qquad \text{ and } \qquad  {n \choose k_1, k_2, \dots,  k_m}  := \frac{n!}{k_1! k_2! \cdots k_n! } 
$$
for the binomial and multinomial coefficients, with the understanding that $0! = 1$.
We recall  that multinomial coefficients can be interpreted as the number of ways of placing $n$ distinct objects into $m$ distinct bins, with the constraint that $k_1$ objects must go in the first bin, $k_2$ objects must  go in the second bin, and so forth.  

The Stirling numbers of the second kind   $\altchoose{n}{k}$ 
are close relatives of the binomial coefficients. 
  $\altchoose{n}{k}$ stands for the number of ways to partition a set of $n$ objects into $k$ nonempty subsets. To give a simple example, $\altchoose{4}{2}=7$ because there are 7 ways to partition  the set $\{1,2,3,4\}$ into two non-empty subsets, namely:
\begin{gather*}
\{1\} \cup \{2,3,4\}, \qquad \{2\} \cup \{1,3,4\}, \qquad  \{3\} \cup \{1,2,4\},  \qquad  \{4\} \cup \{1,2,3\}, \\
\{1,2\} \cup \{3,4\}, \qquad \{1,3\} \cup \{2,4\}, \qquad \{1,4\} \cup \{3,4\}.
\end{gather*}
Stirling numbers are easily computed via the following variant of Pascal's recurrence formula \footnote{Alternatively, Stirling numbers can be defined through the formula $
\altchoose{n}{k}= \frac{1}{k!} \sum_{i=0}^k(-1)^i { k \choose k-i } (k-i)^n \label{fofo}
$}: 
\begin{align*}
& \altchoose{n}{1} = 1, \qquad \altchoose{n}{n}=1  \qquad \text{ for } n \ge 1,\\
 & \altchoose{n}{k}=  \altchoose{n-1}{k-1} + k \; \altchoose{n-1}{k} \qquad \text{ for } 2 \le k \le n-1. 
\end{align*}
The above formula is
 easily derived from
  the definition of the Stirling numbers as providing the number of ways to partition a set of $n$ objects into $k$ nonempty subsets (see for example  chapter 6 of  \cite{graham1989concrete}).
 Table \ref{triangle} shows the  first few Stirling numbers.
 
 \begin{table}
  \caption{First five rows of the Strirling triangle for the Stirling numbers $\altchoose{n}{k}$.}
  \label{triangle}
  \centering
  \begin{tabular}{c|ccccc}
  $n \backslash k$& 1 & 2 & 3 & 4 & 5  \\
    \hline
  1 &1  &\\
  2  &1 & 1 \\
   3  & 1 & 3 & 1 \\
   4 &1 & 7 & 6 & 1 \\
   5 & 1&15 & 25 & 10 & 1 \\
  \end{tabular}
\end{table}

We recall that an undirected graph is an ordered pair $\mathcal G = (V,E)$, where $V$ is the vertex set and 
$$
E \subset \{\{v,v'\}: v,v' \in V \text{ and } v \neq v'\}
$$
is the edge set. Edges are unordered pairs of distinct vertices (so loops are not allowed.)
A tree is a connected graph with no cycles. A tree on $n$ vertices  has exactly $n-1$ edges. Cayley's formula states that there are $n^{n-2}$ ways to put  $n-1$ edges on  $n$ labeled vertices in order to make a tree. 
We formally state this classical result below:
\begin{lemma}[Cayley's formula] There are $n^{n-2}$ trees on $n$ labeled vertices. 
\end{lemma}

\subsection{Graph-Cut Formulation of the Optimal Kernel} \label{section:graphcut}
In this section we  consider  undirected graphs on the vertex set $$\voc:= \{1,2,\ldots,n_w\}.$$  Since the data space $\data$ consists of sentences of length $L$, graphs that have at most $L$ edges will be of particular importance.  We therefore define:
$$
\mathfrak G := \{ \text{All graphs on $\voc$ that have at most $L$ edges}\}. 
$$
In other words, $\mathfrak G$ consists in all the graphs $\mathcal G = (\voc, E)$ whose edge set $E$ has cardinality less or equal to $L$. Since these graphs all have the same vertex set, we will often identify them with their edge set.
We now introduce a mapping between pairs of sentences containing $L$ words, and graphs containing at most $L$ edges.
\begin{definition} \label{def:gmap} The function $\gmap: \data\times \data \to \mathfrak G$ is defined by
\begin{equation} \label{cici}
\gmap(\bx,\by) :=   \bigcup_{\substack{1 \le \ell \le L \\ x_\ell \neq y_\ell}} \{ \{x_\ell,y_\ell \} \big\}.  
\end{equation}
\end{definition}
The right hand side of \eqqref{cici} is a set of at most $L$ edges. Since graphs in $\mathfrak G$ are identified with their edge set, $\gmap$ indeed define a mapping from $\data \times \data$ to $\mathfrak G.$
 Let us give a few examples illustrating  how the map $\gmap$ works. Suppose we have a vocabulary of $n_w=10$ words and sentences of length $L=6$. Consider the pair of sentences $(\bx,\by) \in \data \times \data$ where
\begin{equation} \label{ex00}
\begin{matrix}
& \bx = [&2,&2, &8, &5, &9,   &7 &] \\
& \by = [&2,&5, &8, &2, &2 , &1 &] 
\end{matrix}
\end{equation}
Then   $\gmap(\bx,\by)$ is the set of 3 edges
$$
\gmap(\bx,\by)  = \Big\{\quad  \{2,5\}, \quad \{9,2\}, \quad \{7,1\} \quad  \Big\}.
$$
which indeed define a graph on $\voc$.
Note that position $\ell=2$ and $\ell=4$ of $(\bx,\by)$ `code' for the same edge $\{2,5\}$, position $5$ codes for the edge $\{9,2\}$, and position $6$ codes for the edge $\{7,1\}$. On the other hand, position $1$ and $3$ do not code for any edge: indeed, since $x_1 = y_1$ and $x_3=y_3$,  these two positions do not contribute any edges to the edge set defined by \eqqref{cici}. We will say that positions $1$ and $3$ are \emph{silent}.
We make this terminology formal in the definition below:
\begin{definition}
 Let  $(\bx,\by) \in \data \times \data$. If $x_\ell=y_\ell$ for some $1 \le \ell \le L$, we say that position $\ell$ of the pair $(\bx,\by)$ is silent.  If $x_\ell\neq y_\ell$ for some $1 \le \ell \le L$, we say that position $\ell$ of the pair $(\bx,\by)$
 codes for the edge $\{x_\ell,y_\ell\}$.
\end{definition}
Note that if $(\bx,\by)$ has some silent positions, or if multiple positions codes for the same edge, then
 the graph $\gmap(\bx,\by)$ will have strictly  less than $L$ edges. On the other hand, if none of these take place,
 then $\gmap(\bx,\by)$ will have exactly $L$ edges. For example the pair of sentences
\begin{equation} \label{ex01}
\begin{matrix}
& \bx = [&1, &1, &1, &5,   &6,& 7] \\
& \by = [&2, &3, &4, &6 , &7,& 1] 
\end{matrix}
\end{equation}
does not have silent positions, and all positions code for different edges. The corresponding graph
$$
\gmap(\bx,\by)  = \Big\{\quad  \{1,2\}, \quad \{1,3\}, \quad \{1,4\},  \quad \{5,6\}, \quad \{6,7\} , \quad \{7,1\}  \quad \Big\}
$$
has the maximal possible number of edges, namely $L=6$ edges. From the above discussion, it is clear that any graph with $L$ or less edges can be expressed as $\zeta(\bx,\by)$ for some pair of sentences $(\bx,\by) \in \data \times \data.$ Therefore  $\gmap: \data\times \data \to \mathfrak G$ is surjective. 
On the other hand, different pair of sentences can be mapped to the same graph. Therefore 
$\gmap$ is not injective.  We now introduce the following function.
\begin{definition}[Number of cut-free equipartitions of a graph]
The function $\mathcal I: \mathfrak G \to \nat$ is defined by :
\begin{equation} \label{def:I}
\mathcal I(\mathcal G) = |\{\vphi \in \partition   :  \vphi(v) = \vphi(v') \text{ for all edge $\{v,v'\}$ of the graph } \mathcal G \}|
\end{equation}
\end{definition}
Recall that $\Phi$ is the set of  maps $\vphi :  \voc  \to \{1, \ldots, n_c\}$ that satisfy $|\vphi^{-1}(c)| = s_c$ for all $1\le c \le n_c.$ Given a graph $\mathcal G$, the quantity $\mathcal I (\mathcal G)$ can therefore be interpreted as the number of ways to partition the vertices into  $n_c$ labelled subsets of equal size so that no edges are severed (i.e. two connected vertices must be in the same subset.) In other words, $\mathcal I (\mathcal G)$ is the number of ``cut-free'' equipartition of the graph $\mathcal G$, see Figure \ref{fig:cutfree} for an illustration. Note that if the graph $\mathcal G$ is connected, then $\mathcal I(\mathcal G)=0$ since any equipartition of the graph will severe some edges. On the other hand, if the graph $\mathcal G$ has no edges, then $\mathcal I(\mathcal G)=|\Phi|$ since there are no edges to be cut (and therefore any equipartition is acceptable.)

\begin{figure}
     \centering
     \begin{subfigure}[b]{0.2\textwidth}
         \centering
         \includegraphics[width=\textwidth]{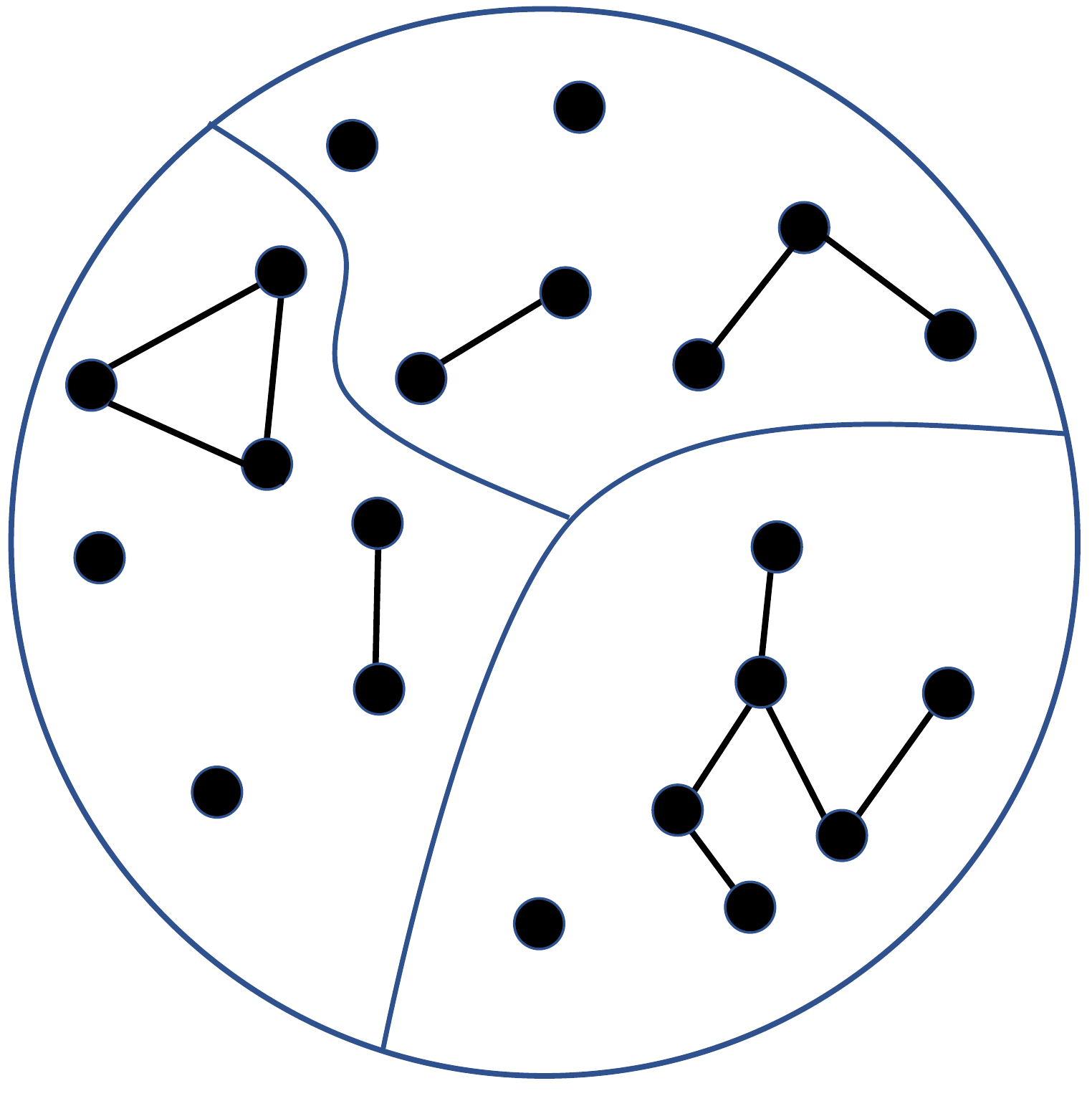}
         \caption{Cut-free}
         \label{a}
     \end{subfigure}
     \hspace{1.5cm}
     \begin{subfigure}[b]{0.2\textwidth}
         \centering
         \includegraphics[width=\textwidth]{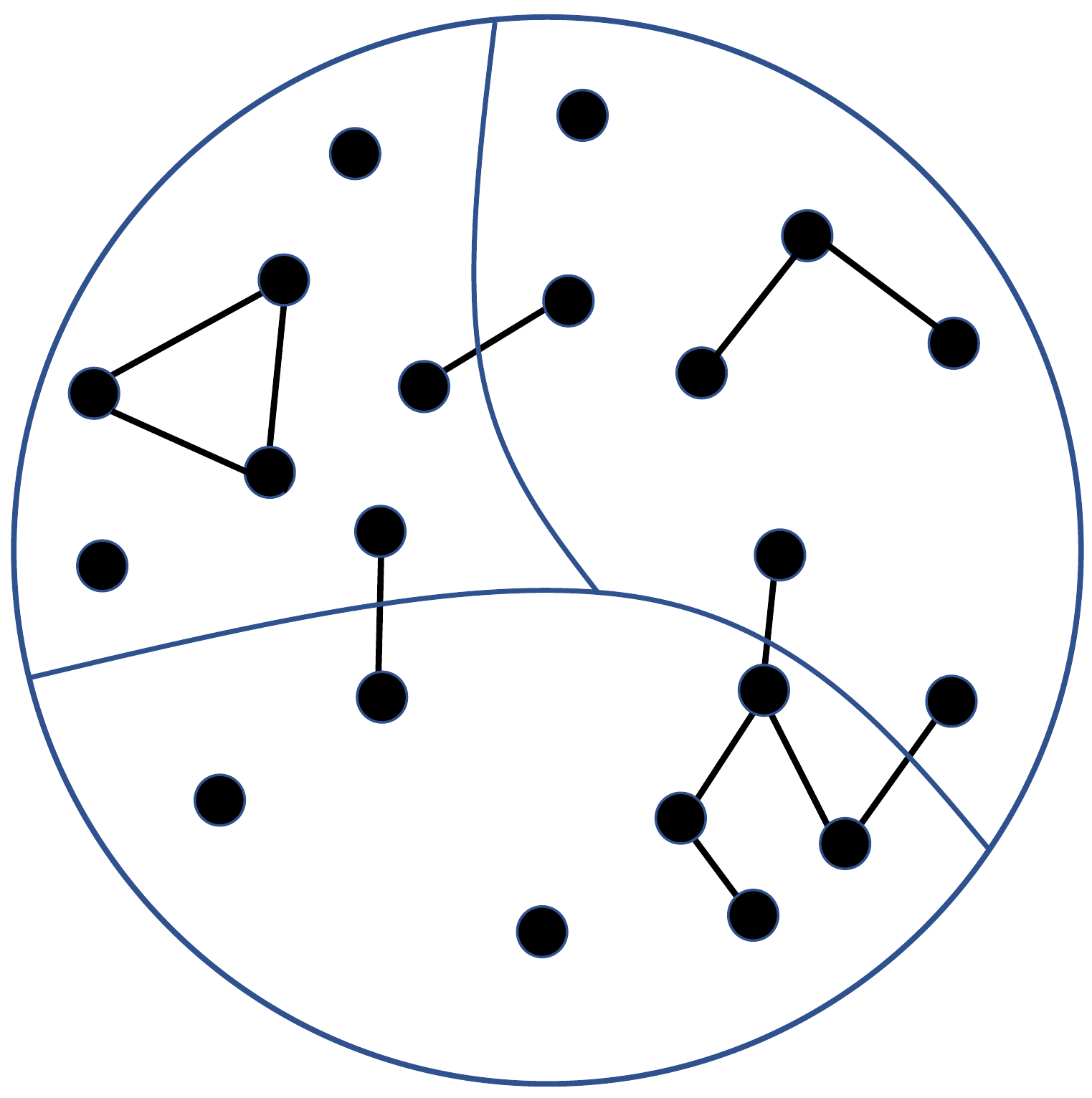}
         \caption{Not cut-free}
         \label{b}
     \end{subfigure}
        \caption{Two  equipartitions of the same graph (each subsets of the equipartitions contain $7$ vertices). The equipartition on the left is  cut-free  (no edges are severed). The equipartition on the right is not cut-free   (4 edges are severed). The optimal kernel $K^\star(\bx,\by)$ can be interpreted as the number of distinct cut-free equipartitions of the graph $\gmap(\bx,\by)$ (modulo some scaling factor.)}
        \label{fig:cutfree}
\end{figure}

The optimal kernel $K^\star$ can be expressed as a composition of the function $\gmap$ and $\mathcal I$. Indeed:
\begin{align}
K^\star(\bx,\by) 
 & =  \frac{1}{|\Phi| s_c^L} \;\;  |\{\vphi \in \partition : \vphi(x_\ell) = \vphi(y_\ell)  \text{ for all } 1 \le \ell \le L\}|  \label{nuni1}\\
 & =  \frac{1}{|\Phi| s_c^L} \;\;  |\{\; \vphi \in \partition   :  \vphi(v) = \vphi(v') \text{ for all $\{v,v'\}\in \gmap(\bx,\by)$ }  \}|  \label{nuni2}\\
 & =   \frac{1}{|\Phi| s_c^L} \;\;  \mathcal I( \gmap(\bx,\by))   \label{nuni3}
\end{align}
where we have simply used that $\gmap(\bx,\by) :=  \bigcup_{\substack{1 \le \ell \le L \\ x_\ell \neq y_\ell}} \{ \{x_\ell,y_\ell \} \big\}$ to go from \eqqref{nuni1} to \eqqref{nuni2}. We will refer to \eqqref{nuni3} as the \emph{graph-cut formulation} of the optimal kernel.

We have discussed earlier that the function $\gmap: \data \times \data \to \mathfrak G$ is surjective but not injective.  We conclude this subsection with a  lemma
 that provides an exact count of how many distinct $(\bx,\by)$ are mapped by $\gmap$ to a same graph $\mathcal G.$
\begin{lemma} \label{iii0} Suppose $\mathcal G \in \mathfrak G$ has $m$ edges.
Then 
$$
 |\gmap^{-1}(\mathcal G)| = |\{ (\bx,\by) \in \data \times \data : \gmap(\bx,\by) = \mathcal G \}|  = m!  \sum_{\alpha=m}^L  {L \choose \alpha}   \altchoose{\alpha}{ m }  2^{\alpha} n_w^{L-\alpha}.
$$
\end{lemma}

\begin{proof} 
Fix once and for all a graph  $\mathcal G = (\voc, E)$  with edge set $E =\{e_1,\ldots, e_m\}$ where $m \le L$.  Given $0\le \alpha \le L$, define the set
$$
\mathcal O_\alpha =\{ (\bx,\by) \in   \data \times \data:    \gmap(\bx,\by) = \mathcal G \text{ and }  (\bx,\by) \text{ has exactly $\alpha$ non-silent positions} \}.
$$
We start by noting that the set $\mathcal O_\alpha$ is empty for all $\alpha<m$: indeed, since  $\mathcal G$ has $m$ edges,  at least $m$ positions of a pair $(\bx,\by)$ must be coding for edges (i.e., must be non-silent) in order to have $\gmap(\bx,\by)=\mathcal G$. We therefore have the following partition:
$$
\gmap^{-1}(\mathcal G) =  \bigcup_{\alpha=m}^{L}  \mathcal O_\alpha   \qquad   \text{ and } \qquad  \mathcal O_\alpha \cap  \mathcal O_{\alpha'} = \emptyset \;\;  \text{ if } \alpha \neq \alpha'.
$$
To conclude the proof, we need to show that
\begin{equation} \label{gaff}
\left|\mathcal O_\alpha\right| = {L \choose \alpha}  \;  \altchoose{\alpha}{m}  \;  m!  \; 2^{\alpha} \;  n_w^{L-\alpha}  \qquad \text{for all } m \le \alpha \le L. 
\end{equation}
Consider the following process to generate an ordered pair $(\bx,\by)$ that belong to $\mathcal O_\alpha$:  we start by deciding which positions of (\bx,\by) are going to be silent, and which positions are going to code for which edge of the graph $\mathcal G$. This is equivalent to choosing a map $\rho: \{1,\ldots,L\} \to \{e_1, \ldots, e_m, s\}$ where $\{1,\ldots, L\}$ denotes the $L$ positions, $e_1, \ldots, e_m$ denote the $m$ edges of the graph $\mathcal G$, and $s$ is the silent symbol. Choosing a map $\rho$  correspond to assigning a ``role" to each position: $\rho(\ell)=e_i$ means that position $\ell$ is given the role to code for edge $e_i$, and $\rho(\ell)=s$ means that position $\ell$ is given the role of being silent. The map $\rho$ must satisfy:
\begin{align} \label{ooo}
|\rho^{-1}(s)| = L-\alpha \qquad \text{ and } \qquad \rho^{-1}(e_i) \neq \emptyset \;\;\; \text{ for $1\le i \le m$}
\end{align}
because $L-\alpha$ position must be silent and each edge must be represented by at least one position. The number of maps $\rho: \{1,\ldots,L\} \to \{e_1, \ldots, e_m, s\}$ that satisfies \eqqref{ooo} is equal to
\begin{equation*} 
{L \choose L-\alpha} \;\;  \altchoose{\alpha}{ m }  \;\;   m! 
\end{equation*}
Indeed, we  need to choose the $L-\alpha$ positions that will be mapped to the silent symbol $s$: there are ${L \choose L-\alpha}$ ways of accomplishing this. We then partition the $\alpha$ remaining positions into $m$ non-empty subsets: there are $\altchoose{\alpha}{ m }$ ways of accomplishing this. We finally map each of these non-empty subset to a different edge:
 there are $m!$ ways of accomplishing this. 
  
We have shown that there are  ${L \choose \alpha}  \altchoose{\alpha}{ m }   m!$  ways to assign roles to the positions.  Let say that position $\ell$ is assigned the role of coding for edge $\{v,v'\}.$ Then we have two choices to generate entries $x_\ell$ and $y_\ell$: either $x_\ell=v$ and $y_\ell=v'$, or $x_\ell=v'$ and $y_\ell=v.$ 
Since $\alpha$ positions are coding for edges, this lead to the factor $2^\alpha$ in equation \eqqref{gaff}. Finally, if the position $\ell$ is silent, then we have $n_w$ choices to generate  entries $x_\ell$ and $y_\ell$ (because we need  to choose $v\in \voc$ such that $x_\ell=y_\ell=v$)
, hence the factor $n_w^{L-\alpha}$ appearing in \eqqref{gaff}.
\end{proof}

\subsection{Level Sets of the Optimal Kernel}\label{section:levelset}

We recall that a connected component (or simply a component) of a graph is a connected subgraph  that is not part of any larger connected subgraph.
 Since graphs in  $\mathfrak G$ have at most $L$  edges, their components contain at most $L+1$ vertices. 
  This comes from the fact that the largest component that can be made with $L$ edges contains $L+1$ vertices.
It is therefore natural, given a vector $\bk = [k_1,\ldots,k_{L+1}] \in \nat^{L+1}$, to define
\begin{multline} \label{setGk}
\mathfrak G_{\bk} := \{ \mathcal G \in \mathfrak G: \mathcal G \text{ has exactly  $k_1$ components of size $1$, exactly $k_2$ components of size $2$,} \\  \text{\ldots,  exactly $k_{L+1}$ components of size $L+1$ }\}
\end{multline}
where the size of a component refers to the number of vertices it contains.  We recall that $\nat = \{0,1,2,\ldots\}$ therefore some of the entries of the vector $\bk$ can be equal to zero.  Note that components of size $1$ are simply isolated vertices.
The following lemma identify which $\bk \in \nat^{L+1}$ lead to non-empty  $\mathfrak G_\bk.$
\begin{lemma} The set $\mathfrak G_\bk$ is not empty if and only if $\bk$ satisfies
\begin{equation} \label{must}
\sum_{i=1}^{L+1} i k_i  = n_w \quad  \text{ and } \quad \sum_{i=1}^{L+1} (i-1) k_i \le L. 
\end{equation}
\end{lemma}
\begin{proof}
Suppose $\mathfrak G_\bk$ is not empty. Then there exists a graph $\mathcal G \in \mathfrak G$ that  has exactly $k_i$ components of size $i$, for $1 \le i \le L+1.$  A component of size $i$ contains $i$ vertices (by definition) and at least $i-1$ edges (otherwise it would not be connected.) Since $\mathcal G \in \mathfrak G$ it must have $n_w$ vertices and at most $L$ edges.  Therefore \eqqref{must} must hold.

Suppose that $\bk\in \nat^{L+1}$ satisfies \eqqref{must}. Then we can easily construct a graph $\mathcal G$ on $\voc$ that has a number of edges less or equal to $L$,   and that has exactly $k_i$ components of size $i$, for $1 \le i \le L+1.$  To do this we first partition the vertices into subsets so that there are  $k_i$ subsets of size $i$, for $1 \le i \le L+1.$ We then put $i-1$ edges on each subset of size $i$ so that they form connected components. The resulting graph  has $k_i$ components of size $i$, for $1 \le i \le L+1$,   and $\sum_{i=1}^{L+1} (i-1) k_i$ edges.
\end{proof}
The previous lemma allows us to partition  $\mathfrak G$ into non-empty subsets as follow:
\begin{align}
& \mathfrak G = \bigcup_{\bk \in \mathcal S} \mathfrak G_\bk, \qquad \mathfrak G_\bk \neq \emptyset \text{ for all } \bk\in \mathcal S,   \qquad  \text{and} \qquad   \mathfrak G_\bk \cap  \mathfrak G_{\bk'} = \emptyset \text{ if } \bk\neq \bk' \label{partition_graph} \\
& \text{where } \mathcal S:= \left\{ \bk \in \nat^{L+1}: \sum_{i=1}^{L+1} i k_i  = n_w \quad  \text{ and } \quad \sum_{i=1}^{L+1} (i-1) k_i \le L \right\}. 
\end{align}
Recall that  $\mathcal I(\mathcal G)$ count the number of equipartitions that do not severe edges of $\mathcal G.$ The next lemma shows
that two graphs that belongs to the same subset  $\mathfrak G_\bk$ have the same number of cut-free equipartitions, and it provides a formula for this number in term of the index $\bk$ of the subset.
\begin{lemma} \label{lemma:I} Suppose $\bk \in \mathcal S$ and define the set of admissible assignment matrices
\begin{equation} \label{setA}
 \mathcal A_{\bk} := \left\{A \in \mathbb N^{(L+1) \times n_c} \;\; : \;\;   \sum_{j=1}^{n_c} A_{ij}  = k_i  \text{ for all $ i$  \;\;\;  and  } \;\;\;  \sum_{i=1}^{L+1} i A_{ij} =  s_c  \text{ for all $j$  }  \right\}.
\end{equation}
Then for all  $\mathcal G \in \mathfrak G_\bk$, we have that
\begin{equation}
\mathcal I(\mathcal G) =  \sum_{A \in \mathcal A_{\bk} }  
\prod_{i=1}^{L+1}{k_i \choose A_{i,1} , A_{i,2}, \ldots, A_{i,n_c}}.
 \label{kako}
\end{equation}
\end{lemma}
Let us remark that, since $0!=1$,  the multinomial coefficient ${k_i \choose A_{i,1} , A_{i,2}, \ldots, A_{i,n_c}}$ appearing in \eqqref{kako} is equal to $1$ when  $k_i$ is equal to $0.$ 
\begin{proof}[Poof of Lemma \ref{lemma:I}] Let $\bk \in \mathcal S$ and fix once and for all a graph $\mathcal G \in \mathfrak G_\bk$. Define the set
\begin{equation*}
\Psi = \{\vphi \in \partition   :  \vphi(v) = \vphi(v') \text{ for all edge $\{v,v'\}$ of the graph } \mathcal G \} 
\end{equation*}
so that $\mathcal I (\mathcal G) = |\Psi|.$ Note that a map $\vphi$ that belongs to  $\Psi$ must map all vertices that are in a connected component to the same concept (otherwise some edges of $\mathcal G$ would be severed.) So a map $\vphi \in \Psi$ can be viewed as assigning connected components to concepts.  Given a matrix $A \in \nat^{(L+1)\times n_c}$ we define the set:
$$
\Psi_A = \{\vphi \in \Psi   : \text{ $\vphi$ assigns $A_{ij}$ components of size $i$  to concept $j$,  for all } 1 \le i \le L+1 \text{ and } 1 \le j \le n_c \}. 
 $$
We then note that the set $\Psi_A$ is empty unless the matrix $A$ satisfies:
\begin{align*}
&  A_{i,1} +  A_{i,2} +   A_{i,3} + \ldots + A_{i,n_c} = k_i \qquad \text{ for all } 1\le i \le L+1 \\
& A_{1,j}  + 2 A_{2,j}  + 3A_{3,j}   +  \ldots + (L+1) A_{L+1,j}  = s_c \qquad \text{ for all } 1\le j \le n_c.
\end{align*}
The first constraint states that the total number of connected components of size $i$ is equal to $k_i$ (because $\mathcal G \in \mathfrak G_\bk$). The second constraint states that concept $j$ must receive a total of $s_c$ vertices (because $\vphi \in \Phi$.) The matrices that satisfy these two constraints constitute the set  
 $\mathcal A_{\bk}$ defined in \eqqref{setA}. We therefore have the following partition of the set $\Psi$:
  \begin{equation*} 
 \Psi = \bigcup_{A \in  \mathcal A_\bk} \Psi_A,  \qquad  \Psi_A \neq \emptyset \text{ if $A \in \mathcal A_\bk$ },  \qquad   \Psi_A \cap \Psi_B = \emptyset  \text{ if $A \neq B$}.
 \end{equation*}
 To conclude the proof, we need to show that
 \begin{equation} \label{zouzou}
 |\Psi_A| = \prod_{i=1}^{L+1}{k_i \choose A_{i,1} , A_{i,2}, \ldots, A_{i,n_c}} \qquad \text{ for all } A \in \mathcal A_\bk.
 \end{equation}
 To see this, consider the $k_i$ components of size $i$. The number of ways to assign them to the $n_c$ concepts so that concept $j$ receives $A_{ij}$ of them is equal to the multinomial coefficient 
 ${k_i \choose A_{i,1} , A_{i,2}, \ldots, A_{i,n_c}}.$ Repeating this reasonning for the components of each size gives \eqqref{zouzou}.
\end{proof}
We now leverage the previous lemma to obtain a formula for $K^\star.$
For $\bk \in \mathcal S$ we define
$$\Omega_\bk := \gmap^{-1}(\mathfrak G_\bk).$$
Since $\gmap: \data \times \data \to \mathfrak G$ is surjective,  partition  \eqqref{partition_graph}  of $\mathfrak G$ induces the following
partition of $\data \times \data$:
\begin{equation} \label{level-set}
\data \times \data = \bigcup_{\bk \in \mathcal S} \Omega_\bk, \qquad \Omega_\bk \neq \emptyset \text{ if } \bk\in \mathcal S  \qquad  \text{and} \qquad  \Omega_\bk \cap  \Omega_{\bk'} = \emptyset \text{ if } \bk\neq \bk'. 
\end{equation}
Using the graph-cut formulation of the optimal kernel together with Lemma \ref{lemma:I} we therefore have  
\begin{equation} \label{zaza}
K^\star(\bx,\by)  = \frac{1}{|\Phi| s_c^L} \;\;  \mathcal I( \gmap(\bx,\by)) 
=   \frac{1}{|\Phi| s_c^L} \;\;  \sum_{A \in \mathcal A_{\bk} }    \prod_{i=1}^{L+1}{k_i \choose A_{i,1} , A_{i,2}, \ldots, A_{i,n_c}}  \quad \text{for all $(\bx,\by) \in \Omega_\bk.$}
\end{equation}
The above formula is key to our analysis. We restate it in the lemma below, but in a slightly different format that will better suit the rest of the analysis. Let  $\fcount: \mathcal S \to \real$ be the function defined by
\begin{equation} \label{f}
 \fcount(\bk):=  \frac{n_c^L \; (s_c!)^{n_c}}{ n_w! } 
 \sum_{A \in \mathcal A_{\bk} }  
\prod_{i=1}^{L+1}{k_i \choose A_{i,1} , A_{i,2}, \ldots, A_{i,n_c}}
\end{equation}
 We then have:
\begin{lemma}[Level set decomposition of $K^\star$] \label{Klevelset} The kernel $K^\star$ is constant on each subsets $\Omega_\bk$ of the partition \eqqref{level-set}. Moreover we have
$$
 K^\star(\bx,\by) = \fcount(\bk) / |\data| \qquad  \text{for all }(\bx,\by) \in \Omega_\bk \text{ and for all } \bk \in \mathcal S.
 $$
\end{lemma}
\begin{proof} 
The quantity $|\Phi|$ appearing in  \eqqref{zaza} can be interpreted as the number of ways to assign the $n_w$ words to the $n_c$ concepts so that each concept receives $s_c$ words. Therefore
 \begin{equation*} 
 |\Phi| = {n_w \choose s_c, s_c, \ldots, s_c} = \frac{n_w!}{ (s_c!)^{n_c}}.
 \end{equation*}
Combined with the fact that $|\data|=n_w^L$, this leads to the desired formula for $K^\star.$
 \end{proof}
The above lemma provides us with the level sets of the optimal kernel.
Together with  Lemma  \ref{lambda5}, this allows us   to  derive the following upper bound for the permuted moment of $K^\star$. 
\begin{lemma} \label{lemma:first_go} Let  $\Omega = \data \times \data.$ The inequality
$$
\frac{1}{|\data|} \sum_{\bx \in \data} \entropy_t(  K^\star_\bx) \le  \left( 1 -  \sum_{\bk \in \mathcal S'} \frac{|\Omega_{\bk}|}{|\Omega|} \fcount(\bk)  \right)  + \frac{1}{t+1} \left(  \max_{\bk \in \mathcal S'} \fcount(\bk) \right)
$$
holds for all  $\mathcal S' \subset \mathcal S$.
\end{lemma}
\begin{proof} Let  $\mathcal S' \subset \mathcal S$ and define:
$$
 \lambda= \max_{\bk \in \mathcal S'} \max_{ (\bx,\by) \in \Omega_\bk} K^\star(\bx,\by) = \frac{1}{|\data|} \max_{\bk \in \mathcal S'}  \fcount(\bk)
$$
where we have used the fact that $K^\star$ is equal to $\fcount(\bk)/|\data|$ on $\Omega_\bk.$ We then appeal to Lemma  \ref{lambda5} to obtain:
\begin{align}
 \frac{1}{|\data|} \sum_{\bx \in \data} \entropy_t \left(K^\star(\bx, \cdot) \right) 
& \le \frac{1}{|\data|} \sum_{\bx \in \data}  \left( \frac{\lambda |\data|}{t+1}  +  1-  \sum_{\by \in \data} \min \{K^\star(\bx,\by) , \lambda \}  \right) \\
& = \frac{\lambda |\data|}{t+1} + 1 -  \frac{1}{|\data|} \sum_{\bx \in \data}   \sum_{\by \in \data} \min \{K^\star(\bx,\by) , \lambda \} \label{wewe0} \\
& =  \frac{\lambda |\data|}{t+1} + 1 -  \frac{1}{|\data|} \sum_{\bk \in \mathcal S}   \sum_{(\bx,\by) \in \Omega_\bk} \min \{K^\star(\bx,\by) ,  \lambda \} \label{wewe1}  \\
& \le  \frac{\lambda |\data|}{t+1} + 1 -  \frac{1}{|\data|} \sum_{\bk \in \mathcal S'}   \sum_{(\bx,\by) \in \Omega_\bk} \min \{K^\star(\bx,\by) ,  \lambda \} \label{wewe2}  \\
& =   \frac{\lambda |\data|}{t+1} + 1 -  \frac{1}{|\data|} \sum_{\bk \in \mathcal S'}   \sum_{(\bx,\by) \in \Omega_\bk} K^\star(\bx,\by) \label{wewe3} \\
& =   \frac{\lambda |\data|}{t+1} + 1 -  \frac{1}{|\data|} \sum_{\bk \in \mathcal S'}  |\Omega_\bk|  \frac{f(\bk)}{|\data|}
\end{align}
where we have use the fact that $\data \times \data = \bigcup_{\bk \in \mathcal S} \Omega_\bk$ to go from \eqqref{wewe0} to \eqqref{wewe1}, and the fact that $K^\star(\bx,\by)\le \lambda$ on $ \bigcup_{\bk \in \mathcal S'} \Omega_\bk$  to go from \eqqref{wewe2} to \eqqref{wewe3}.  To conclude the proof, we simply note that $\lambda |\data| =  \max_{\bk \in \mathcal S'}  \fcount(\bk) $
according to our definition of $\lambda.$
\end{proof}
The bound provided by the above lemma is not fully explicit because it involves the size of level sets $\Omega_\bk$. In the next section,  we appeal to Cayley's formula to obtain a lower bound for $|\Omega_\bk|.$

\subsection{Forest Lower Bound for  the Size of the Level Sets} \label{section:forest}
We recall that a forest is
 a graph whose connected components are trees (equivalently, a forest is a graph with no cycles.) Let us define:
$$
\mathfrak F := \{\mathcal G \in \mathfrak G: \mathcal G \text{ is a forest}\}.
$$
In other words, $\mathfrak F$ is the set of forests on $\mathcal V = \{1,2,\ldots,n_w\}$ that have at most $L$ edges.
We obviously have the following lower bound on the size of the level sets: 
\begin{equation} \label{fofo}
|\Omega_\bk| =  \left| \gmap^{-1}(\mathfrak G_\bk) \right| \ge  \left| \gmap^{-1}(\mathfrak G_\bk \cap \mathfrak F) \right|. 
\end{equation}
In this subsection, we use Cayley's formula to derive an explicit formula for  $\left| \gmap^{-1}(\mathfrak G_\bk \cap \mathfrak F) \right|.$ We start with the following lemma:

 \begin{lemma} \label{iii1}
 Let $\bk \in \mathcal S$, then
\begin{equation} \label{yaki}
|\mathfrak G_{\bf k} \cap \mathfrak F| =  \frac{n_w!}{k_1! k_2! \cdots k_{L+1}!} \prod_{i=2}^{L+1}\left( \frac{i^{i-2}}{i!} \right)^{k_i}
\end{equation}
 \end{lemma}
\begin{proof}
First we note that \eqqref{yaki} can be written as
$$
|\mathfrak G_{\bf k} \cap \mathfrak F| =   {n_w \choose  k_1, 2k_2, \ldots,  (L+1)k_{L+1}} \;\; \prod_{i=2}^{L+1} \frac{ i^{k_i(i-2)} }{k_i!}   { i k_i \choose i ,i, \ldots, i}   
 $$
 We now explain the above formula. The set $\mathfrak G_{\bf k} \cap \mathfrak F$ consists in all the forests that have exactly $k_1$ trees of size $1$, $k_2$ trees of size $2$, \ldots, $k_{L+1}$ trees of size $L+1$.  In order to construct a forest with this specific structure, we start by assigning the $n_w$ vertices to $L+1$ bins, with bin 1 receiving $k_1$ vertices, bin 2 receiving $2k_2$ vertices, \ldots, bin $L+1$ receiving $(L+1)k_{L+1}$ vertices. The number of ways of accomplishing this is $${n_w \choose  k_1, 2k_2, \ldots,  (L+1)k_{L+1}}.$$
  Let us now consider the vertices in bin $i$ for some $i \ge 2$. We claim that there are 
 $$
 \frac{1}{k_i!}   { i k_i \choose i ,i, \ldots, i}   i^{k_i(i-2)}
 $$
 ways of putting edges on these $i k_i$ vertices in order to
 make $k_i$ trees of size $i$. Indeed, there are  $\frac{1}{k_i!}   { i k_i \choose i ,i, \ldots, i} $ ways of partitioning the vertices into  $k_i$ disjoint subsets of size $i$, and then, according to Cayley's formula, there are $i^{i-2}$ ways of putting edges on each of these subsets so that they form a tree.
 To conclude, we remark that there is obviously only one way to to make $k_1$ trees of size 1 out of the $k_1$ vertices in the first bin.
 \end{proof}

Recall that a tree on $n$ vertices always has $n-1$ edges. So a graph that belongs to $\mathfrak G_{\bf k} \cap \mathfrak F$ has
$$
m=\sum_{i=1}^{L+1} (i-1) k_i
$$
edges since it is made  of $k_1$ trees of size $1$,  $k_2$ trees of size $2$, \ldots, $k_{L+1}$ trees of size $(L+1)$. 
The fact that all graphs in $\mathfrak G_{\bf k} \cap \mathfrak F$ have the same number of edges
 allows us to 
to obtain an explicit formula for $|\gmap^{-1}(\mathfrak G_{\bf k} \cap \mathfrak F )|$ by combining
combine Lemma \ref{iii0} and \ref{iii1}, namely
\begin{equation*}
|\gmap^{-1}(\mathfrak G_{\bf k} \cap \mathfrak F )| = \left( \frac{n_w!}{k_1! k_2! \cdots k_{L+1}!} \prod_{i=2}^{L+1}\left( \frac{i^{i-2}}{i!} \right)^{k_i}\right) \left( m!  \sum_{\alpha=m}^L  {L \choose \alpha}   \altchoose{\alpha}{ m }  2^{\alpha} n_w^{L-\alpha} \right).
\end{equation*}
This lead us to define the function $\mathfrak g: \mathcal S \to \real$ by 
\begin{gather*}
 \mathfrak g (\bk) = \frac{1}{n_w^{2L}} \left( \frac{n_w!}{k_1! k_2! \cdots k_{L+1}!} \prod_{i=2}^{L+1}\left( \frac{i^{i-2}}{i!} \right)^{k_i}\right) \left( \gamma(\bk)!  \sum_{\alpha=\gamma(\bk)}^L  {L \choose \alpha}   \altchoose{\alpha}{ \gamma({\bk}) }  2^{\alpha} n_w^{L-\alpha} \right) \label{titu} \\
 \quad \text{where } \gamma(\bk) = \sum_{i=1}^{L+1} (i-1) k_i.
\end{gather*}
Recalling \eqqref{fofo} we therefore have that
$$
\frac{|\Omega_\bk|}{|\Omega|} \ge \mathfrak g(\bk) \qquad \text{ for all }\bk \in \mathcal S.
$$
Combining the above inequality with Lemma \ref{lemma:first_go} we obtain:
\begin{theorem} \label{theorem:general} The inequality
\begin{equation*} 
\frac{1}{|\data|} \sum_{\bx \in \data} \entropy_t(  K^\star_\bx) \le  \left( 1 -  \sum_{\bk \in \mathcal S'} \mathfrak g(\bk) \fcount(\bk)  \right)  + \frac{1}{t+1} \left(  \max_{\bk \in \mathcal S'} \fcount(\bk) \right)
\end{equation*}
holds for all  $\mathcal S' \subset \mathcal S$.
\end{theorem}
The above theorem is more general than Theorem \ref{thm:count} --- indeed, 
in Theorem \ref{thm:count}, the choice of the subset $\mathcal S'$ is restricted to the $L+1$ candidates:
$$
 \mathcal S_\ell:= \left\{ \bk \in \nat^{L+1}:  \;\; \sum_{i=1}^{L+1} i k_i  = n_w \quad  \text{ and } \quad \ell \le \sum_{i=1}^{L+1} (i-1) k_i \le L \right\}\qquad \text{ where }  \ell = 0,1, \ldots, L.
$$  
When $L=9$, $n_w=150$, $n_c=5$ and $t =1999$ (these are the parameters used in Theorem \ref{thm:special}), choosing $\mathcal S' = \mathcal S_7$ leads to a relatively tight upper bound.  When $L=15$, $n_w=30$, $n_c=5$ and $t=5999$ (these are the parameters corresponding the the second experiment of the experimental section), choosing  $\mathcal S' = \mathcal S_{11}$ gives a good upper bound.


\section{Multiple  Unfamiliar Sentences per Category} \label{appendix:multiple}

 \begin{figure}[t]
          \centering
         \includegraphics[totalheight=0.35\textheight]{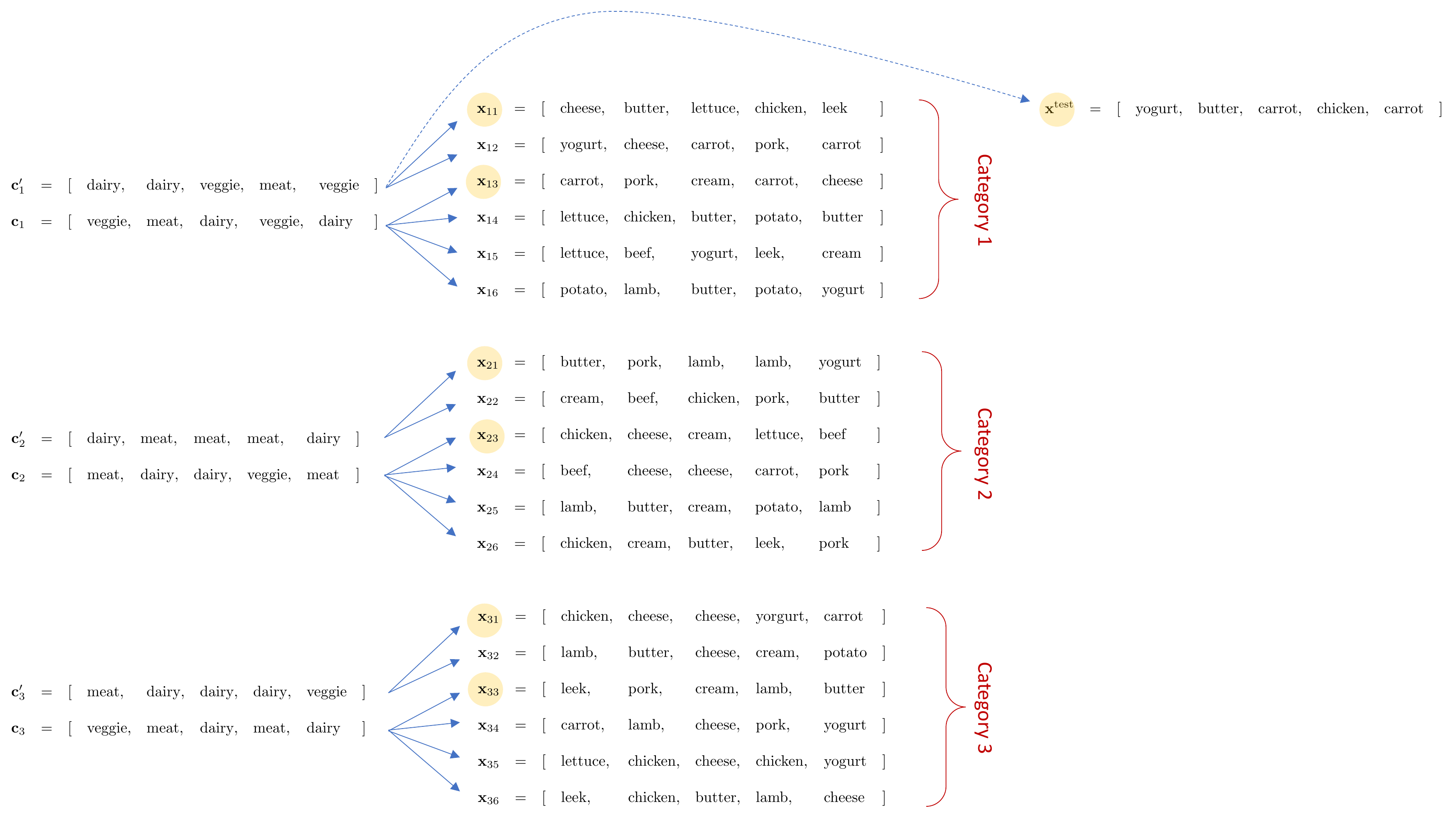}
            \caption{Data Model with $n^*=2$ unfamiliar sentences per category. The other parameters in this example are set to  $L=5$, $n_w=12$, $n_c=3$, $R=3$ and   $n_\text{spl} = 6$. The points highlighted in yellow are the ones involved in the definition of the event $A^{(1)}$, see equation \eqqref{A1yellow}.}
            \label{figure:foodmult-test}
\end{figure}

In the data model depicted in Figure \ref{figure:food}, each unfamiliar sequence of concept has a single representative in the training set. In this section we consider the more general case in which each unfamiliar sequence of concepts has $n^*$ representatives in the training set,  where
$$
1 \le n^* \le n_{\rm spl}.
$$
An example with $n^*=2$ is depicted in Figure \ref{figure:foodmult-test}.
The variables $L, n_w, n_c, R, n_\text{spl}$ and $n^*$ parametrize instances of this more general data model, and  the associated sampling process is:

{\bf Sampling Process DM2:} 

\noindent\fbox{%
    \parbox{\textwidth}{%
\begin{enumerate}[label=(\roman*)]
 \item Sample  $\mathcal T = (\; \vphi \;  ; \;   \bc_1, \ldots, \bc_R \; ; \;  \bc'_1, \ldots, \bc'_R \; ) $ uniformly at random in  $\mathfrak T=\Phi \times \mathcal Z^{2R}$.
 \item For $r=1,\ldots, R$:
 \begin{itemize}
  \item Sample $(\bx_{r,1} , \ldots,  \bx_{r,n^*})$ uniformly at random in $\ephi^{-1}(\bc'_r) \times \ldots \times \ephi^{-1}(\bc'_r)$. %
 \item Sample $(\bx_{r,n^*+1} , \ldots,  \bx_{r,n_\text{spl}})$ uniformly at random in $\ephi^{-1}(\bc_r) \times \ldots \times \ephi^{-1}(\bc_r)$. 
 \end{itemize}
 \item Sample  $\bx^\text{test}$  uniformly at random in $\ephi^{-1}(\bc'_1)$.
\end{enumerate}
}
}

Our analysis easily adapts to this more general case  and gives: 
\begin{theorem} \label{thm:main2}
Let   $\mathfrak T = \Phi \times \mathcal Z^{2R}$. 
 Then 
\begin{equation} \label{main_lower_bound2}
1 - \overline{\text{err}}(\featspace, \psi, \mathfrak T) \le 
n^*
\left[ \left( 1 -  \sum_{\bk \in \mathcal S_\ell} \fcount(\bk)  \mathfrak g(\bk)  \right)  + \frac{1}{2R} \left(  \max_{\bk \in \mathcal S_\ell} \fcount(\bk) \right)  \right] + \frac{1}{R}
\end{equation}
for all  feature space  $\featspace$, all feature map $\featmap: \data \mapsto \featspace$, and
 all $0 \le \ell \le L$.
\end{theorem}
Theorem  \ref{thm:main}  in the main body of the paper can be viewed as a special case of the above theorem --- indeed, setting $n^*=1$ in inequality \eqqref{main_lower_bound2} we exactly recover \eqqref{main_lower_bound}.

In order to prove Theorem \ref{thm:main2}, we simply need to revisit subsection \ref{section:secondpart}. 
We denote by  $\Omega_{\rm DM2}$ and $\mathbb P_{{\rm DM2}}$ the  sample space and probability measure associated with the sampling process DM2. 
As in subsection  \ref{section:secondpart},  given a kernel $K \in \mathcal K$,  we define the event
\begin{multline*}
 E_K= \Big\{ \omega \in \Omega_{\rm DM2} : \;\; \text{There exists } 1 \le s^* \le \nspl \text{ such that } \\ 
  K\left(\bx^\text{test}, \bx_{1,s^*} \right) >  K\left(\bx^\text{test}, \bx_{r,s} \right) \text{ for all } 2 \le r \le R \text{ and all } 1 \le s \le \nspl 
 \Big\}.
 \end{multline*}
Such event describes all outcomes corresponding to successful classification of the test point $\bx^{\rm test}$. 
For simplicity let us assume that $n^*=2$ (therefore matching the scenario depicted in Figure \ref{figure:foodmult-test}). 
We further partition the event $E_K$ according to which training point from the first category is most similar to the test point:
\begin{equation} \label{tutut97}
E_K = E^{(1)}_\text{meaningful} \cup E^{(2)}_\text{meaningful} \cup E_\text{luck}
\end{equation}
The event $E^{(1)}_{\text{meaningful}}$ consists in all the outcomes where, among the points from first category, $\bx_{1,1}$ is the most similar to $\bx^{\rm test}$,  $E^{(2)}_{\text{meaningful}}$ consists in all the outcomes where, among the points from first category, $\bx_{1,2}$ is the most similar to $\bx^{\rm test}$, and $E_{\text{luck}}$ consists in all the remaining cases. Formally we have:
\begin{align*}
  E^{(1)}_{\text{meaningful}} =  E_K &\; \cap \;  \Big\{ \omega \in \Omega_{\rm DM2}:  \;\;  K\left(\bx^\text{test}, \bx_{1,1} \right) >  K\left(\bx^\text{test}, \bx_{1,2} \right)   \Big\}  \\ & 
  \; \cap \;  \Big\{ \omega \in \Omega_{\rm DM2}:  \;\;  K\left(\bx^\text{test}, \bx_{1,1} \right) >  K\left(\bx^\text{test}, \bx_{1,s} \right)   \text{ for all }  3 \le s \le \nspl \Big\}  
  \end{align*}
  \begin{align*}
  E^{(2)}_{\text{meaningful}} =  E_K & \; \cap \;  \Big\{ \omega \in \Omega_{\rm DM2}:  \;\;  K\left(\bx^\text{test}, \bx_{1,2} \right) \ge  K\left(\bx^\text{test}, \bx_{1,1} \right)   \Big\}  \\ & 
 \; \cap \;  \Big\{ \omega \in \Omega_{\rm DM2}:  \;\;  K\left(\bx^\text{test}, \bx_{1,2} \right) >  K\left(\bx^\text{test}, \bx_{1,s} \right)   \text{ for all }  3 \le s \le \nspl \Big\}  
 \end{align*}
 \begin{align*}
 E_{\text{luck}} = E_K&  \; \cap \;  \Big\{ \omega \in \Omega_{\rm DM2}:  \;\; \text{there exists } 3 \le s^* \le \nspl \text{ such that } \\
 & \hspace{4cm}   K\left(\bx^\text{test}, \bx_{1,s^*} \right) \ge  K\left(\bx^\text{test}, \bx_{1,s} \right)   \text{ for all }  1 \le s \le \nspl 
 \Big\} 
  \end{align*}

Exactly as in subsection \ref{section:secondpart}, we then prove that
\begin{align}
&\mathbb P_{\rm DM2} [E^{(i)}_\text{meaningful} ] \le \frac{1}{|\data|} \sum_{\bx \in \data} \mathcal H_{2R-1}( K^\star_\bx )  \quad \text{for } i=1,2 \label{coco1}\\
&\mathbb P_{\rm DM2} [E_\text{luck} ] \le \frac{1}{R}. \label{coco2}
\end{align}
The proof of inequality \eqqref{coco1} is essentially identical to the proof of Lemma \ref{lemma:bigigi}. We define the event  
\begin{multline}\label{A1yellow}
A^{(1)} := 
\Big\{ \omega \in \Omega_{\rm DM2} : \;\;
  K\left(\bx^\text{test}, \bx_{1,1} \right) >  K\left(\bx^\text{test}, \bx_{r,1} \right) \text{ for all } 2 \le r \le R 
 \Big\} \\
\cap
\Big\{ \omega \in \Omega_{\rm DM2} : \;\;
  K\left(\bx^\text{test}, \bx_{1,1} \right) >  K\left(\bx^\text{test}, \bx_{r,3} \right) \text{ for all } 1 \le r \le R 
 \Big\}  \end{multline}
and the event 
\begin{multline*}
A^{(2)} := 
\Big\{ \omega \in \Omega_{\rm DM2} : \;\;
  K\left(\bx^\text{test}, \bx_{1,2} \right) >  K\left(\bx^\text{test}, \bx_{r,2} \right) \text{ for all } 2 \le r \le R 
 \Big\} \\
\cap
\Big\{ \omega \in \Omega_{\rm DM2} : \;\;
  K\left(\bx^\text{test}, \bx_{1,2} \right) >  K\left(\bx^\text{test}, \bx_{r,3} \right) \text{ for all } 1 \le r \le R 
 \Big\}  \end{multline*}
The  $\bx$'s involved in the definition of the event $A^{(1)}$ are highlighted in yellow in Figure \ref{figure:foodmult-test}. Crucially they are all generated by different sequences of concepts, except for $\bx_{1,1}$ and $\bx^{\rm test}$. 
We can therefore appeal to  Theorem \ref{theorem:beautiful} to obtain
 $$
 \mathbb P_{\rm DM2}[A^{(1)}] \le   \frac{1}{|\data|} \sum_{\bx \in \data} \mathcal H_{2R-1}( K^\star_\bx ) 
 $$
 since there is a total of $t=2R-1$ `distractors' (the `distractors' in Figure  \ref{figure:foodmult-test} are $\bx_{1,3}$, $\bx_{2,1}$, $\bx_{2,3}$, $\bx_{3,1}$ and $\bx_{3,3}$).
 We then use the fact 
$E^{(1)}_\text{meaningful} \subset A^{(1)}$ to obtain \eqqref{coco1} with $i=1$. The case $i=2$ is exactly similar.

 We now prove \eqqref{coco2}. The proof is similar to the proof of Lemma \ref{lemma:bigigi2}.  For $1 \le r \le R$, we  define the events
 \begin{align*}
 B_r = \bigcap_{\substack{1 \le r' \le R \\ r'\neq r}} \left\{\omega \in \Omega_{\rm DM2}: \;\;  \max_{3 \le s \le \nspl} K(\bxt , \bx_{r,s}) > \max_{3 \le s' \le \nspl} K(\bxt , \bx_{r',s'}) \right\}
 \end{align*}
 By symmetry, these events are equiprobable. They also are mutually disjoints, and therefore $\mathbb P_{\rm DM2}[B_r] \le 1/R$. Inequality  \eqqref{coco2} then comes from the fact that $E_{\rm luck} \subset B_1$.

Combining \eqqref{tutut97}, \eqqref{coco1}, \eqqref{coco2} then gives
 \begin{equation} \label{albator79}
\sup_{K \in \mathcal K}   \mathbb P_{{\rm DM2}}\Big[  E_K \Big] \le   \frac2{|\data|} \sum_{\bx \in \data} \mathcal{H}_{2R-1}\left(K^\star_\bx\right) + \frac1{R}.
\end{equation}  
and, going back to the general case where $n^*$ denotes the number of representatives that each sequence of unfamiliar concepts has in the training set, 
 \begin{equation} \label{albator7999}
\sup_{K \in \mathcal K}   \mathbb P_{{\rm DM2}}\Big[  E_K \Big] \le   \frac{n^*}{|\data|} \sum_{\bx \in \data} \mathcal{H}_{2R-1}\left(K^\star_\bx\right) + \frac1{R}
\end{equation}  
which in turn implies \eqqref{connection11}. Combining inequalities   
\eqqref{upper_bound_pm} and  \eqqref{connection11} then  concludes the proof of Theorem \ref{thm:main2}.


\section{Details of the Experiments} \label{section:ex}

In this section we provide the  details of the  experiments described in Section \ref{section:empirical}, as well as additional experiments. Table \ref{table50}  provides the results of experiments in which the  parameters $L$, $n_w$, $n_c$ and $R$  are set to 
$$
 L=9, \quad n_w=150, \quad n_c=5,  \quad R=1000. 
 $$
 The parameters $n_{\rm spl}$ and $n^*$ are chosen so that the training set contains $5$ familiar sentences per category, and between $1$ and $5$ unfamiliar sentences per category.  Table  \ref{table150} is identical to Table \ref{thetablemain}  in Section \ref{section:empirical}, with the exception that it contains additional information  (i.e. the standard deviations of the obtained accuracies). The abbreviation NN appearing in  Table \ref{table150}   stands for `Nearest Neighbor'.
  
   Table \ref{table50}  provides the results of additional experiments
    in which the parameters $L$, $n_w$, $n_c$ and $R$ are set to 
$$
 L=9, \quad n_w=50, \quad n_c=5,  \quad R=1000. 
 $$
 The parameters $n_{\rm spl}$ and $n^*$ are chosen, as in the previous set experiments, so that the training set contains $5$ familiar sentences per category, and between $1$ and $5$ unfamiliar sentences per category. The tasks considered in this set of experiments are easier due to the fact that the vocabulary is smaller ($n_w=50$ instead of $n_w=150$).

 \begin{table}
  \caption{Accuracy in $\%$ on unfamiliar test points ($L=9$, $n_w=150$, $n_c=5$, $R=1000$).}
  \label{table150}
  \centering {\small
  \begin{tabular}{llllll}
    \toprule
               &    \small{$n^*\!=\!1$}   & \small{$n^*\!=\!2$}  &    \small{$n^*\!=\!3$}  & \small{$n^*\!=\!4$}  & \small{$n^*\!=\!5$}  \\
                      &    \small{$n_{\rm spl}\!=\!6$}   & \small{$n_{\rm spl}\!=\!7$}  &    \small{$n_{\rm spl}\!=\!8$}  & \small{$n_{\rm spl}\!=\!9$}  & \small{$n_{\rm spl}\!=\!10$}  \\
        \midrule
    Neural network   &  $99.8 \pm 0.3$ &  $99.9 \pm 0.1 $   & $99.9 \pm 0.1$ & $99.9 \pm 0.1$ & $100 \pm0.1$   \\
     {NN on feat. extracted by neural net}   &  $99.9 \pm 0.1$ &  $99.9 \pm 0.1$   & $99.9 \pm 0.1$& $99.9 \pm 0.1$&$99.9 \pm 0.1$  \\
     \midrule
       {NN on feat. extracted by $\psi^\star$}     &  $0.7 \pm 0.2$ & $1.1 \pm 0.3$    &$1.5 \pm 0.3$ & $1.8 \pm 0.3$  & $2.2 \pm 0.3$  \\ 
        NN on feat. extracted by $\psi_{\rm one-hot}$     &                                     $0.6 \pm 0.2$ & $1.1 \pm 0.3$    &$1.4 \pm 0.2$ & $1.7 \pm 0.3$  & $2.1 \pm 0.3$  \\
     Upper bound  ($0.015 n^* + 1/1000$)   &  $1.6$ & $3.1$    &$4.6$ & $6.1$  & $7.6$  \\
     \midrule
             SVM  on feat. extracted by $\psi^\star$       &  $0.6 \pm 0.3$    &  $1.5 \pm 0.4 $  &$2.2 \pm 0.4$& $3.2 \pm 0.6$ & $4.2 \pm 1.0$\\
       SVM  on feat. extracted by $\psi_{\rm one-hot}$      &  $0.5 \pm 0.1$ & $1.1\pm0.1$    &$1.9 \pm 0.1$ & $2.8 \pm 0.2$  & $3.8 \pm 0.2$  \\
        SVM with Gaussian kernel      &   $0.6 \pm 0.1$    &  $1.1 \pm 0.1$ &$2.0 \pm 0.1$&$2.8 \pm 0.2$& $3.6 \pm 0.2$ \\
    \bottomrule
  \end{tabular}}
\end{table}

\subsection{Neural Network Experiments} \label{sub:nn}

We consider the neural network in Figure  \ref{figure:mlpmixer}. 
The number of neurons  in the input, hidden and output layers of the MLPs constituting the neural network  are set to:
$$
\text{MLP 1:} \;\; d_\text{in} = 150, d_\text{hidden} = 500, d_{out} = 10, \qquad   \text{MLP 2:} \;\;  d_{in} = 90, d_\text{hidden} = 2000, d_\text{out} = 1000.
$$ 
For each of the 10 possible parameter settings in Table \ref{table150} and Table \ref{table50},  we do $104$ experiments.  For each experiment we generate:
\begin{itemize}
 \item A training set containing $R\times n_{\rm spl}$ sentences.
 \item A test set containing $10,000$ unfamiliar sentences (10 sentences per category).
 \end{itemize}
 We then train the neural network with  stochastic gradient descent until the training loss reaches $10^{-4}$ (we use a cross entropy loss).  The learning rate
 is set to  $0.01$ (constant learning rate), and the batch size  to $100$. 
 At test time, we either use the neural network to classify the test points (first row of the tables) or we use a nearest neighbor classification rule on the top of the features extracted by the neural network (second row of the tables). The mean and standard deviation of the $104$ test accuracies, for each of the 10 settings, and for each of the two evaluation strategies, is reported in the first two rows of 
Table \ref{table150} and Table \ref{table50}.

\subsection{Nearest-Neighbor Experiments} \label{appendix:nearest}
In these experiments we use a nearest neighbor classification rule on the top of features extracted by $\psi^\star$ (third row of Table \ref{table150} and \ref{table50}) or $\psi_{\rm one-hot}$ (fourth row). For each of the 10 possible parameter settings in Table \ref{table150} and Table \ref{table50},  we do $50$ experiments.  For each experiment we generate:
\begin{itemize}
 \item A training set containing $R\times n_{\rm spl}$ sentences.
 \item A test set containing $1,000$ unfamiliar sentences (one sentences per category).
 \end{itemize}
In order  to perform the nearest neighbor classification rule on the features extracted by $\psi^\star$, one  needs to evaluate the kernel $K^\star (\bx,\by) = \langle \psi^\star(\bx), \psi^\star(\by) \rangle_{\featspace^\star}$ for each pair of sentences. 
 Computing $K^\star(\bx,\by)$  requires an expensive  combinatorial calculation which is the reason why we perform fewer experiments and use a smaller test set than in \ref{sub:nn}. 
 In order to break ties, the values of $K^\star(\bx,\by)$ are perturbed  according to \eqqref{add_noise}. 
 
 With the parameter setting $L=9$, $n_w=50$, $n_c=5$ and $R=1000$, our theoretical lower bound for the generalization error is
 \begin{equation}\label{zia2}
\overline{\text{err}}(\featspace, \psi, \mathfrak T) \ge 1 - 0.073\, n^* - 1/R \qquad \text{for all $\mathcal F$ and all $\psi$,}
\end{equation}
which is obtained by  choosing $\ell=6$ in inequality \eqqref{main_lower_bound2}. This lead to an upper bound of $0.073\, n^* + 1/R$ on the success rate. This upper bound is evaluated for $n^*$ ranging from $1$ to $5$ in the fifth row of  Table \ref{table50}.

\begin{table}
  \caption{Accuracy in $\%$ on unfamiliar test points ($L=9$, $n_w=50$, $n_c=5$, $R=1000$).}
  \label{table50}
  \centering {\small
  \begin{tabular}{llllll}
    \toprule
                &    \small{$n^*\!=\!1$}   & \small{$n^*\!=\!2$}  &    \small{$n^*\!=\!3$}  & \small{$n^*\!=\!4$}  & \small{$n^*\!=\!5$}  \\
                      &    \small{$n_{\rm spl}\!=\!6$}   & \small{$n_{\rm spl}\!=\!7$}  &    \small{$n_{\rm spl}\!=\!8$}  & \small{$n_{\rm spl}\!=\!9$}  & \small{$n_{\rm spl}\!=\!10$}  \\
         \midrule
    Neural network   &  $99.9 \pm 0.1$ &  $99.9 \pm 0.1 $   & $99.9 \pm 0.1$ & $99.9 \pm 0.1$ & $100 \pm0.1$   \\
     {NN on feat. extracted by neural net}   &  $99.9 \pm 0.1$ &  $99.9 \pm 0.1$   & $99.9 \pm 0.1$& $99.9 \pm 0.1$&$99.9 \pm 0.1$  \\
     \midrule
     {NN on feat. extracted by $\psi^\star$}                        &  $2.4 \pm 0.3$     & $4.1 \pm 0.6$    &$5.5 \pm 0.6$ & $6.9 \pm 0.8$  & $8.0 \pm 0.8$  \\ 
            NN on feat. extracted by $\psi_{\rm one-hot}$     &  $2.0 \pm 0.3$     & $3.4 \pm 0.5$    &$4.8 \pm0.6$  & $5.7 \pm 0.5$  & $6.7 \pm 0.7$  \\
     Upper bound ($0.073 n^* + 1/1000$)   &  $7.4$ & $14.7$    &$22.0$ & $29.3$  & $36.6$  \\
     \midrule
           SVM  on feat. extracted by $\psi^\star$       &  $2.2 \pm 0.5$    &  $5.2 \pm 0.9 $  & $8.6 \pm 0.9$& $11.7 \pm 0.6$ & $15.1 \pm 1.2$ \\
              SVM  on feat. extracted by $\psi_{\rm one-hot}$       &  $1.2 \pm 0.1$ & $3.5 \pm 0.2$    &$6.4 \pm 0.2$ & $9.9 \pm 0.3$  & $13.6 \pm 0.4$  \\
     SVM with Gaussian kernel      &   $2.0 \pm 0.1$    &  $3.7 \pm 0.2$ & $5.4 \pm 0.2$ & $8.6 \pm 0.3$ & $12.1 \pm 0.3$ \\
    \bottomrule
  \end{tabular} }
\end{table}

\subsection{SVM on Features Extracted by $\psi_{\rm one-hot}$ and SVM with Gaussian Kernel}

For each of the 10 possible parameter settings in Table \ref{table150} and Table \ref{table50},  we do $100$ experiments.  For each experiment we generate:
\begin{itemize}
 \item A training set containing $R\times n_{\rm spl}$ sentences.
 \item A test set containing $10,000$ unfamiliar sentences (10 sentences per category).
 \end{itemize}
 We use  the feature map $\psi_{\rm one-hot}$ (which simply concatenates the one-hot-encodings of the words composing a sentence) to extract features from each sentence. These features are further normalized according to the formula
  \begin{equation}\label{duduche}
\tilde \bx  = \frac{\psi_{\rm one-hot}( \bx) - p}{ \sqrt{p(1-p)}}  \qquad \text{ where } p =1/n_w
\end{equation}
so that they are centered around $0$ and are $O(1)$.  We then use the SVC function of Scikit-learn  \cite{scikit-learn},  which itself relies on the LIBSVM  library \cite{chang2011libsvm}, in order to 
 run a soft multiclass SVM algorithm on these features. We tried various values for the parameter controlling the $\ell_2$ regularization in the soft-SVM formulation, and found that the algorithm, on this task,  is not sensitive to this choice ---  so we chose   $C=1$. The results are reported in the seventh row of both tables.

 We also tried a soft SVM with Gaussian Kernel
 $$
K(\bx, \by) = e^{-\gamma\| \bx-  \by\|^2}
$$
applied on the top of features extracted by $\psi_{\rm one-hot}$ and normalized according to \eqqref{duduche}.
We use the SVC function of Scikit-learn with $\ell_2$ regularization parameter set to $C=1$. 
For the experiments in Table  \ref{table150}  ($n_w=150$), 
the parameter $\gamma$  involved in the definition of the kernel was set to $\gamma= 0.25$ when  $n^*\in \{1,2\}$ and to  $\gamma = 0.1$ when  $n^* \in \{3,4,5\}$. 
For the experiments in Table  \ref{table50}  ($n_w=50$), it
was set to  $\gamma = 0.75$ when  $n^*=1$, to  $\gamma = 0.1$   when  $n^*=2$, and finally to
  $\gamma = 0.005$ when   $n^* \in \{3,4,5\}$.

 \subsection{SVM on Features Extracted by $\psi^\star$ }
 
  \begin{table}
  \caption{Search for the hyperparameter $\alpha$}
  \label{table:beta}
  \centering
  \begin{tabular}{llll}
    \toprule
        $\alpha$     &  $\lambda_\text{min} (G^\text{train}) $   &  $\lambda_\text{max} (G^\text{train}) $ & Test Accuracy  \\
        \midrule
      $0.001$  & $-90.9$ & $50,583.3$  &  6.1\% \\
      $0.01$  & $-81.5$ & $32,334.9$   & 6.1\% \\
      $0.1$  & $-56.8$ & $15,358.7$   & 6.6\% \\
     $1.0$    & $-22.9$& $4,191.5$     &  $7.6\%$ \\
     $10$ & $-2.5$  & $673.4$ & $7.2\%$ \\
     $15$ & 	$-0.138$ &	$471.7$ &	 $7.0\%$ \\
     $16$ &  $0.2$  &	 $445.5$	& $7.0\%$ \\
     $100$	& $4.7$ &	$86.9$	& $5.4\%$ \\
    $1000$ & 	$4.983$  &	$15.573$ &	$5.0\%$ \\
    \bottomrule
  \end{tabular}
\end{table}

Applying a SVM on the feature extracted by $\psi^\star$ is equivalent to running a kernelized SVM with kernel $K^\star$. A naive implementation of such algorithm leads to very poor results on our data model.  For such algorithm to not completely fail, it is important
 to carefully  ``rescale" $K^\star$ so that the eigenvalues of the corresponding Gram matrix are well behaved.
 Recall that
  \begin{equation} \label{def:Kstar_appendix_exp}
K^\star(\bx,\by) =
 \frac{n_c^L}{n_w^L} \; \;
\frac{\big|\{\vphi \in \partition : \vphi(x_\ell) = \vphi(y_\ell)  \text{ for all } 1 \le \ell \le L\}\big|}{|\Phi|}
\end{equation}
and let $\xi: \real \to \real$ be a strictly increasing function.  Since the nearest neighbor classification rule works by comparing the values of $K^\star$ on various pairs of points, it is clear that
using the kernel  $K^{\star\star}(\bx,\by)= \xi( K^\star(\bx,\by) )$ is equivalent to using the kernel  $K^\star(\bx,\by).$  In particular, choosing  $\xi(x):= \log(1+ (n_w^L/\alpha)x)$ gives the following family of optimal kernels:
\begin{align} \label{family}
K_\alpha^{\star\star}(\bx,\by) &= \log\left( 1 + \frac{n_w^L }{\alpha} K^\star(\bx,\by)\right) 
\end{align}
To be clear, all these kernels are exactly equivalent to the the kernel $K^\star$ when using a nearest neighbor classification rule. However, they lead to different algorithms when used for kernelized SVM. We have experimented with various choice of the function $\xi$ and found out that this logarithmic scaling works  well for kernelized SVM.

For each of the 10 possible parameter settings in Table \ref{table150} and Table \ref{table50},  we do $10$ experiments.  For each experiment we generate:
\begin{itemize}
 \item A training set containing $R\times n_{\rm spl}$ sentences.
 \item A test set containing $1,000$ unfamiliar sentences (one sentences per category).
 \end{itemize}
 Let us denote by $\bx^\text{train}_i$, $1 \le i \le R \times n_{\rm spl}$, the data points in one of these training set, and by $\bx^\text{test}_i$, $1 \le i \le 1000$, the data points in the corresponding test set. 
 In order to run the kernelized SVM algorithm we need to form the Gram matrices
\begin{equation} \label{Gram}
G^\text{train}_{ij} = K_\alpha^{\star\star}(\bx^\text{train}_i,\bx^\text{train}_j) \qquad \text{and } \qquad G^\text{test}_{ij} = K_\alpha^{\star\star}(\bx^\text{test}_i,\bx^\text{train}_j) \
\end{equation}
Constructing each of these Gram matrices  takes a few days on CPU. We then use the SVC function of Scikit-learn   to 
 run a soft multiclass kernelized-SVM algorithm. We tried various values for the parameter controlling the $\ell_2$ and found that the algorithm is not sensitive to this choice ---  so we chose   $C=1$.
 The algorithm, on the other hand, is quite sensitive to the choice of the hyperparamater $\alpha$ defining the kernel $K_\alpha^{\star\star}$. We experimented with various choices of $\alpha$
   and  found that choosing the smallest $\alpha$ that makes the  Gram matrix  $G^\text{train}$ positive definite works well (note that the Gram matrix {\it should} be positive semidefinite for the kernelized SVM method to make sense). In Table \ref{table:beta} we show an example, on a specific pair of train and test set\footnote{the training set and test set used in this experiment were generated by our data model with parameters
   $L=9$,  $n_w=50$, $n_c=5$,  $R=1000$, $n_{\rm spl}= 8$, and $n^*= 3$},  of how the eigenvalues of $G^\text{train}$ and the test accuracy depends on  $\alpha$.

\end{document}